%% file: eccv_stable_index_camera_rdy_arxiv.tex
\documentclass[runningheads]{llncs}
\usepackage{eccv}

\usepackage{eccvabbrv}

\usepackage{graphicx}
\usepackage{booktabs}

\usepackage[accsupp]{axessibility}  % Improves PDF readability for those with disabilities.

\def\iou{\text{IoU}}
\def\mAPH{$\text{mAPH}$}
\def\Conf{$\text{SI}_\text{c}$}
\def\Loc{$\text{SI}_\text{l}$}
\def\Ext{$\text{SI}_\text{e}$}
\def\Hea{$\text{SI}_\text{h}$}

\def\Second{Second\hspace*{-2pt}~\cite{yan2018second}}
\def\PPL{Pointpillar\hspace*{-2pt}~\cite{lang2019pointpillars}}
\def\CPPPL{CenterPoint$^*$\hspace*{-4pt}~\cite{yin2021center}}

\def\CP{CenterPoint\hspace*{-2pt}~\cite{yin2021center}}
\def\VoxelNext{VoxelNext\hspace*{-2pt}~\cite{chen2023voxelnext}}
\def\DSVT{DSVT\hspace*{-2pt}~\cite{wang2023dsvt}}
\def\VoxelRCNN{Voxel\,R-CNN\hspace*{-2pt}~\cite{deng2021voxel}}
\def\PartaNet{PartA2Net\hspace*{-2pt}~\cite{shi2019part}}
\def\PVRCNN{PV\,R-CNN\hspace*{-2pt}~\cite{shi2020pv}}
\def\PVRCNNPP{PV\,R-CNN$^+$\hspace*{-4pt}~\cite{shi2021pv}}
\def\TransFusion{TransFusion\dag\hspace*{-2pt}~\cite{bai2022transfusion}}

\def\Fir#1{\cellcolor{teal!50}#1}
\def\Sec#1{\cellcolor{teal!30}#1}

\def\improvea#1{{ (#1)}}
\def\improveb#1{{ \color[rgb]{0.27, 0.71, 0.45} (+#1)}}

\usepackage{multirow}
\usepackage{multicol}
\usepackage{colortbl}
\usepackage{overpic}
\def\MetricS{SI}
\def\MetricL{Stability Index}

\usepackage[% font=small,skip=5pt
]{caption}
\usepackage[pagebackref,breaklinks,colorlinks,citecolor=eccvblue]{hyperref}
\usepackage{orcidlink}

\begin{document}

\title{Towards Stable 3D Object Detection} 


% Second Author\inst{2,3}\orcidlink{1111-2222-3333-4444} \and
% Third Author\inst{3}\orcidlink{2222--3333-4444-5555}}

\author{
Jiabao Wang\inst{1}\thanks{Equal contribution.}
\and Qiang Meng\inst{2}$^{*}$
\and  Guochao Liu\inst{2}
 \and  Liujiang Yan\inst{2}
 \and Ke Wang\inst{2}
 \and Ming-Ming Cheng\inst{1, 3}
 \and Qibin Hou\inst{1, 3}\thanks{Corresponding author.}
}
\authorrunning{F.~Author et al.}

\institute{
    VCIP, College of Computer Science, Nankai University \and
    KargoBot Inc., China \and 
    NKIARI, Shenzhen Futian
% \email{lncs@springer.com}
\\
\url{https://github.com/jbwang1997/StabilityIndex}}
% \email{jbwang@mail.nankai.edu.cn \quad irvingmeng@outlook.com}}

\maketitle

\input{camera_rdy_contents/0-abstract.tex}
\input{camera_rdy_contents/1-introduction.tex}
\input{camera_rdy_contents/2-related.tex}

\input{camera_rdy_contents/3-methods.tex}
\input{camera_rdy_contents/4-experiments.tex}
\input{camera_rdy_contents/5-conclusions.tex}
\input{camera_rdy_contents/6-supplementary.tex}

\clearpage  % TODO REVIEW/FINAL: This \clearpage needs to be removed from both review and camera-ready versions.

\bibliographystyle{splncs04}
\bibliography{eccv}
\end{document}

%% file: camera_rdy_contents/0-abstract.tex
\begin{abstract}

In autonomous driving, the temporal stability of 3D object detection greatly impacts the driving safety.
However, the detection stability cannot be accessed by existing metrics such as mAP and MOTA, and consequently is less explored by the community.
To bridge this gap, this work proposes \MetricL{} (\MetricS{}), a new metric that can comprehensively evaluate the stability of 3D detectors in terms of confidence, box localization, extent, and heading.
By benchmarking state-of-the-art object detectors on the Waymo Open Dataset, \MetricS{} reveals interesting properties of object stability that have not been previously discovered by other metrics.
To help models improve their stability, we further introduce a general and effective training strategy, called Prediction Consistency Learning (PCL).
PCL essentially encourages the prediction consistency of the same objects under different timestamps and augmentations, leading to enhanced detection stability. 
Furthermore, we examine the effectiveness of PCL with the widely-used CenterPoint, and achieve a remarkable \MetricS{} of 86.00 for vehicle class, surpassing the baseline by 5.48.
We hope our work could serve as a reliable baseline and draw the community's attention to this crucial issue in 3D object detection.

\keywords{3D Object Detection \and Temporal Stability}

\end{abstract}

%% file: camera_rdy_contents/1-introduction.tex
\section{Introduction}\label{sec:intro}

\begin{figure}
    \centering
    \includegraphics[width=1\textwidth]{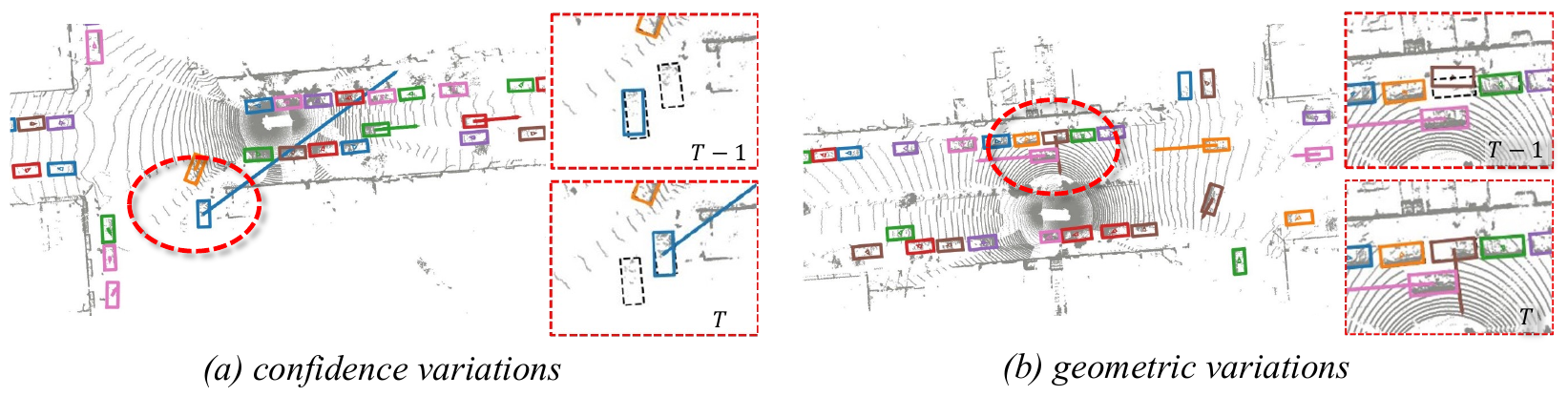}
    \captionof{figure}{
        Visualizations of potential safety threats caused by detection instability.
        On the left, confidence fluctuations lead to flickering boxes, which results in inaccurate object association and induces an abnormal velocity estimation.
        On the right, an intent of merging into traffic is erroneously forecast because of the shaking boxes, though the vehicle is stationary in fact.
        Here, dashed boxes represent the ground truths.
        Detection results are predicted by~\cite{yan2018second}, and object tracking is conducted with SimpleTrack~\cite{pang2021simpletrack}.
    }\label{fig:intro}
 \end{figure}

3D object detection aims to perceive objects of interest within the surrounding environment, utilizing data from diverse sources such as point clouds~\cite{zhou2018voxelnet, yan2018second, yin2021center, lang2019pointpillars, deng2021voxel, shi2019part}, camera images~\cite{li2022bevformer, wang2022detr3d}, multi-sensors~\cite{prakash2021multi, liu2023bevfusion,chen2017multi}, \etc.
Serving as a foundational component in autonomous driving, this task has attracted great attention from both academia and industry.
Numerous performant detectors~\cite{fan2021rangedet, bewley2020range, li2021lidar, zhang2022not, yang20203dssd, wang2023dsvt, wang2022detr3d, zhou2022centerformer} have been proposed recently, significantly advancing the development of 3D object detection.

Counterintuitively, it is rather common for highly performant detectors to exhibit instability. Sensor noise, model sensitivity, slight scene changes, and non-deterministic operators, all contribute to detection instability.
Despite great advancements, current state-of-the-art detectors predominantly emphasize improving \textit{single-shot} detection accuracy,
while often neglecting such \textit{temporal} stability.

Detection stability encompasses more than mere robustness; it extends to the broader context of ensuring human safety in autonomous driving. 
As exemplified in \cref{fig:intro}, unstable detections, on both confidence scores and bounding boxes, can result in abnormal velocity estimated by tracking.
These erroneous estimations may trigger false judgement on the behaviors of surrounding agents, potentially misleading the ego-vehicle to make improper or even hazardous decisions.
In addition, systematically complementing poor detection stability requires extra modules (\eg, Kalman filters~\cite{bishop2001introduction, wojke2017simple, bewley2016simple} with carefully and usually manually tuned parameters).
This not only increases system complexity and latency, but also necessitates tedious engineering efforts.
As a conclusion, enhancing detection stability is a crucial step towards safe and reliable autonomous driving.

To the best of our efforts, we find no prior work dedicating on detection stability for 3D object detectors.
One primary reason is the absence of an appropriate metric to quantify such stability.
% \textcolor{red}{To the best of our efforts, we find no prior work dedicating on detection stability for 3D object detectors, either can not find an appropriate metric to quantify such stability.}
Current metrics in measuring detection accuracy, such as  mAP~\cite{pascal-voc-2007},  usually overlook temporal information, which is fundamental for stability assessment. 
On the other hand, metrics designed for temporal object tracking (\eg,  MOTA and MOTP~\cite{bernardin2008evaluating}) are tailored to evaluate how well objects are tracked over time.
% However, tracking quality and detection stability are two orthogonal concepts.
Trackers are designed to be robust with respect to detection noises.
A well-implemented tracking algorithm will certainly hide instabilities of upstream detectors.
As illustrated in \cref{fig:intro}(b), although the detector produced inconsistent yaw and positions across two frames, trackers can still associate the two boxes and fuse the well-behaving box extent information.

% Moreover, these metrics are highly influenced by tracking algorithms in use and therefore fail to reflect detection properties directly.
We argue a new metric is needed for detection stability. %in the field of detection stability.
For this purpose, we initiate a comprehensive analysis of the task, identifying that an effective metric should exhibit four core properties:
\textit{1) Comprehensiveness}: The metric must take all detected attributes into account.
% \textit{1) Comprehensiveness}: The metric must take all relative elements into account.
\textit{2) Homogeneity}: All attributes should be uniformly integrated into the metric.
\textit{3) Symmetry}: The metric should be consistent regardless of the input order.
\textit{4) Marginal Unimodality}: The metric value will never increase as the stability of any element deteriorates. %needs to monotonically decrease as the stability of any element deteriorates.
Based on our analysis, we accordingly propose a novel metric called \MetricL{} (\MetricS{}), which evaluates stability by quantifying the temporal consistency in terms of the confidence score, box location, extent, and heading.
Through our meticulously designed schemes, the proposed \MetricS{} fully complies with all the aforementioned requirements, as demonstrated by our rigorous theoretical proofs.

On the large-scale Waymo Open Dataset (WOD)~\cite{sun2020scalability}, we thoroughly benchmark various popular 3D object detectors and observe that there is no evident correlation between existing metrics (\eg, mAP and MOTA) and our proposed stability metric \MetricS{}.
Furthermore, our experiments reveal that some effective tricks in object detection, like using more augmentations and multi-frame strategies, fail to yield many improvements in terms of stability. % without incurring acceptable burdens.

To this end, we additionally introduce a framework called Prediction Consistency Learning (PCL), which in essence penalizes prediction errors' discrepancies from the same objects under different timestamps and augmentations.
It's noteworthy that our PCL is a general framework applicable to all detectors, and it introduces no additional cost during inference.
Without bells and whistles, PCL %significantly 
boosts the \MetricS{} of \CP{} from 80.52 to an impressive 86.00 \MetricS{} on the vehicle class,
surpassing all state-of-the-art detectors.

We summarize our contributions as follows:
\begin{enumerate}
    \itemsep0em
    \item For the first time, we provide a comprehensive analysis of detection stability. %in 3D object detection.
    Subsequently, we introduce the \MetricL{} (\MetricS{}) metric, which uniformly evaluates and positively indicates the stability of all detection elements.
    Rigorous theoretical proofs are further presented to validate the efficacy of \MetricS{}.
    \item A general framework termed Prediction Consistency Learning (PCL) is proposed to boost detection stability.
    Extensive experiments on the Waymo Open Dataset unearth several intriguing insights of object stability as well as demonstrate the effectiveness of the PCL.
    %  on improving detection stability.
\end{enumerate}

%% file: camera_rdy_contents/2-related.tex
\section{Related Work}\label{sec:related_work}

\subsection{3D Object Detection}
3D object detection, a fundamental building block in autonomous driving, focuses on accurately locating objects within a three-dimensional space.
Prior works in this domain can be broadly categorized based on input modalities.% into LiDAR-based, image-based, and multi-modal methods.

Most existing LiDAR-based works~\cite{zhou2018voxelnet, yan2018second, yin2021center, lang2019pointpillars, deng2021voxel, shi2019part,fan2022embracing,dosovitskiy2020image, li2023pillarnext, zhou2022centerformer} transform non-uniform point clouds into regular 2D pillars or 3D voxels, and employ convolutions for efficient processing in later stages.
Beyond voxel-based methods, the task can also be accomplished using alternative point cloud representations, including the range-view~\cite{fan2021rangedet, bewley2020range, meyer2019lasernet, chai2021point, sun2021rsn}, point-view~\cite{chen2019fast, shi2019pointrcnn, li2021lidar, yang2018ipod, zhang2021varifocalnet, qi2018frustum, zhang2022not, yang20203dssd, shi2020point, qi2019deep, shi2020pa2}, and their combinations~\cite{sun2021rsn, shi2020pv}.
% combinations of heterogeneous views~\cite{sun2021rsn, shi2020pv}.
Several studies~\cite{sun2020scalability, liu2023flatformer, wang2023dsvt} introduce recently popular Transformer architectures and achieve remarkable detection accuracies.
% boosts in terms of detection accuracy.

Transformer architectures also demonstrate great success in transforming camera images into bird's-eye-view features.
Such algorithmic breakthrough paved the way for vision~\cite{li2022bevformer, wang2022detr3d} and fusion~\cite{prakash2021multi, liu2023bevfusion,chen2017multi} based 3D detection for self-driving vehicles.
It's noteworthy that our proposed metric and method are agnostic to input modalities, thus applicable to all 3D object detection methods.

\subsection{Related Metrics}
Properly measuring performances is crucial for any machine learning task, let alone 3D object detection.
The KITTI~\cite{Geiger2012CVPR} dataset plays a pioneering role in evaluating autonomous driving tasks, employing the well-established average precision (AP) as metric.
Waymo Open Dataset~\cite{sun2020scalability} further extends the metric into APH by accounting for heading errors.
In contrast, nuScenes~\cite{caesar2020nuscenes} questions the suitability of IOU-based metrics for vision-only methods, which usually come with large localization errors.
Therefore, a new metric called NDS is proposed to assess error-prone predictions by utilizing a thresholded 2D center distance.

Multi-object tracking (MOT), the downstream task of object detection, stands as another critical component for autonomous driving.
% \citet{bernardin2008evaluating}
Bernardin and Stiefelhagen~\cite{bernardin2008evaluating} 
introduces metrics of MOTA and MOTP, where MOTA combines errors including False Negatives, False Positives, and Identity Switches, while MOTP focuses on how good sequences overlap with ground truths.
% \citet{weng20203d} 
Weng \etal~\cite{weng20203d}
points out that both metrics do not take scores into account, and extends them into AMOTA and AMOTP by averaging scores across different recall levels.
In general, detection metrics disregard temporal relationships of detected boxes, whereas tracking metrics mainly focus on whether objects are correctly associated across frames.
Previous methods fall short in capturing detection stability across frames, which serves as the key motivation behind this work.

% \subsection{Temporally Consistent Model Outputs}
% The pursuit of consistent model outputs has been similarly mentioned in other domains~\cite{chen2020optical, zhang2016stability}. % \cite{meng2021learning} add if accepted
% For example, StableVideo~\cite{chen2020optical} preserves objects' appearance in diffusion-based generative methods by introducing temporal dependency.

% citet{zhang2016stability} decomposes stability of video detection and tracking into fragment error, center position error, scale and ratio error, and demonstrates the low correlation with accuracy metric like mAP. 
% % However, in spite of the success in other fields, no prior works have specifically addressed stability in 3D object detection to the best of our knowledge.

%% file: camera_rdy_contents/3-methods.tex
\section{Methodology}\label{sec:method}

In this section, we first comprehensively analyze the stability in 3D object detection.
Based on our analysis, we introduce a novel metric called \MetricL{}(\MetricS) and prove its key properties.
In the end, we introduce our Prediction Consistency Learning (PCL) to enhance detection stability.

\subsection{Notations}\label{sec:method_note}

A valid prediction $P$ from 3D object detectors comprises a confidence score $c$ and a 3D bounding box defined as $B = (x, y, z, l, w, h, \theta)$. 
Here, $(x, y, z)$ are the coordinates of the box center, while $(l, w, h)$ denote the box extent, and $\theta$ represents the yaw angle. \textit{Elements} and \textit{attributes} are used interchangeably to refer to the box properties.
% For a ground-truth box, we label it with a superscript $g$, denoting it as $B^g$.

%Assume that there are two boxes $B_1, B_2$, where $B_i = (x_i, y_i, z_i, l_i, w_i, h_i, \theta_i), i\in\{1, 2\}$.
% Given two boxes $B_1$ and $B_2$, a transformation function $T(B_1, B_2)$ can be defined to represent the mapping from $B_1$ to $B_2$.
% Correspondingly, we apply the transformation on an arbitrary box $B$ and get $B_o = T(B_1, B_2) \otimes B$.
% By letting $\Delta\theta = \theta_2 - \theta_1$, our customized operation $T(B_1, B_2)\otimes$ is expressed with
% \begin{equation*}
%     \footnotesize
%     \begin{split}
%         \begin{pmatrix}
%             x_o \\
%             y_o \\
%             z_o
%         \end{pmatrix} & = 
%         \begin{pmatrix}
%                 \cos\Delta\theta & \sin\Delta\theta & 0 \\
%                 -\sin\Delta\theta & \cos\Delta\theta & 0 \\
%                 0                & 0                & 1
%         \end{pmatrix}
%         \begin{pmatrix}
%              x - x_1 \\
%              y - y_1 \\
%              z - z_1
%         \end{pmatrix}
%         +  \begin{pmatrix}
%             x_2 \\
%             y_2 \\
%             z_2
%         \end{pmatrix}
%         , \\    
%         \begin{pmatrix}
%             l_o \\ w_o \\ h_o
%         \end{pmatrix} & =
%         \begin{pmatrix}
%             l / l_1 \times l_2 \\
%             w / w_1 \times w_2 \\
%             h / h_1 \times h_2 \\
%         \end{pmatrix}, \qquad
%        \theta_o = \theta + \Delta\theta. \\
%     \end{split}
% \end{equation*}

Given two boxes $B_1$ and $B_2$, we define a transformation function $T_{B_1 \rightarrow B_2}(\cdot)$ which represents the mapping from $B_1$ to $B_2$.
Consequently, we can apply this customized transformation to an arbitrary box $B$, resulting in $\hat{B} = T_{B_1 \rightarrow B_2}(B)$, where
% $\hat{B} = T_{B_1 \rightarrow B_2}(B)$ represents applying the $B_1$ to $B_2$ transformation to an arbitrary box $B$ and get corresponding $\hat{B}$, which can be mathematically expressed with
\begin{equation*}
    \footnotesize
    \begin{split}
        \begin{pmatrix}
            \hat{x} \\
            \hat{y} \\
            \hat{z}
        \end{pmatrix} & = 
        \begin{pmatrix}
                \cos(\theta_2 - \theta_1) & \sin(\theta_2 - \theta_1) & 0 \\
                -\sin(\theta_2 - \theta_1) & \cos(\theta_2 - \theta_1) & 0 \\
                0                & 0                & 1
        \end{pmatrix}
        \begin{pmatrix}
             x - x_1 \\
             y - y_1 \\
             z - z_1
        \end{pmatrix}
        +  \begin{pmatrix}
            x_2 \\
            y_2 \\
            z_2
        \end{pmatrix}
        , \\    
        \begin{pmatrix}
            \hat{l} \\ \hat{w} \\ \hat{h}
        \end{pmatrix} & =
        \begin{pmatrix}
            l_2 / l_1 \times l \\
            w_2 / w_1 \times w \\
            h_2 / h_1 \times h \\
        \end{pmatrix}, \qquad
       \hat\theta = \theta + (\theta_2 - \theta_1). \\
    \end{split}
\end{equation*}
In essence, this operation transforms discrepancies between $B_1$ and $B$ into those between $B_2$ and $\hat{B}$.

\begin{figure}[!t]
    \centering
    \includegraphics[width=0.9\textwidth]{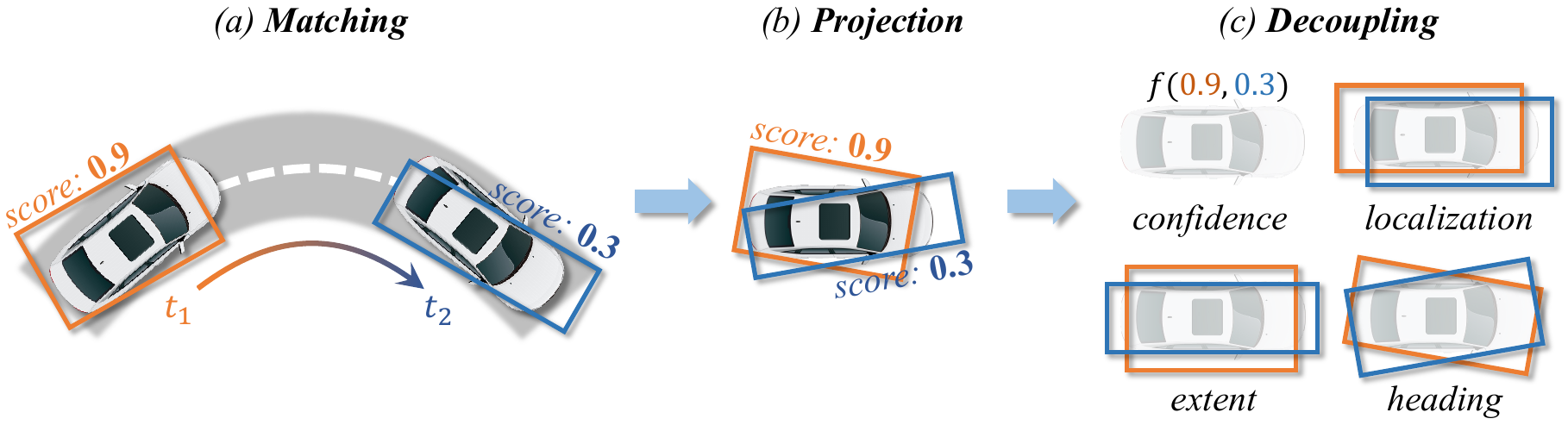}
    \caption{
        The procedure of computing \MetricL{}.
        The \textcolor{orange}{orange} and \textcolor{blue}{blue} 
        boxes represent the best matches between the predictions and the ground-truths searched by the Hungarian algorithm.
        These boxes are subsequently associated across frames using their object ID labels.
        After projecting predictions into a pre-built pivot box, \MetricS{} decouples them into element-wise computations, which are then aggregated for the final assessment of detection stability.
    }
    \label{fig:calculation}
\end{figure}

\subsection{Analysis of Detection Stability}\label{sec:method_analysis}

In the context of autonomous driving, variations in any of the predicted attributes from detectors may result in hazardous situations.
For instance, fluctuations in box locations and heading may lead to inaccurate velocity estimations, potentially leading to unsafe interaction decisions.
Unstable confidence scores may cause flickering predictions and hinder the autonomous driving system from accurately tracking objects.
Moreover, erratic predictions of a nearby vehicle's size may prompt the ego-vehicle to take improper evasive maneuvering.
In summary, the stability of all detection elements must be comprehensively taken into account to ensure the safety of autonomous driving.

A naive approach for assessing stability is to sum variations of all these elements, essentially extending 
% \citet{zhang2016stability} 
Zhang and Wang~\cite{zhang2016stability}
in 2D video detection.
However, these variations should not be directly added as these detection attributes represent different physical properties of the object.
Moreover, element variations are agnostic of the object properties and therefore fail to capture the hazard levels caused by unstable predictions.
For example, large jitters on the yaw angle can lead to rapid changes in object behaviors for large-volume objects (\eg, vehicles).
In contrast, pedestrians suffer more from the instability of center offsets and box dimensions rather than headings.
Therefore, how to standardize these elements into a single and consistent unit remains a challenging problem.

One possible way to unify physical units of the box-related elements is to adopt the Intersection-over-Union (\iou{}) as used in the mAP metric.
To achieve this, we begin by assessing detection stability at the smallest unit, involving a single object at two timestamps $t_1, t_2$ as illustrated in \cref{fig:calculation}.
Denote the ground-truth boxes as $B_1^g, B_2^g$ and the predictions are  $P_i = \{c_i, B_i\}, i \in \{1, 2\}$.
The \iou{} between the two predicted 3D boxes cannot be directly computed due to object movement.
In contrast, our pre-defined operation enables the measurement by projecting boxes onto one of the ground-truths.
For example, we can project $B_1$ into the second ground-truth as $\hat B_1 = T_{B_1^g \rightarrow B_2^g}(B_1)$ and compute $\iou(\hat B_1, B_2)$.
This \iou{} can reflect the detection stability to some extent.
Nevertheless, this measurement has two significant flaws (proved by Properties 1 and 2 in the supplementary):
(1) \iou{} varies with the order of frames, \ie, $\iou(\hat B_1, B_2)\neq \iou(B_1, \hat B_2)$.
(2) \iou{} is not marginal unimodal. 
In other words, enhancing the stability of an element can, at times, lead to a poorer \iou{} value.
Both flaws prohibit \iou{} from serving as an effective assessment of detection stability.

Through the detailed analysis of stability and exploration of potential solutions, 
% we have identified four key properties that an effective metric for stability should meet:
we identify four key properties that an effective metric should meet:
\begin{itemize}
    \itemsep0em
    \item \textit{\textbf{Comprehensiveness}}: 
    % The metric should comprehensively reflect influences from all relevant elements, including confidences, center locations, box extents, and headings.
    The metric should comprehensively reflect influences from all relevant detection elements.
    \item \textit{\textbf{Homogeneity}}:
    Influences caused by all elements should be well-processed into unified physical units.
    \item \textit{\textbf{Symmetry}}:
    The metric values should be consistent when applied to both forward and reverse inputs.
    \item \textit{\textbf{Marginal Unimodality}}:
    For each element with others fixed, the metric should be unimodal \wrt its stability.
\end{itemize}

\subsection{\MetricL}\label{sec:method_metric}
While the \iou{} is a promising starting point, meeting the four properties demands careful designs to effectively integrate the confidence score and address the asymmetry and non-unimodality design flows.
To this end, we introduce schemes of projection with pivot boxes, element decoupling, and stability aggregation, as illustrated in \cref{fig:calculation}. 
Ultimately, we assess the stability of object pairs in consecutive frames and denote the metric as \MetricL{} (\MetricS{}).

\noindent\textbf{Projection with pivot box.}
Since projections onto either of ground-truths can introduce the asymmetry issue, we therefore propose to cast predictions onto an intermediary pivot box $B^p=(0, 0, 0, l^p, w^p, h^p, 0)$.
Here, we leverage geometric averages $l^p=\sqrt{l^g_1l^g_2}$, $w^p=\sqrt{w^g_1w^g_2}$, $h^p=\sqrt{h^g_1h^g_2}$ to ensure that the pivot box's dimensions closely match those of the ground-truths $B^g_1, B^g_2$.
This is crucial for accurate stability measurements, as objects of different sizes are affected by fluctuations to varying degrees.
Finally, we have $\hat B_1 = T_{B^g_1 \rightarrow B^p}(B_1), \hat B_2 = T_{B^g_2 \rightarrow B^p}(B_2)$ as indicated in \cref{fig:calculation}(b).

\noindent\textbf{Element decoupling.}
For the marginal unimodality, the metric must exhibit the following two characteristics when all elements are fixed except for one arbitrary element:
(1) The metric reaches the peak value if and only if the element is stable. % across frames.
(2) The metric value is monotonically non-decreasing as the stability of an element deteriorates in any continuous direction.
We recognize that \iou{} fails to meet these characteristics due to the mutual interference between elements, and therefore propose to decouple them into four distinct parts as shown in \cref{fig:calculation}(c).
For instance, to measure the localization stability, we make elements except for box centers in $\hat B_1, \hat B_2$ to be identical.
Specifically, we replace them with those from the pivot box, resulting in $\hat B^{loc}_i = (\hat x_i, \hat y_i, \hat z_i, l^p, w^p, h^p, 0), i\in \{1, 2\}$.
Similarly, we can have $\hat B^{ext}_i, \hat B^{hdg}_i, i\in \{1, 2\}$ for the box extent and heading.
Then, we assess the stability in box localization and extent by the two equations:
\begin{equation}    
    \small
    \MetricS{}_{l} = \iou(\hat B^{loc}_1, \hat B^{loc}_2),\quad
    \MetricS{}_{e}  = \iou(\hat B^{ext}_1, \hat B^{ext}_2).
\end{equation}

Directly employing $\iou{}(\hat B^{hdg}_1, \hat B^{hdg}_2)$, however, violates the unimodality if the angle difference between $\hat\theta_1$ and $\hat\theta_2$ exceeds $\pi/4$ (proved in Lemma 3 in our supplementary).
Therefore, we regard this case as a failure and explicitly set the metric to be 0. 
The stability in box heading finally is
\begin{equation}    
    \small
    \MetricS{}_{h} = \left\{
    \begin{split}
        &0, \qquad\quad \text{if } |\hat \theta_1-\hat\theta_2| \geq \pi/4,\\
        &\iou(\hat B^{hdg}_1, \hat B^{hdg}_2),\  \text{otherwise.}\\
    \end{split}
    \right.
\end{equation}

The stability in confidence can be captured by the difference between the scores $c_1, c_2$, \ie, using $1 - |c_1 - c_2|$.
A remaining issue is that this function is vulnerable to intrinsic confidence scales of object detectors.
For example, if all scores are divided by a scaling factor, the detection performance and stability should remain unaffected.
However, the value of $1-|c_1-c_2|$ would increase, leading to an inaccurate measurement of stability.
To address this issue, we calculate 99\% and 1\% percentile of all confidences as $c^{0.99}$ and $c^{0.01}$.
The confidence stability is then calibrated by 
\begin{equation}
    \small
    \MetricS{}_{c} = \max\left(0, 1 - |c_1 - c_2| / (c^{0.99}- c^{0.01})\right).
\end{equation}

\noindent\textbf{Stability aggregation.}
In the last step, we aggregate stability from all components using the following formulation:
\begin{equation}
    \small
    SI = SI_c \times (SI_l + SI_e + SI_h)/3.
\end{equation}
Here, $\MetricS{}_c\in[0, 1]$ is treated as the weight of the box stability.
$\MetricS{}_l, \MetricS{}_e, \MetricS{}_h$ can be averaged thanks to the same unit of \iou{}.
In the end, \MetricS{} successfully satisfies the four properties of a valid stability evaluator according to \cref{prop:1,prop:2}.
Detailed analyses and theoretical proofs are available in our supplementary.

\begin{lemma}
    \MetricS{} is a symmetric metric which uniformly assesses all elements' influences on the detection stability. 
   % with all elements for detection stability involved.
    \label{prop:1}
\end{lemma}
\begin{lemma}
    \MetricS{} is marginal unimodality \wrt all elements. 
    The maximum value of 1 is reached if and only if the detection is perfectly stable across frames.
    \label{prop:2}
\end{lemma}

Our previous discussions focus on the smallest set, consisting of a single object at two consecutive timestamps. 
To assess \MetricS{} for large-scale benchmarks, we begin by pairing each ground-truth with a prediction using the Hungarian algorithm.
With the labeled object IDs, we segment the evaluation into calculating \MetricS{} for numerous smallest sets.
% smallest sets of object pairs. 
The final result is simply the average of all values.
More details like the handling of corner cases are presented in the supplementary.

\subsection{Prediction Consistency Learning}\label{sec:method_loss}

\begin{figure}[!t]
    \centering
    \includegraphics[width=1\textwidth]{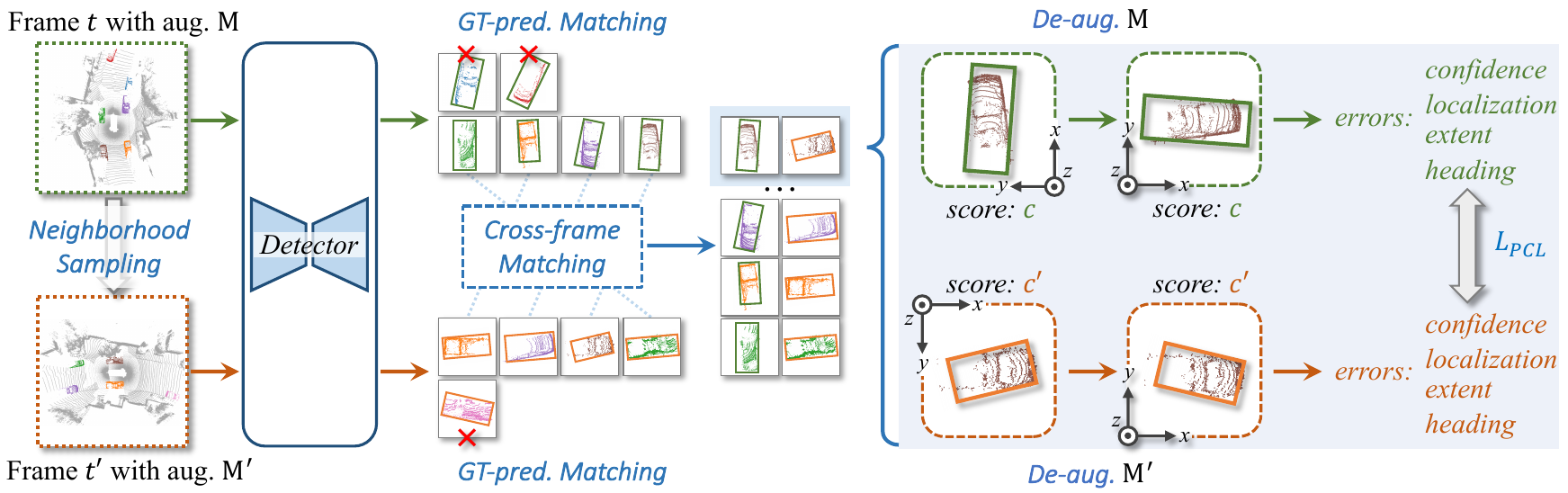}
    \caption{
        The pipeline of the proposed Prediction Consistency Learning (PCL).
        In each iteration, PCL samples a pair of frames at neighboring timestamps $t$ and $t'$, and applies augmentations $\mathbf{M}$ and $\mathbf{M}'$ to the paired samples.
        GT-prediction matching and cross-frame matching then collaboratively associate the detector's predictions from the same objects between the two frames.
        After the de-augmentation procedure, PCL calculates the prediction errors in terms of confidence, localization, extent, and heading, which are defined in the object self-coordinate system.
        Finally, PCL penalizes the error disparities among all prediction pairs to enforce the temporal consistency.
        In the figure, pred. and aug. represent prediction and augmentation, respectively.
        % After generating predictions from those frames, PCL collect the respective predictions by the gt-prediction and cross-frame matching.
        % To contrast prediction errors under the same coordinate system, de-augmentation operations are conducted predictions, respectively.
        % Finally, PCL calculate the prediction consistently losses between each pair of corresponding predictions in terms of confidence, localization, extent, and heading.
    }
    \label{fig:structure}

\end{figure}

Beyond the design of the metric, we further attempt to boost the detection stability of 3D object detectors.
For this purpose, we introduce a general and effective training strategy named Prediction Consistency Learning (PCL), as illustrated in~\cref{fig:structure}.
Our PCL is built on the core idea of encouraging prediction consistency across frames under various augmentations and timestamps.
It consists of four key stages: neighborhood sampling, prediction pairing, de-augmentation, and prediction consistency loss.

\noindent\textbf{Neighborhood sampling.}
For each frame $F$ with timestamp $t$, we begin by uniformly sampling an integer $\Delta t$ from the range $[-n, n]$, where $n$ is a pre-defined parameter.
Subsequently, we get the frame $F'$ at timestamp $t+\Delta t$ and bundle $F, F'$ as a pair-wise input for the network.
The frames are further augmented separately by random flipping, rotation, and scaling.
We record the augmentations into matrices $\mathbf{M}$ and $\mathbf{M}'$, where $\mathbf{M}$ can be described as follows:
\begin{equation}
    \small
    \mathbf{M}\!=\!
    \begin{pmatrix}
        i_x & 0 & 0 \\
        0 & i_y & 0 \\
        0 & 0 & 1
    \end{pmatrix} \!
    \begin{pmatrix}
        \cos(\alpha) & \sin(\alpha) & 0 \\
        -\sin(\alpha) &  \cos(\alpha)  & 0 \\
        0                    & 0                     & 1
    \end{pmatrix} \cdot s.
\end{equation}
Here, $i_x$ and $i_y$ indicate whether the frames are x and y direction flipped, with -1 meaning the corresponding flipping occurs and 1 otherwise.
$\alpha$ is the angle applied by random rotation, and $s$ denotes the factor for random scaling.

\noindent\textbf{Prediction pairing.}
After the detector generating predictions from paired samples, our next step is to gather the corresponding predictions for comparisons.
We first perform the GT-prediction matching to assign each ground-truth box with the best-matched prediction, which can be accomplished by the Hungarian algorithm or any other rational method.
Subsequently, cross-frame matching associates predictions between two frames by corresponding object IDs and creates prediction pairs for later comparisons.

\noindent\textbf{De-augmentation.} 
Data augmentation used during training can largely alter the patterns of detection errors, impeding fair comparisons of predictions in each pair.
For example, random scaling can scale up the errors in box locations and extents, while random flipping may change the error direction.
Therefore, we apply a de-augmentation step on each prediction to eliminate the influences of augmentations.
For a prediction $P = \{c, x, y, z, l, w, h, \theta\}$, we recover it into $\bar P = \{\bar c, \bar x, \bar y, \bar z, \bar l, \bar w, \bar h, \bar \theta\}$ with the corresponding  $\mathbf{M}$:
% utilize the corresponding $\mathbf{M}$ to recover it into $\bar P = \{\bar c, \bar x, \bar y, \bar z, \bar l, \bar w, \bar h, \bar \theta\}$ by the following formulations:
\begin{equation}
    \small
    \begin{cases}
        \bar c = c, \\
        (\bar x, \bar y, \bar z)^T = \mathbf{M}^{-1} (x, y, z)^T, \\
        (\bar l, \bar w, \bar h)^T = (l, w, h)^T /s, \\
        \bar \theta = i_x \cdot i_y \cdot (\theta - \alpha).
    \end{cases}
\end{equation}

\noindent\textbf{Prediction consistency loss.}
Before introducing the consistency loss, we first compute prediction errors for a de-augmented prediction $\bar P$ with respect to the ground-truth box $B^g$.
We define the error for confidence as $e_c = 1 - \bar c$.
% \begin{equation}
%     \small
%     e_c = 1 - \bar c.
% \end{equation}
Prediction errors in box localization, extent, and heading are computed in the object's ego-coordinate system.
Specifically, the error for box center is calculated by
\begin{equation}
    \small
    \mathbf{e}_l\!=\! 
    \begin{pmatrix}
            \cos\theta^g  & \sin\theta^g & 0 \\
            -\sin\theta^g & \cos\theta^g  & 0 \\
            0 & 0 & 1
    \end{pmatrix}\!
    \begin{pmatrix}
        \bar x - x^g \\
        \bar y - y^g \\
        \bar z - z^g \\
    \end{pmatrix}.    
\end{equation}    
The prediction error for the box extent is formulated as
\begin{equation}
    \small
    \mathbf{e}_e = \left(\bar l/ l^g, \bar w/ w^g, \bar h/h^g\right)^T.
\end{equation}
In the end, the error $\mathbf{e}_h$ for box heading is encoded into trigonometric vectors:
\begin{equation}
    \small
    \mathbf{e}_h = \left(\sin(\bar \theta - \theta^g), \cos(\bar \theta - \theta^g)\right)^T.
\end{equation}

Our final step is to encourage each prediction pair to reveal similar patterns in terms of prediction errors.
Thereby, we collect the pair-wise errors $\{e_{c, i}, e_{c, i}'\}$, $\{\mathbf{e}_{l, i}$, $\mathbf{e}_{l, i}'\}$, $\{\mathbf{e}_{e, i}$, $\mathbf{e}_{e, i}'\}$, and $\{\mathbf{e}_{h, i}, \mathbf{e}_{h, i}'\}$ for $i\in\{1, 2, \cdots, N\}$,
% an object to have similar prediction errors across frames.
% For this purpose, we match objects based on their ground-truth IDs and collect the pair-wise errors $\{e_{c, i}, e_{c, i}'\}, \{\mathbf{e}_{l, i}, \mathbf{e}_{l, i}'\}, \{\mathbf{e}_{e, i}, \mathbf{e}_{e, i}'\}, \{\mathbf{e}_{h, i}, \mathbf{e}_{h, i}'\}$ for $i=1, 2, \cdots, N$.
where $N$ is the number of successfully associated objects between frames $F$ and $F'$.
In the end, our prediction consistency loss is:
\begin{equation}
    \small
    \begin{split}
        L_{PCL} = & \frac{1}{N} \sum_{i=1}^N \left( w_1\cdot \text{MSE}(e_{c, i}, e_{c, i}') + w_2\cdot L_1(\mathbf{e}_{l, i}, \mathbf{e}_{l, i}')
            \right. \\
        & \left.+ w_3\cdot L_1(\mathbf{e}_{e, i}, \mathbf{e}_{e, i}') +  w_4 \cdot L_1(\mathbf{e}_{h, i}, \mathbf{e}_{h, i}')\right).
    \end{split}    
\end{equation}
Here $w_1$, $w_2$, $w_3$, and $w_4$ are weights to balance different parts in our loss, which being 1 if not specified.
MSE and $L_1$ are losses of mean square error and $L_1$ distance, respectively.
Both original detection losses and our prediction consistency loss are leveraged to train the object detectors.

%% file: camera_rdy_contents/4-experiments.tex
\section{Experiments}\label{sec:exp}
% In this part, we first evaluate and analyze the detection stability of commonly used 3D detectors in \cref{sec:benckmark}.
% Then, we present extensive experiments to demonstrate the effectiveness of the proposed PCL framework in \cref{sec:boosting}.
% Our additional experiments, including comparisons of various metrics, analysis of improvements from PCL, results on the NuScenes~\cite{caesar2020nuscenes} benchmark and more visualizations, are presented in detail in the supplementary.

\subsection{Benchmark on the Waymo Open Dataset}\label{sec:benckmark}

\begin{table}[t]
    \centering
    \caption{
        Benchmarks on Waymo Open Dataset.
        Models are sorted based on \mAPH{} on the class vehicle.
        We use two different intensities of colors to highlight the highest and second-highest results in each column.
        ``\dag'' denotes the model is not LiDAR-based only.
        \CPPPL \ represents the pillar version of CenterPoint.
        % \CPPPL: CenterPoint-Pillar; \PVRCNNPP: PV R-CNN++;
    }\label{tab:benchmark}
    % \small
    % \setlength{\tabcolsep}{1.5pt}
    \renewcommand{\arraystretch}{1.2}
    \scalebox{0.76}{
    \begin{tabular}{l >{\columncolor{lightgray!40}}c >{\columncolor{lightgray!40}}c|cccc >{\columncolor{lightgray!40}}c >{\columncolor{lightgray!40}}c| cccc >{\columncolor{lightgray!40}}c >{\columncolor{lightgray!40}}c| cccc}
        \toprule[0.8pt] 
        % \\[-13pt] \toprule[0.8pt]
        \multirow{2}*{Methods} & \multicolumn{6}{c}{Vehicle(\%)}                                                   & \multicolumn{6}{c}{Pedestrain(\%)}                                                              & \multicolumn{6}{c}{Cyclist(\%)}                                   \\ \cmidrule(r){2-7} \cmidrule(r){8-13} \cmidrule(r){14-19}
                               & \mAPH\hspace*{-2.4pt}       & SI          & \Conf       & \Loc        & \Ext        & \Hea        & \mAPH\hspace*{-2.4pt}        & SI          & \Conf       & \Loc        & \Ext        & \Hea        & \mAPH\hspace*{-2.4pt}       & SI          & \Conf       & \Loc        & \Ext        & \Hea       \\ \toprule[0.9pt]
        \Second                & 72.60       & 81.37       & 90.2       & 84.2       & 92.0       & 92.2       & 59.81       & 63.07       & 83.9       & 69.6       & 87.8       & 67.6       & 61.95       & 67.21       & 81.1       & 76.1       & 88.3       & 83.8       \\
        \CPPPL               & 72.82       & 80.61       & 89.0       & 85.4       & 91.0       & 92.8       & 65.28       & 64.57       & 83.2       & 74.4       & 87.4       & 68.9       & 65.87       & 68.06       & 80.8       & 77.7       & 87.0       & 85.9       \\ 
        \PPL                   & 72.84       & 80.84       & 89.6       & 84.4       & 92.3       & 91.6       & 54.64       & 62.03       & 84.7       & 72.1       & \Sec{88.8} & 57.9       & 59.51       & 66.14       & 82.2       & 74.9       & 88.0       & 77.4       \\ 
        \CP                    & 73.73       & 80.52       & 89.0       & 85.3       & 90.7       & 92.9       & 69.50       & 68.40       & 85.7       & 73.3       & 88.6       & 75.0       & 71.04       & 68.40       & 80.3       & 78.5       & 87.4       & 89.8       \\ 
        \PartaNet              & 75.02       & 82.86       & 91.4       & 85.4       & 91.7       & 91.7       & 66.16       & 65.08       & 84.6       & 73.6       & 86.7       & 67.0       & 67.90       & 72.73       & 85.9       & 79.3       & 87.0       & 84.3       \\ 
        \PVRCNN                & 75.92       & 83.73       & 91.9       & 86.4       & 92.3       & 91.7       & 66.28       & 66.17       & 86.0       & 73.5       & 87.4       & 66.6       & 68.38       & 73.53       & 86.8       & 78.9       & \Sec{88.4} & 83.2       \\ 
        \VoxelRCNN \hspace*{-2pt}        & 77.19       & 84.26       & 92.0       & 86.7       & 92.1       & 93.3       & 74.21       & 69.50       & 86.9       & 75.3       & 88.1       & 73.6       & 71.68       & 73.23       & 84.4       & 80.1       & 87.7       & 89.3       \\ 
        \VoxelNext             & 77.84       & \Sec{84.82}       & \Fir{92.9} & 86.3       & 91.6       & 94.2       & 76.24       & \Fir{74.74} & \Fir{92.7} & \Sec{75.7}       & 88.0       & 75.8       & \Fir{75.59} & \Fir{76.48} & \Fir{90.0} & 79.2       & 84.9       & 87.8       \\ 
        \PVRCNNPP              & 77.88       & 84.49       & 92.1       & \Fir{87.2} & \Sec{92.4}       & 93.2       & 73.99       & 69.27       & 86.8       & 75.3       & 88.1       & 73.2       & 71.84       & 73.05       & 84.2       & \Sec{80.3}       & 87.7       & 89.2       \\
        % \PVRCNNResPP           & 78.33       & \Fir{85.17} & 92.5       & \Fir{87.5} & \Sec{92.5} & 93.9       & 75.75       & 70.15       & 87.2       & \Sec{75.8} & 87.9       & 74.8       & 72.47       & 73.31       & 84.3       & \Fir{80.6} & 87.6       & 89.7       \\
        \DSVT                  & \Sec{78.82} & \Fir{84.90} & \Sec{92.5} & \Sec{86.9}       & 91.5       & \Sec{94.8} & \Fir{76.81} & \Sec{74.58} & \Sec{91.9} & \Fir{76.5} & 88.7       & \Sec{75.9} & \Sec{75.44} & \Sec{76.20} & \Sec{88.2} & \Fir{80.5} & 86.1       & \Sec{89.9} \\
        \TransFusion           & \Fir{79.00} & 82.32       & 89.3       & 86.8       & \Fir{92.7} & \Fir{95.7} & \Sec{76.52} & 69.11       & 84.5       & 75.4       & \Fir{89.9} & \Fir{78.8} & 70.11       & 70.35       & 80.6       & 79.5       & \Fir{90.6} & \Fir{91.1} \\
        \toprule[0.8pt] 
        % \\[-13pt] \toprule[0.8pt]
    \end{tabular}
    }
\end{table}

\noindent\textbf{Implementation details.}
% We replicate LiDAR-based 3D detectors, including \Second{}, \PPL{}, \CP{}, \PartaNet{}, \VoxelRCNN{}, \VoxelNext{}, \PVRCNN{}, and \DSVT{}, on top of OpenPCDet~\cite{openpcdet2020}, and a fusion-based 3D detector TransFusion~\cite{bai2022transfusion}, on top of MMDetection3D~\cite{mmdet3d2020}.
We replicate commonly used LiDAR-based and fusion-based 3D detectors on top of OpenPCDet~\cite{openpcdet2020} and MMDetection3D~\cite{mmdet3d2020}.
All detectors are trained on Waymo Open Dataset (WOD)~\cite{sun2020scalability} with default configurations.
Our training uses the full version of the training set, consisting of 798 sequences with 158,361 samples.
We evaluate these models with the LEVEL 1 mAP weighted by Heading accuracy (\mAPH) and the proposed \MetricS{} on the validation set, which contains 202 sequences with 40,077 samples.
Besides the \MetricS{}, we further present its sub-indicators of stability on confidence (\Conf), localization (\Loc), extent (\Ext), and heading (\Hea).

\noindent\textbf{Relation between \MetricS{} and mAPH.}
\cref{tab:benchmark} presents model results on categories of vehicle, pedestrian, and cyclist.
Models are sorted by the mAPH on the class vehicle, and we highlight the two best performing models in each column.
% We observe that model stability is not always positively related to detection accuracy.
From the results, we find that there is no evident correlation between detection accuracy and model stability.
For instance, TransFusion has the highest mAPH on the class vehicle while its SI is much lower than the LiDAR-based counterparts with similar detection metrics.
That could be because the fusion model improves detection accuracy by additional information from camera images.
The visual information, however, is indirect in inferring precise 3D locations, thereby increasing the detection uncertainty.
% camera images improve model accuracy by additional information, which however is lack of precise 3D locations.
% That defect limits the stable detection of the model.
On the other hand, CenterPoint achieves 73.73 mAPH for vehicle detection, higher than Second and PointPillar.
But it has the lowest \MetricS{} of 80.52 among all detectors.
These results negate definitive positive relations between the two metrics.

\begin{figure}[!t]
    \centering
    \includegraphics[width=1\textwidth]{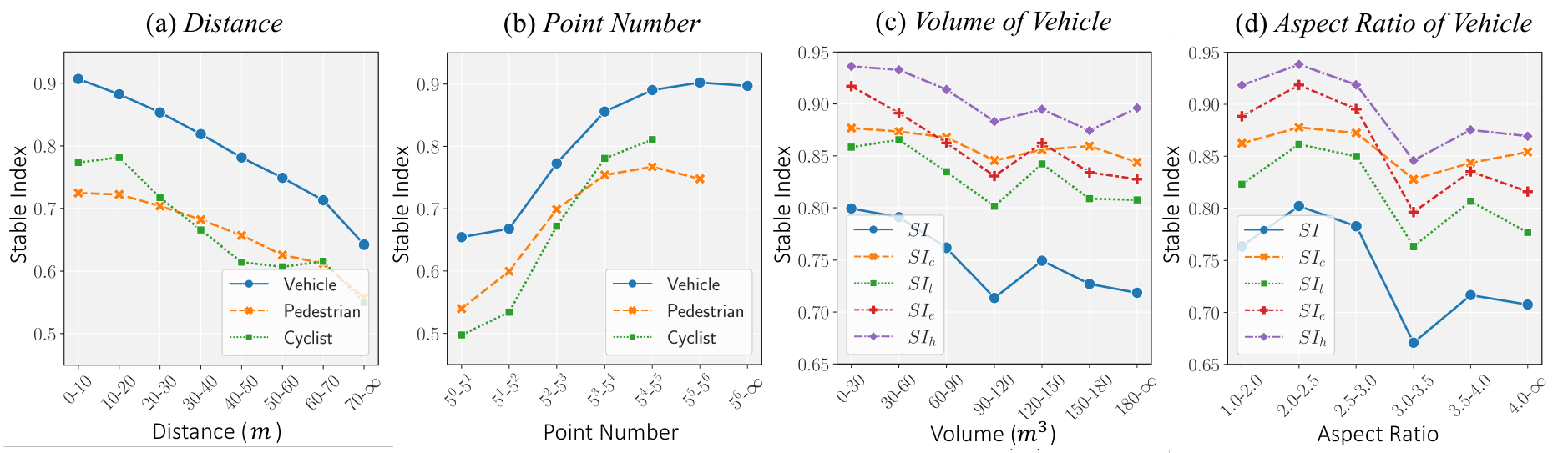}
    \caption{
        Relationships between object properties and detection stability.
    }
    \label{fig:influence}
    % \vspace{-0.2\baselineskip}
\end{figure}

\noindent\textbf{Influence of object properties.}
\cref{fig:influence} shows how various object properties affect detection stability.
We group objects based on specified properties and detect them with CenterPoint.
\cref{fig:influence}(a) presents a negative relationship between detection stability and object distance, where longer distances correspond to harder objects to learn in general.
For all classes, \MetricS{} increases with the number of object points and becomes saturated when the point number reaches $5^3$, as demonstrated in \cref{fig:influence}(b).
\cref{fig:influence} (c) and (d) further explore the effects of object volumes and length-to-width ratios for vehicles.
We find that small vehicles tend to have more stable detection.
Vehicles with length-to-width ratios between 2 and 3 exhibit relatively high \MetricS{} values.
This may be attributed to the prevalence of such vehicles in real-world scenarios.
Vehicles with larger length-to-width ratios, such as trucks/trams/buses, are relatively scarcer in the dataset, and require larger receptive field requirements, making them more unstable in detection.

\begin{table}[t!]
    \small
    \centering
    \caption{Effects of multi-frame strategy on the detection stability. 
    % \#f denotes the number of frames used during training and evaluation.
    }\label{tab:multi-frame}
    \setlength{\tabcolsep}{3pt}
    \scalebox{0.9}{
    \begin{tabular}{cccccccc}
        \toprule[1pt]
        \multirow{2}*{Methods}      & 
        % \multirow{2}*{\#f} 
        \multirow{2}*{Number of frames} 
        & \multicolumn{2}{c}{Vehicle(\%)} & \multicolumn{2}{c}{Pedestrian(\%)} & \multicolumn{2}{c}{Cyclist(\%)}  \\ \cmidrule(r){3-4} \cmidrule(r){5-6} \cmidrule(r){7-8}
                                    &                    & \mAPH & SI                      & \mAPH & SI                         & \mAPH & SI \\  \midrule
        \multirow{3}*{CenterPoint}  & 1                  & 73.73 & 80.52                   & 69.50 & 68.40                      & 71.04 & 68.40  \\
                                    & 2                  & 75.04 & 80.86                   & 75.17 & 70.40                      & 71.23 & 69.39 \\
                                    & 4                  & 75.85 & 81.74                   & 75.38 & 71.69                      & 71.68 & 69.93 \\ \midrule
        \multirow{3}*{PV R-CNN}     & 1                  & 78.33 & 85.17                   & 75.75 & 70.15                      & 72.47 & 73.31 \\
                                    & 2                  & 79.62 & 86.39                   & 80.37 & 73.79                      & 73.66 & 76.78 \\
                                    & 4                  & 80.51 & 87.50                   & 81.12 & 75.32                      & 74.77 & 76.34  \\ \toprule[1pt]
    \end{tabular}
    }
\end{table}

\noindent\textbf{Effects of multi-frame strategy.}
% Due to the issue of sparsity in LiDAR data, a commonly used strategy to address this challenge is the use of multi-frames, which essentially merging several consecutive point clouds as the input.
Merging several consecutive point clouds as one input is a commonly used strategy to address the sparsity in LiDAR data.
\cref{tab:multi-frame} reveals that this scheme not only improves model accuracy but also benefits detection stability.
% In addition to object detection, we find that the multi-frame strategy can also improve the detection stability, as demonstrated in \cref{tab:multi-frame}.
Taking the vehicle for example, using four frames results in notable improvements in the detection accuracy of CenterPoint and PV R-CNN, reaching 75.85 and 80.51 mAPH, which surpass the baseline by 2.12 and 2.18 mAPH.
Meanwhile, the values of \MetricS{} for CenterPoint and PV R-CNN are increased by 1.22 and 2.33, respectively.
This trend is consistent for all classes, illustrating the general effectiveness of the strategy in boosting detection performances of both accuracy and stability.

\noindent\textbf{Summary.}
Our experiments verify that the proposed \MetricS{} is a complementary metric to detection accuracy.
The metric value varies a lot for different model types and demonstrates several interesting patterns \wrt object properties.
We also examine two common-used schemes including data augmentation (in the supplementary) and the multi-frame strategy.
Increasing the degree of data augmentation has a minor impact on detection stability.
Though using multi-frames is proven to be beneficial, it places heavy computational overhead during encoding data into voxel features.
In contrast, our proposed PCL introduces no additional computations during inference, while significantly improving detection stability as illustrated by the later experiments.

\subsection{Experiments on PCL}\label{sec:boosting}
% This section presents experiments on our PCL framework.

\noindent\textbf{Implementation details.}
We employ the widely-used CenterPoint as our base model, which is trained with the default setting in OpenPCDet.
Specifically, we train the model for 36 epochs with the Adam optimizer. 
The one-cycle policy with an initial learning rate 0.003 is used. 
The learning rate gradually increases to 0.03 in the first 40\% epochs and then gradually decreases in the rest of training.

Instead of end-to-end training, we choose to fine-tune the base model with PCL equipped for a few epochs.
Training configuration mirrors that of the end-to-end one, with the exception that the epoch number is reduced to 5 and the learning rate is divided by 10.
It's noteworthy that the scheme not only highly reduces training cost, but also shows how effectively PCL can take effect.

\begin{table}[tb]
    \small
    \centering
    \caption{
        The effects of the proposed PCL.
        ``-'' is the base model and ``w/o PCL'' represents the model fine-tuned without prediction consistency loss.
        % $n$ represents the maximum interval between frame pairs, as described in the context of neighborhood sampling.
    }\label{tab:effectiveness}
    \setlength{\tabcolsep}{8pt}
    \scalebox{0.9}{
    \begin{tabular}{lcccccc}
        \toprule[1pt]
        \multirow{2}*{Methods} & \multicolumn{2}{c}{Vehicle(\%)} & \multicolumn{2}{c}{Pedestrian(\%)} & \multicolumn{2}{c}{Cyclist(\%)}  \\ \cmidrule(r){2-3} \cmidrule(r){4-5} \cmidrule(r){6-7}
                               & \mAPH & SI                      & \mAPH & SI                         & \mAPH & SI \\  \midrule
        -                      & 73.73 & 80.52                   & 69.50 & 68.40                      & 71.04 & 68.40 \\
        w/o PCL                & 73.70 & 80.93                   & 69.55 & 68.35                      & 71.27 & 68.20 \\
        \hline
        PCL ($n=0$)            & 75.57 & 85.42                   & 70.18 & 71.87                      & 70.86 & 68.80 \\
        PCL ($n=4$)            & 75.26 & 85.83                   & 69.56 & 72.76                      & 70.65 & 69.22 \\
        PCL ($n=8$)            & 75.04 & 85.94                   & 68.82 & 72.87                      & 70.31 & 69.32 \\
        PCL ($n=12$)           & 74.64 & 85.93                   & 68.50 & 72.95                      & 70.85 & 69.33 \\ 
        PCL ($n=16$)           & 74.54 & 86.00                   & 67.82 & 73.14                      & 70.25 & 69.16 \\ \toprule[1pt]
    \end{tabular}
    }
\end{table}

\noindent\textbf{Effectiveness of PCL.}
We compare the performances of models fine-tuned with and without PCL, as shown in \cref{tab:effectiveness}.
It can be observed that directly fine-tuning the model has a marginal impact on both model accuracy and stability.
In contrast, when using PCL without cross-frame information involved (\ie, $n=0$), we already achieve SI values of 84.54, 70.95, and 68.80 for vehicle, pedestrian, and cyclist, respectively.
These results reveal significant enhancements, with gains of +4.49, +3.52, and +0.60 compared to the baseline.
For the mAP, we find an interesting phenomenon: the mAP of three classes changes by +1.87, +0.63, and -0.39.
That leads to two valuable conclusions:
(1) Our PCL not only enhances stability but also improves the overall detection accuracy, particularly for the vehicle class.
(2) Regardless of how the mAP changes, 
% whether the mAP increases or decreases, 
the SI is consistently improved, reinforcing that these two metrics assess different model attributes.

\noindent\textbf{Effects of the interval between frame pairs.}
A key hyper-parameter of PCL is the maximum interval $n$ between a pair of frames, as described in the neighborhood sampling.
The larger $n$ becomes, the longer the spans between two frames contrasted by PCL will be.
The results in \cref{tab:effectiveness} show opposing trends in the detection accuracy and stability with $n$ changes.
Take the class vehicle as an example.
When $n=0$, we have the highest mAP of 75.57 and the lowest SI of 85.42 among all PCL models.
The mAP eventually drops to 74.54 as $n$ grows.
On the contrary, the model stability gradually rises to 86.00 SI with $n$ being 16.
This may be because object morphology can change considerably when the frame interval $n$ grows.
Forcefully aligning them can bring damage to model accuracy.
However, such alignment promotes consistent predictions for the same objects, which subsequently leads to stable detection.

\begin{table}[!t]
    \small
    \centering
    \caption{
        Results (\%) on vehicle class with different components in PCL.
        ``C", ``L", ``E", and ``H" denote applying the loss parts relative to confidence, localization, extent, and heading, respectively.
    }\label{tab:ablation}
    \setlength{\tabcolsep}{8pt}
    \scalebox{0.9}{
        \begin{tabular}{ccccc|c|cccc}
            \toprule[1pt]
            \multicolumn{4}{c}{Components}                      & \multirow{2}*{\mAPH} & \multirow{2}*{SI} & \multirow{2}*{\Conf} & \multirow{2}*{\Loc} & \multirow{2}*{\Ext} & \multirow{2}*{\Hea} \\ \cline{1-4}
            C          & L          & E          & H          &                      &                   &                      &                     &                     &                     \\ \midrule
                       &            &            &            & 73.70                & 80.93             & 88.90                & 85.90               & 91.50               & 93.64               \\
            \hline
            \checkmark &            &            &            & 75.64                & 84.04             & 92.28                & 85.80               & 91.44               & 93.65               \\
                       & \checkmark &            &            & 74.15                & 81.80             & 89.34                & 87.16               & 91.79               & 93.73               \\
                       &            & \checkmark &            & 73.86                & 81.73             & 89.14                & 86.13               & 93.37               & 93.67               \\
                       &            &            & \checkmark & 73.44                & 80.89             & 88.88                & 85.64               & 91.47               & 93.85               \\
            \checkmark & \checkmark & \checkmark & \checkmark & 75.57                & 85.42             & 92.81                & 86.78               & 93.31               & 93.90              \\ \toprule[1pt]
        \end{tabular}
    }
\end{table}

\noindent\textbf{Effects of loss components in PCL.}
Our introduced consistency loss comprises components of confidence, localization, extent, and heading.
We examine the model performances with various combinations of these loss components and report the results in \cref{tab:ablation}.
These results show that each part of the loss in PCL can boost the model stability from the corresponding aspect, which confirms the effectiveness of each component in PCL.
The detection accuracy and stability are the highest with all loss parts involved.

The loss component related to confidence score yields the highest improvements in the final \MetricS{}. 
This may be because that the classification loss for training detectors primarily focuses on whether an object is correctly classified, leaving sufficient room for enhancing consistency.
In contrast,  box parameters already have a latent potential for consistent predictions as they all use ground-truth labels as the targets.
Enforcing consistency on these parameters is not as influential as it is on the confidence score. 
Furthermore, we observe that the loss associated with the heading component leads to the least improvement, indicating that maintaining consistency in heading is a challenging task.

\begin{figure}[tb]
    \centering
    \includegraphics[width=1\textwidth]{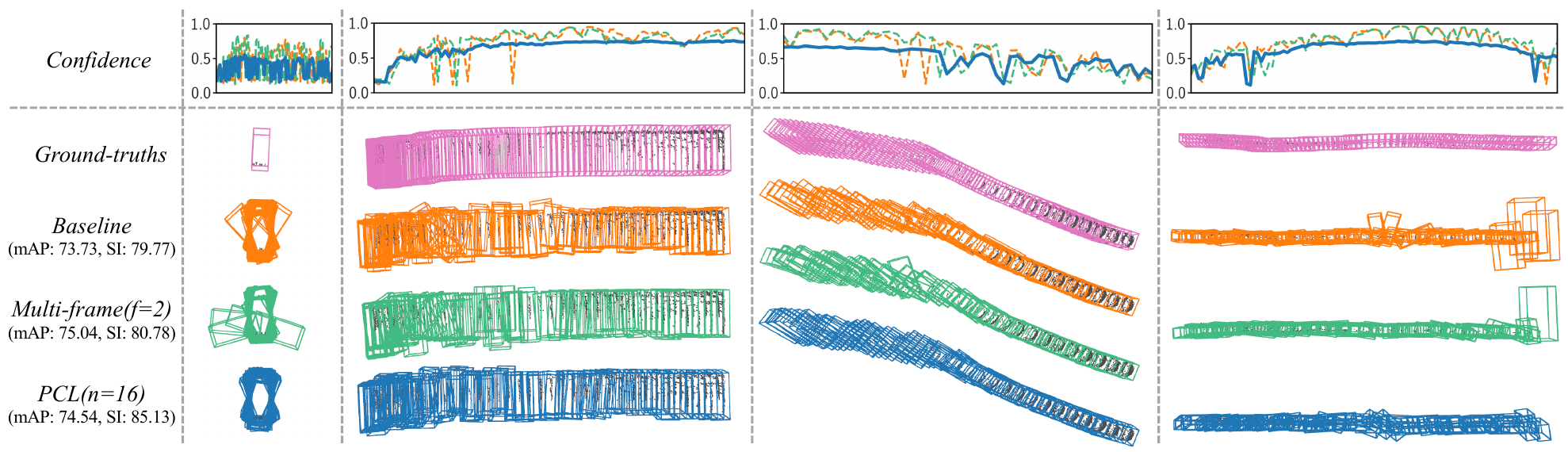}
    \caption{    
        Visualizations of ground-truths (in \textcolor[RGB]{227, 120, 194}{pink}) and predictions of CenterPoint models trained by the baseline (in \textcolor[RGB]{255, 128, 13}{orange}), multi-frame strategy (in \textcolor[RGB]{66, 186, 130}{green}), and PCL strategy (in \textcolor[RGB]{31, 120, 181}{blue}).
        Predicted confidences (top row) and  3D boxes (bottom row) are all 
        presented.
        % Visualizations of predictions from models with different accuracy and stability.
        % %
        % We illustrate ground-truths and the predictions of CenterPoint models trained by the baseline, multi-frame, PCL strategy in \textcolor[RGB]{227, 120, 194}{pink}, \textcolor[RGB]{255, 128, 13}{orange}, \textcolor[RGB]{66, 186, 130}{green}, and \textcolor[RGB]{31, 120, 181}{blue}, respectively.
        %
        % Furthermore, we also display the confidence variations above the predictions.
    }
    \label{fig:visualization}
    % \vspace{-0.2\baselineskip}
\end{figure}

\noindent\textbf{Visualization.}
In \cref{fig:visualization}, we present visualizations of a few ground-truth data and detection results from three distinct models: the baseline CenterPoint, CenterPoint with 2 frames as input, and our PCL model with $n=16$.
In the first row, we plot the trends in confidence scores with time changes and find the PCL model exhibits superior capability in suppressing confidence score fluctuations compared to the other two models.
For the predicted boxes, PCL also has more stable results than other models.
It's noteworthy that our PCL model, despite having a lower mAP compared to the multi-frame version of CenterPoint, significantly outperforms it in terms of \MetricS{}.
This further verifies that detection accuracy and stability capture independent aspects of model performance.
These phenomena all demonstrate the effectiveness of PCL in enhancing detection stability.

%% file: camera_rdy_contents/5-conclusions.tex
\section{Conclusions and Limitations}\label{sec:conclusion}
In this work, we comprehensively study a critical but overlooked issue in object detection, \ie, detection stability.
% Neglecting this issue can potentially lead to hazardous behaviors, particularly in the field of autonomous driving where human safety is closely tied.
% , which is closely tied to human safety. 
For evaluation of such stability, we carefully design a well-proved metric named \MetricL{} (\MetricS{}).
The prediction consistency learning framework is further proposed to enhance model stability.
Our extensive experiments have verified the rationality of \MetricS{} and the effectiveness of the proposed framework.
We hope our work can serve as a reliable baseline and draw the community's attention to this crucial issue in 3D object detection.

To motivate future work, we outline a few limitations based on our current comprehension:
(1) The proposed \MetricS{} focuses solely on the default detection elements for the purpose of generalization.
However, some detectors yield additional predictions such as velocity and attribute.
Integrating the stability of these extra elements into the metric, while maintaining the properties of \MetricS{}, is a practically valuable direction.
(2) In the pursuit of a general baseline approach,  we restrict the design of PCL to be compatible with existing object detectors, avoiding the introduction of extra computations during inference to ensure broad applicability.
Future works may surpass these constraints to explore possibilities for enhanced performance.
% For instance, incorporating extra stabilizer modules holds promise as an option. 
% Moreover, achieving temporal stability by involving multiple or streaming frames presents another potential avenue for exploration.

%% file: camera_rdy_contents/6-supplementary.tex
\appendix

In the supplementary, we first provide comprehensive analyses and theoretical proofs for \MetricS{} in \cref{sec:t}.
\cref{sec:d} shows extra details of \MetricS{}.
In the end, we present extensive experiments (\eg, comparisons of different metrics, analyses on PCL, results in NuScenes benchmark, \etc) in \cref{sec:e}.

\section{Theoretical Analysis}\label{sec:t}
In this section, we provide detailed proofs for the proposed metric \MetricL{}.

\subsection{Proofs for Naive Approaches}
Denote the ground-truth bounding boxes as $B_1^g, B_2^g$, and the predictions as $P_i = \{c_i, B_i\}, i=1, 2$.
The naive approach projects one predicted box onto the location of the second ground truth for the calculation of IoU.
In more detail, we can project $B_1$ onto $B_2^g$ as $\hat B_1 = T_{B_1^g \rightarrow B_2^g}(B_1)$ and then calculate $\iou(\hat B_1, B_2)$, or alternatively, compute it in reverse as $\iou(B_1, \hat B_2)$.
We next prove that the naive approach fails to satisfy the properties of \textit{symmetry} and \textit{marginal unimodality} by the following Property \ref{prop:non-symmetric} and Property \ref{prop:not-unimodal}, respectively.

\begin{property}
    \label{prop:non-symmetric}
    The equality $\iou(\hat B_1, B_2) = \iou(B_{1}, \hat B_2)$ does not always hold.
\end{property}
\begin{proof}
    When considering the reverse projection, we can derive that 
    \begin{equation*}
      \small
      \begin{split}
      & \iou(\hat B_1, B_2)  = \iou(T_{B_1^g \rightarrow B_2^g}(B_1), B_2) \\
      & =\iou(T_{B_1^g \rightarrow B_2^g}(B_1), T_{B_1^g \rightarrow B_2^g}( T_{B_2^g \rightarrow B_1^g}(B_2))) \\
      % & =\iou(T(B_1^g, B_2^g)\otimes B_1, T(B_1^g, B_2^g)\otimes (T(B_2^g, B_1^g)\otimes B_2)) \\
      & = \iou(T_{B_1^g \rightarrow B_2^g}(B_1), T_{B_1^g \rightarrow B_2^g}(\hat B_2))
      \end{split}
    \end{equation*}
    If $\iou(\hat B_1, B_2) = \iou(B_{1}, \hat B_2)$ always hold, it will imply
    \begin{equation*}
      \small
        \iou(B_1, \hat B_2) = \iou(\hat B_1, B_2) = \iou(T_{B_1^g \rightarrow B_2^g}(B_1), T_{B_1^g \rightarrow B_2^g}(\hat B_2))
    \end{equation*}
    This would suggest that any arbitrary projection $T_{B_1^g \rightarrow B_2^g}(\cdot)$ does not alter the \iou{} of two boxes.    

    Nevertheless, this assumption can be refuted with a straightforward example. 
    Imagine two adjacent squares initially possessing an IOU of 0. 
    However, upon rotating the squares, an intersection is formed, leading to an IOU value greater than 0. 
    This evident contradiction highlights that the proposition $\iou(\hat B_1, B_2) = \iou(B_1, \hat B_2)$ does not hold universally.
\end{proof}

\begin{property}
    \label{prop:not-unimodal}
    The $\iou(\hat B_1, B_2)$ is not marginal unimodal concerning the box elements.
\end{property}
\begin{proof}
    To illustrate this, consider the following example: 
    Let a Box be defined $\{0, 0, 0, 2, 1, 1, 0\}$, and another Box as $\{\delta_x, 0.05, 0, 3, 1, 1, \delta_{\theta}\}$.
    Here, we restrict $\delta_x\in \{0, 0.25, 0.5\}$ and $\delta_{\theta}\in [-10^\circ, 10^\circ]$. 
    The \iou{} curve between the two boxes is shown in \cref{fig:curve1}.

    \begin{figure}[h]
        \centering
        \includegraphics[trim={0pt 0pt 0pt 0pt},clip, width=0.6\textwidth]{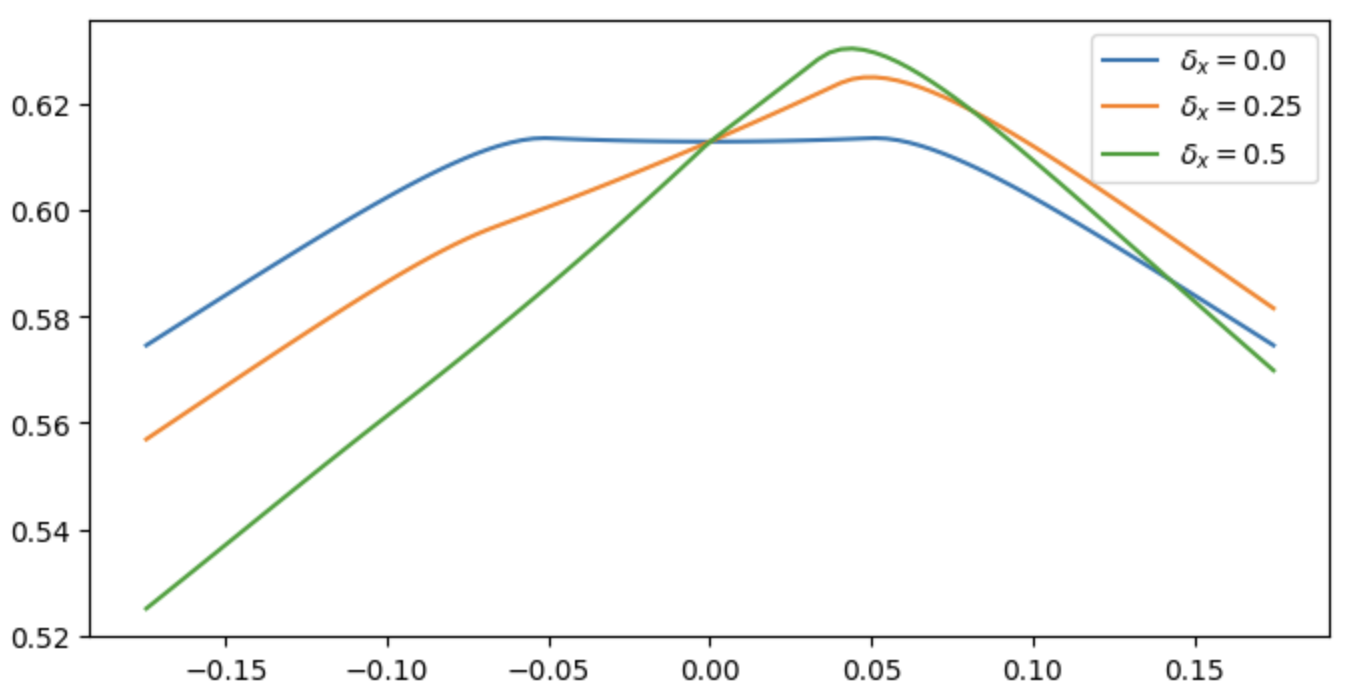}
        \caption{
            The IoU curves \wrt the $\delta_{\theta}$ variation.
            }
        \label{fig:curve1}
      \end{figure}    
    Upon examining the plot, it is evident that each curve is not centered at $\delta_{\theta}=0$.
    This suggests that a less stable prediction on the box angle results in a higher \iou{}. 
    The same conclusion can be drawn for $\delta_x$. Therefore, the \iou{} function is not marginally unimodal.
\end{proof}

\subsection{Proofs for \MetricL{}}

It can be readily inferred from the formulation of \MetricS{} that it encompasses all detection components and consolidates them into a unified metric (the weighted \iou{}). 
Next, we proceed to prove the properties of \textit{Symmetry} and \textit{Marginal Unimodality}.

\begin{property}
    \label{prop:Symmetry}
    The proposed \textit{Stable Index} is symmetric.
  \end{property}  
  \begin{proof}
    The proposed score $\delta_c = 1 - |c_1 - c_2|$ is symmetric, given that we take the absolute value of the difference.
    Additionally, the metric essentially computes the \iou{} with an intermediate pivot box,  ensuring consistent values with changes in frame order.
    Therefore, the final metric is symmetric.
  \end{proof}
  
  Next, we demonstrate that the proposed metric adheres to the principle of \textit{Marginal Unimodality}. 
  Before delving into the proof, we introduce a few lemmas.
  
  \begin{lemma}
    \label{lemma:center}
    Denote the \iou{} value between two boxes $\{x_1, y_1, z_1, l, w, h, \theta\}$ and $\{x_2, y_2, z_2, l, w, h, \theta\}$ as $F(x, y, z)$, where $x=x_2-x_1, y=y_2-y_1, z=z_2-z_1$.
    Then $F(x, y, z)$ is marginal unimodal \wrt $x, y, z$.
  \end{lemma}
  \begin{proof}
    We can observe that the IOU value is 0 if $|x| > l/2$, $|y| > w/2$ or $|z| > h/2$.
    Otherwise, when none of these conditions are met, we can calculate the intersection volume as $V_{int} = (l - |x|) \cdot (w - |y|) \cdot (h - |z|)$. 
    This leads us to the IOU value equation: 
    \begin{equation*}
      \small
      F(x, y, z) = \frac{V_{int}}{2lwh - V_{int}}  = 1\ / \ (\frac{2lwh}{(l-|x|)(w-|y|)(h-|z|)} - 1).
    \end{equation*}
  
    With $y$ and $z$ fixed, $F(x, y, z)$ is monotonically decreasing with $x>0$ and monotonically increasing with $x<0$. 
    Therefore, $F(x, y, z)$ is unimodal with respect to $x$ when $y$ and $z$ are fixed. 
    Similar conclusions can be derived for $y$ and $z$. 
    In summary, $F(x, y, z)$ exhibits marginal unimodality with respect to $x, y, z$.
  \end{proof}

  \begin{lemma}
    \label{lemma:size}
    Denote the \iou{} value between two boxes $\{0, 0, 0, l_1, w_1, h_1, 0\}$ and $\{0, 0, 0, l_2, w_2, h_2, 0\}$ as $F(l, w, h)$, where $l=l_2/l_1, w=w_2/w_1, h=h_2/h_1$.
    Then $F(l, w, h)$ is marginal unimodal \wrt $l, w, h$.
  \end{lemma}
  \begin{proof}
    The intersection volume  of the two boxes is $V_{int} = \min(l_1, l_2)\cdot \min(w_1, w_2) \cdot \min(h_1, h_2)$. 
    The \iou{} value is then $F(l, w, h) = V_{int}/ (l_1w_1h_1+l_2w_2h_2 - V_{int})$.
  
    Let's first prove that $F(l, w, h)$ is marginal unimodal \wrt $l$ with $w, h$ fixed.
    We examine the reciprocal of the $F(l, w, h)$ as 
    \begin{equation*}
      \small
      \frac{1}{F(l, w, h)} = \frac{l_1w_1h_1+l_2w_2h_2}{\min(l_1, l_2)\cdot \min(w_1, w_2) \cdot \min(h_1, h_2)} - 1.
    \end{equation*}

    If $l = l_2/l_1 > 1$, then we can have
    \begin{equation*}
      \small
      \begin{split}
        \frac{1}{F(l, w, h)} &= \frac{l_1w_1h_1+l_2w_2h_2}{\min(l_1, l_2)\cdot \min(w_1, w_2) \cdot \min(h_1, h_2)} - 1,\\
        &= \frac{w_1h_1 + l w_2h_2}{\min(w_1, w_2) \cdot \min(h_1, h_2)} - 1.
      \end{split}
    \end{equation*}
    Then it's easy to derive that the first term monotonically increases with $l$, leading to $F(l, w, h)$ being monotonically decreasing with $l$.

    Conversely, if $l = l_2/l_1 < 1$, the reciprocal of the $F(l, w, h)$ becomes
    \begin{equation*}
      \small
      \begin{split}
        \frac{1}{F(l, w, h)} &= \frac{l_1w_1h_1+l_2w_2h_2}{\min(l_1, l_2)\cdot \min(w_1, w_2) \cdot \min(h_1, h_2)} - 1,\\
        &= \frac{lw_1h_1 + w_2h_2}{\min(w_1, w_2) \cdot \min(h_1, h_2)} - 1.
      \end{split}
    \end{equation*}
    Similarly, we can reach that $F(l, w, h)$ monotonically increases with $l$.

    In conclusion, $F(l, w, h)$ is proven to be marginally unimodal with respect to $l$ with $w, h$ fixed. 
    This proof can be extended to demonstrate that $F(l, w, h)$ is also marginally unimodal with respect to $w$ and $h$.
    \end{proof}
    
%   \begin{lemma}
%     \label{lemma:angle}
%     Denote two boxes $\{x, y, z, l, w, h, \theta_1\}$ and $\{x, y, z, l, w, h, \theta_2\}$, and the IOU as $F(\theta)$ where $\theta = \theta_1 - \theta_2$.
%     $F(\theta)$ is nearly unimodal \wrt  $\theta \in [0, \pi/2)$ when  $\max(l, w)/\min(l, w) > 2$.
%   \end{lemma}
%   \begin{proof}
%     \cref{fig:curve2} plots $F(\theta)$ curves with different ratios of  $\max(l, w)/\min(l, w)$ when two boxes are the same besides the rotation angle.
%     We can observe that the curves are nearly unimodal \wrt $\theta$ when $\max(l, w)/\min(l, w) > 2$.
%   \end{proof}

\begin{lemma}
    \label{lemma:angle}
    If two 3D boxes are identical except for their heading values, then their \iou{} is not unimodal with respect to the angle difference $\Delta\theta$. 
    However, within the range $|\Delta\theta| \leq \pi/4$, the \iou{} is an unimodal function.
\end{lemma}
\begin{proof}
    We encountered challenges in establishing a mathematical proof for this assertion, prompting us to turn to experimental results for validation. 
    Specifically, we generate a curve plotting the \iou{} against the angle difference for various length-to-width ratios. 
    The graph depicted in \cref{fig:curve} serves as empirical evidence supporting our claim.
    
    \begin{figure}[h]
        \centering
        \includegraphics[width=0.6\textwidth]{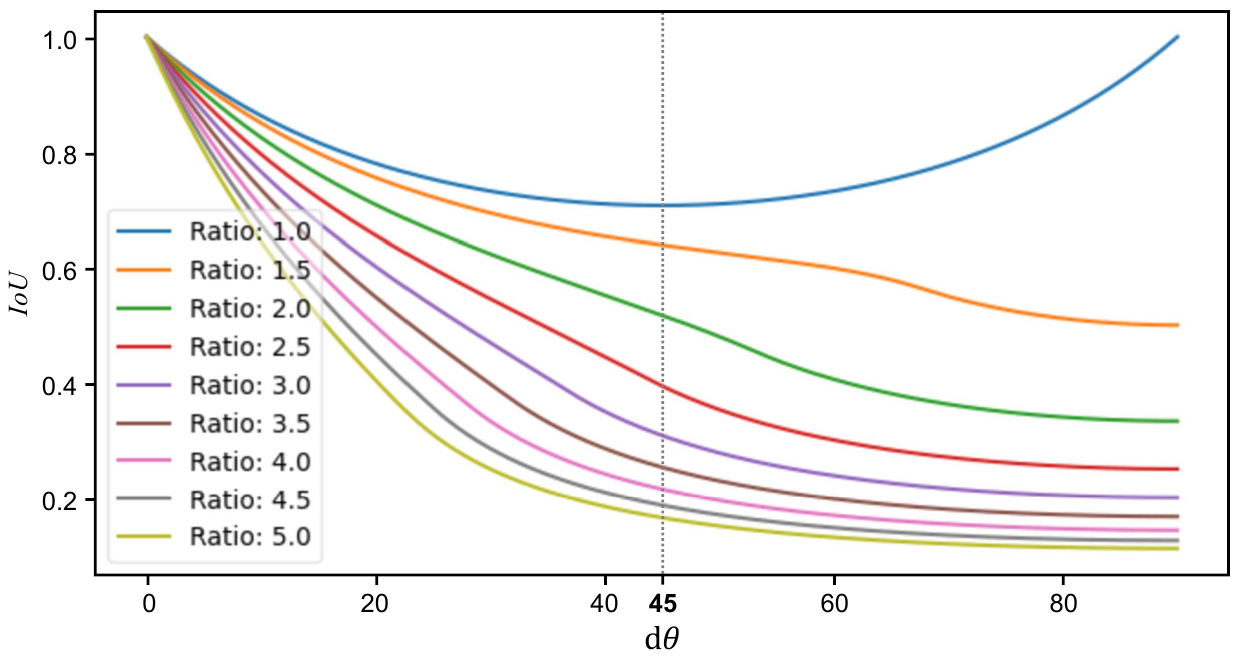}
        \setlength{\abovecaptionskip}{3pt}
        \caption{
          The \iou{} curves with respect to $\theta$ divergence where two boxes share the same centers and sizes.
          As illustrated, when the $\Delta\theta$ divergence is less than $\pi/4$, the \iou{} consistently exhibits a monotonic behavior.
        }
        \label{fig:curve}
        \vspace{-0.2\baselineskip}
    \end{figure}
\end{proof}

Our \MetricL{} is defined as 
\begin{equation}
  \small
  SI = SI_c \times (SI_l + SI_e + SI_h)/3.
\end{equation}
The previous \cref{lemma:size,lemma:center,lemma:angle} essentially validate the $SI_l, SI_e, SI_h$ is marginally unimodal.
Next, we prove that the final \MetricS{} is also marginally unimodal.
  
  \begin{property}
    \label{prop:unimodal}
    The proposed \MetricL{} (\MetricS{}) is marginal unimodal \wrt the disparities of all prediction elements including the prediction score, box center, box size and box heading.
  \end{property}
  \begin{proof}
    For the $SI_c = \max(0, 1 - |c_1 - c_2|/(c^{0.99}-c^{0.01}))$, we can easily conclude that $SI_c$ monotonically non-decreases as the score discrepancy $|c_1 - c_2|$ decreases.
    That means $SI_c$ is unimodal \wrt $c_1 - c_2$.
    $SI_l, SI_e, SI_h$ is also marginally unimodal according to \cref{lemma:size,lemma:center,lemma:angle}.

    As $SI_c, SI_l, SI_e, SI_h$ are all non-negative and each variable is only associated with one of prediction score, box center, box size and box heading, it's easy to derive that $SI = SI_c \times (SI_l + SI_e + SI_h)/3$ is marginal unimodal \wrt all elements.
  \end{proof}

  Our final proof is about the maximum value of the proposed metric:
  \begin{lemma}
    \label{lemma:peak}
    \MetricL{} reaches the peak value of 1 if and only if the predictions are perfectly stable.
  \end{lemma}
  \begin{proof}
      Since $SI_c, SI_l, SI_e, SI_h$ are in the range of [0, 1], achieving $SI=1$ implies that all values are 1. 
      This condition is met when the scores and all elements of the bounding boxes, as transformed by the defined operations, are identical.

      Conversely, if the predictions are perfectly stable, meaning that all \iou{}s are 1 and $SI_c =\max(0, 1 - |c_1 - c_2|/(c^{0.99}-c^{0.01}))= 1$, we can deduce that $SI = 1$.
  \end{proof}

\section{Extra Details in Stable Index}\label{sec:d}
SI essentially evaluates the stability of an object across two consecutive frames. 
In the procedure of matching objects, there are two corner cases which are handled:
(1) If an object is observed and labeled in just one frame, this case is disregarded as it doesn't form a valid object pair. 
(2) If the object exists in both frames but the Hungarian algorithm fails to find two predictions, the SI value is set to 0.

We define the consecutive frames as two frames with a time interval of $\Delta_t$.
In our implementation, we set $\Delta_t$ to 0.5s.
For a trajectory of length $N$, there will be $N-\Delta_t/d$ object pairs, where $d$ is the time interval for capturing data points.
We opt not to consider all object pairs, as we deem stability more meaningful within the context of short time intervals.
The computation of SI is efficient. 
On a machine with an A6000 GPU and an 8352Y CPU, calculating SI takes $\sim$2 mins, much faster than computing mAP ($>$30 mins) and MOTA ($\sim$12 mins).

\section{Extra Experiments}\label{sec:e}
This section presents our additional experiments.

\subsection{Effects of data augmentation.}\label{sec:e1}
\begin{table}[tb]
  \centering
  \caption{
      Effects of augmentation on detection stability. "Trans" and "Drop" mean applying random translation and random point dropping in training, respectively. Param.: parameter.
  }\label{tab:augmentation}
  \setlength{\tabcolsep}{4pt}
  \scalebox{0.9}{
      \begin{tabular}{cccccccc}
          \toprule[1pt]
          \multirow{2}*{Aug}   & \multirow{2}*{Param.} & \multicolumn{2}{c}{Vehicle(\%)} & \multicolumn{2}{c}{Pedestrian(\%)} & \multicolumn{2}{c}{Cyclist(\%)}  \\ \cmidrule(r){3-4} \cmidrule(r){5-6} \cmidrule(r){7-8}
                               &                      & \mAPH & SI                      & \mAPH & SI                         & \mAPH & SI \\ \midrule
          -                    & -                    & 73.73 & 79.77                   & 69.50 & 67.43                      & 71.04 & 68.48  \\ \midrule
          \multirow{3}*{Trans} & 0.1                  & 73.84 & 79.92                   & 70.07 & 67.90                      & 65.10 & 67.55  \\
                               & 1                    & 73.90 & 79.98                   & 69.68 & 67.99                      & 64.97 & 67.57  \\
                               & 10                   & 73.77 & 79.93                   & 70.47 & 67.87                      & 64.53 & 67.71  \\ \midrule
          \multirow{3}*{Drop}  & 20\%                 & 74.02 & 79.97                   & 70.38 & 68.02                      & 65.19 & 68.10  \\
                               & 40\%                 & 73.94 & 80.06                   & 70.22 & 67.93                      & 65.16 & 68.35  \\
                               & 60\%                 & 74.01 & 80.05                   & 69.80 & 67.83                      & 65.49 & 67.89  \\ \toprule[1pt]
      \end{tabular}  
  }
\end{table}

Data augmentation is a commonly used technique to enhance model robustness against variations in the dataset. 
We examine whether introducing more augmentation will enhance model stability.
In addition to basic augmentations, we incorporated random translation and point dropping during the training of CenterPoint~\cite{yin2021center}.
For each augmentation, we selected three different scales to ensure experiment universality. 
The results are presented in \cref{tab:augmentation}.

Despite the increased augmentation scale, the changes in mAPH and SI are marginal. 
Notably, the model achieves its highest stability in vehicle detection at 80.06 when randomly dropping 40\% of points during training, which is only 0.29 higher than the baseline. 
This indicates that the application of augmentation offers limited influences in improving model stability.

\subsection{Comparisons of Different Metrics}

MAP and the proposed \MetricS{} are two metrics for object detectors. 
In the main text, we have demonstrated that these metrics capture different properties of the detection results.
It's also an interesting question how \MetricS{} relates to tracking metrics such as MOTA/MOTP, as they all somehow capture temporal information.
In this part, we provide more analyses on this question.

\MetricL{} and tracking metrics differ in the following aspects:
(1) Tracking metrics primarily assess object trackers instead of directly evaluating detectors.
Consequently, their values can be highly influenced by the effectiveness of the tracking modules.
In contrast, \MetricS{} serves as a detection metric.
(2) Tracking metrics concentrate more on the long-term tracking performances, while \MetricS{} is designed to capture short-term properties as the stability is more meaningful within the context of short time intervals.
(3) Tracking metrics emphasize whether objects are well-tracked while disregarding the inconsistency across frames.
As a result, they can exhibit different patterns compared to the proposed \MetricS{}.
\cref{fig:mota_si} shows some toy examples that demonstrate the lack of correlation between MOTA and \MetricS{}.

\begin{figure}[t!]
    \centering
    \includegraphics[width=0.8\textwidth]{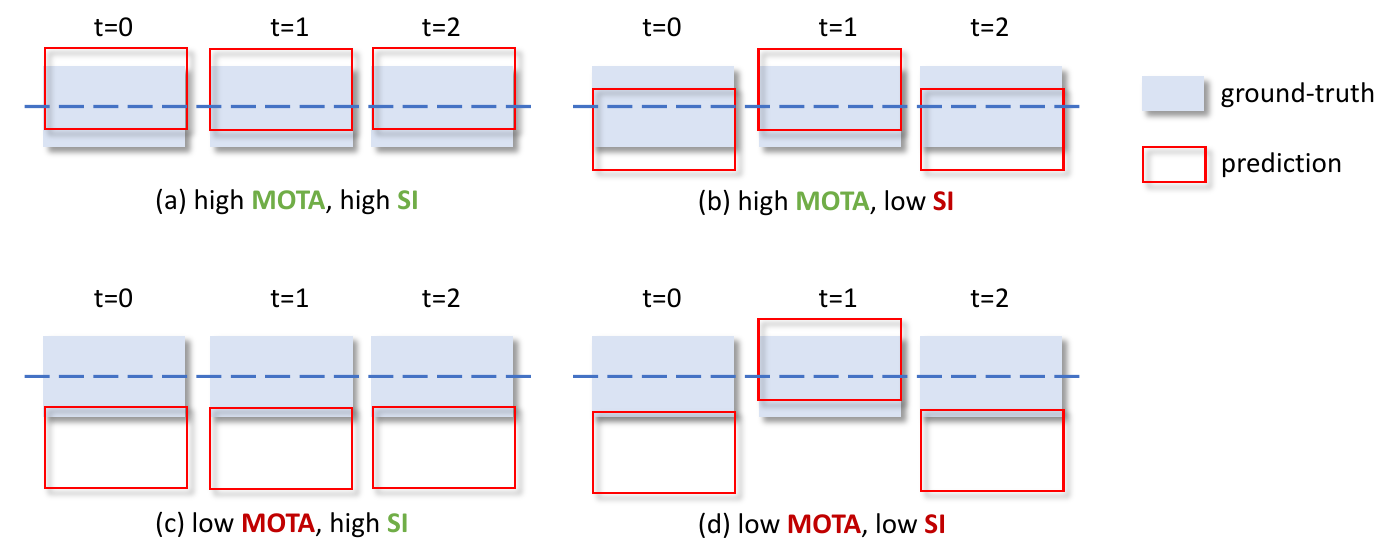}
    \caption{MOTA and \MetricS{} can have distinct patterns for different trajectories.
    }\label{fig:mota_si}
\end{figure}

We further provide experimental comparisons of these metrics, as presented in \cref{tab:metric_comp}.
Object tracking was performed using SimpleTrack~\cite{pang2021simpletrack} with default settings.
It can been observed that there is no clear correlation between \MetricS{} and MOTA.
For example, Second has a much higher \MetricS{} value than CenterPoint (81.37 \vs 80.52). 
However, the MOTA of Second is 1.03 lower than that of CenterPoint.
Notably, PV R-CNN++ achieves the best tracking results, while lagging behind VoxelNet and DSVT in terms of \MetricS{}.

\begin{table}[t!]
    % \small
    \centering
    \caption{Metric values on class vehicle for different object detectors.
    The tracking results for MOTA/MOTP are generated by SimpleTrack with default setting.
    }\label{tab:metric_comp}
    \setlength{\tabcolsep}{12pt}
    \scalebox{0.9}{
        \begin{tabular}{lcccc}
            \toprule[1pt]
            Methods     & \mAPH $\uparrow$        & SI $\uparrow$         & MOTA $\uparrow$        & MOTP $\downarrow$    \\ \hline
            \Second & 72.60 & 81.37 & 53.77 & 17.25 \\
            \CPPPL & 72.82 & 80.61 & 53.10 & 16.73 \\
            \PPL & 72.84 & 80.84 & 53.59 & 17.23 \\
            \CP & 73.73 & 80.52 & 54.80 & 16.56 \\
            \PartaNet & 75.02 & 82.86 & 58.40 & 16.46 \\
            \PVRCNN & 75.92 & 83.73 & 59.17 & 16.56 \\
            \VoxelRCNN & 77.19       & 84.26       & \Sec{59.78} & 16.53 \\
            \VoxelNext   & 77.84       & \Sec{84.82} & 59.19       & 16.49 \\
            \PVRCNNPP  & \Sec{77.88} & 84.49       & \Fir{60.43} & \Fir{16.34}       \\
            \DSVT        & \Fir{78.82} & \Fir{84.90} & 59.64       & \Sec{16.48}       \\ 
            \toprule[1pt] 
        \end{tabular}
    }
\end{table}

\subsection{Analysis of Performance Enhancements with PCL}
In addition to enhancing detection stability, our proposed PCL framework demonstrates evident performance improvements, particularly for the vehicle class. 
To delve into the reasons behind mAP boosts, we analyze object recalls under precision 0.6, as depicted in \cref{tab:pr}. 
It can been seen that the recalls improves especially for the 
infrequent and hard objects with large length-to-width ratios.
This indicates that encouraging prediction consistency, rather than benefiting easy cases, contributes to greater gains in hard scenarios.

\begin{table}[t!]
    % \small
    \centering
    \caption{The vehicle recalls for different length-to-width ratios under precision 0.6 .}\label{tab:pr}
    \setlength{\tabcolsep}{5pt}
    \scalebox{0.9}{
        \begin{tabular}{clll}
            \toprule[1pt]
            Length-to-width ratio (LWR)    & $0\sim3$  & $3\sim4$  & $4\sim\inf$ \\ 
            \hline
            CenterPoint (mAP: 73.73)  & 59.15 & 42.61 & 36.48\\
            + PCL (n=16) (mAP: 74.54) & 59.13 \improvea{-0.02} & 44.14 \improveb{1.53} & 39.74 \improveb{3.26}\\ 
            \toprule[1pt] 
        \end{tabular}
    }
\end{table}

\subsection{The Effects of PCL on DSVT}
\begin{table}[t]
  \centering
  \caption{
      The performances of DSVT models.
  }\label{tab:dsvt}
  % \small
  % \setlength{\tabcolsep}{1.5pt}
  \renewcommand{\arraystretch}{1.2}
  \scalebox{0.76}{
  \begin{tabular}{l >{\columncolor{lightgray!40}}c >{\columncolor{lightgray!40}}c|cccc >{\columncolor{lightgray!40}}c >{\columncolor{lightgray!40}}c| cccc >{\columncolor{lightgray!40}}c >{\columncolor{lightgray!40}}c| cccc}
      \toprule[0.8pt] 
      % \\[-13pt] \toprule[0.8pt]
      \multirow{2}*{Methods} & \multicolumn{6}{c}{Vehicle(\%)}                                                   & \multicolumn{6}{c}{Pedestrain(\%)}                                                              & \multicolumn{6}{c}{Cyclist(\%)}                                   \\ \cmidrule(r){2-7} \cmidrule(r){8-13} \cmidrule(r){14-19}
                             & \mAPH\hspace*{-2.4pt}       & SI          & \Conf       & \Loc        & \Ext        & \Hea        & \mAPH\hspace*{-2.4pt}        & SI          & \Conf       & \Loc        & \Ext        & \Hea        & \mAPH\hspace*{-2.4pt}       & SI          & \Conf       & \Loc        & \Ext        & \Hea       \\ \toprule[0.9pt]
      \DSVT                  & 78.82 & 84.90 & 92.5 & 86.9       & 91.5       & 94.8 & 76.81 & 74.58 & 91.9 & 76.5 & 88.7       & 75.9 & 75.44 & 76.20 & 88.2 & 80.5 & 86.1  & 89.9 \\
      \hline
      w/o PCL & 78.79 & 85.04 & 92.5 & 87.0 & 91.6 & 95.0 & 76.65 & 74.72 & 91.8 & 76.5 & 88.6 & 76.5 & 75.42 & 76.13 & 88.0 & 80.4 & 86.2 & 90.3\\
      w/ PCL & 78.84 & 85.81 & 93.1 & 87.1 & 92.3 & 95.1 & 76.69 & 75.94 & 92.4 & 76.5 & 90.2 & 77.3 & 75.34 & 77.06 & 88.7 & 80.4 & 87.1 & 90.4\\
      \toprule[0.8pt] 
  \end{tabular}
  }
\end{table}

In the main text, we apply the PCL framework to the popular CenterPoint model~\cite{yin2021center}. 
To validate the generality of our PCL framework across different detectors, we implement the PCL on the transformer-based DSVT model and present the results in \cref{tab:dsvt}.
If fine-tuning process without  prediction consistency loss, we observe a slight drop in mAP, while the \MetricS{} shows a modest increase. 
The overall performance of DSVT before and after fine-tuning does not exhibit significant differences. 
In contrast, our PCL aids DSVT in maintaining detection performance after fine-tuning and substantially increases the \MetricS{}.
The \MetricS{} is boosted by 0.91, 1.36, and 0.86 for the vehicle, pedestrian, and cyclist classes, respectively. 
These results demonstrate the efficacy and generality of the proposed PCL framework.

The sub-indicators of the \MetricS{} offer insights into the specific aspects contributing to the stability improvements.
From \cref{tab:dsvt},we observe enhancements in the stability of confidence scores and box extents, while the improvements in the stability of the other two elements are comparatively less pronounced.
This phenomenon diverges from the behavior observed in CenterPoint, where the stability of all elements experiences a significant boost.
One obvious explanation is that DSVT outperforms CenterPoint in terms of detection performance.
Another possible reason is that transformer-based feature extractor 
aligns better with the sparse nature of point clouds compared to CNN-based approaches.
Consequently, the transformer-based model is capable of generating more stable estimations for heading and localization.

\subsection{Analysis of Offline Auto-labelling Methods}
Recently, offline auto-labeling methods~\cite{qi2021offboard,fan2023once,ma2023detzero} have achieved exciting performances, surpassing even human labels. 
We utilize our \MetricS{} to analyze how can these auto-labeling methods improve detection stability. 
In \cref{tab:ctrl}, the results of a 16-frames detector before and after using CTRL~\cite{fan2023once} are presented, showcasing a substantial improvement in detection stability from 89.90 to 93.38.
Some other interesting findings include:
(1) Box localization and extent exhibit the most significant stability improvements. 
Confidence scores also display increased stability after the offline refinements. 
However, box heading shows the lowest improvements, indicating that heading stability is the most challenging aspect for enhancements.
(2) Heading stability is enhanced only for objects farther than 50m. 
This suggests that objects within 50m may already have sufficiently accurate heading estimations.
(3) The stability improvement is positively correlated with object distance, aligning with intuition as there is more room for optimization for farther objects.

\begin{table}[t!]
    % \small
    \centering
    \caption{The detection stability before and after the detector uses the offline auto-labelling method.}\label{tab:ctrl}
    \setlength{\tabcolsep}{8pt}
    \scalebox{0.9}{
        \begin{tabular}{ll|c|cccc}
            \toprule[1pt]
            Method   & Breakdown   & SI    & \Conf{} & \Loc{} & \Ext{} & \Hea{} \\ 
            \hline
            \multirow{4}*{Before CTRL} & \Sec{Overall} & \Sec{89.90} & \Sec{94.67}  & \Sec{92.47}  & \Sec{95.76}  & \Sec{95.26}  \\
            &$[0m, 30m)$ & 95.48 & 98.35 & 95.49 & 97.05 & 98.35 \\
            &$[30m, 50m)$ & 90.60 & 95.19 & 92.83 & 95.91 & 95.62 \\
            &$[50m, \inf)$ & 83.23 & 90.22 & 88.68 & 94.24 & 91.59 \\
            \hline
            \multirow{4}*{After CTRL}  & \Sec{Overall} & \Sec{93.38} & \Sec{96.84}   & \Sec{95.36}  & \Sec{97.60}  & \Sec{95.45}  \\
            &$[0m, 30m)$ & 96.77 & 98.86 & 96.94 & 98.30 & 98.23 \\
            &$[30m, 50m)$ & 93.78 & 97.05 & 95.69 & 97.74 & 95.69 \\
            &$[50m, \inf)$ & 89.44 & 94.53 & 93.37 & 96.74 & 92.31\\
            \toprule[1pt] 
        \end{tabular}
    }
\end{table}

\subsection{Visualizations}
% \begin{figure}[t]
%   \centering
%   \includegraphics[trim={0pt 0pt 0pt 0pt},clip, width=0.45\textwidth]{figures/supp_trackings/supp_vis1.pdf}
%   \includegraphics[trim={0pt 0pt 0pt 0pt},clip, width=0.45\textwidth]{figures/supp_trackings/supp_vis2.pdf}
%   \includegraphics[trim={0pt 0pt 0pt 0pt},clip, width=0.45\textwidth]{figures/supp_trackings/supp_vis3.pdf}
%   \includegraphics[trim={0pt 0pt 0pt 0pt},clip, width=0.45\textwidth]{figures/supp_trackings/supp_vis4.pdf}
%   \includegraphics[trim={0pt 0pt 0pt 0pt},clip, width=0.45\textwidth]{figures/supp_trackings/supp_vis5.pdf}
%   \includegraphics[trim={0pt 0pt 0pt 0pt},clip, width=0.45\textwidth]{figures/supp_trackings/supp_vis6.pdf}
%   \caption{
%       Visualizations of the object tracking results using the detections predicted from the baseline and PCL models.
%       }
%   \label{fig:tracking}
% \end{figure}  

\begin{figure}[t]
    \centering
    \includegraphics[trim={0pt 0pt 0pt 0pt},clip, width=0.23\textwidth]{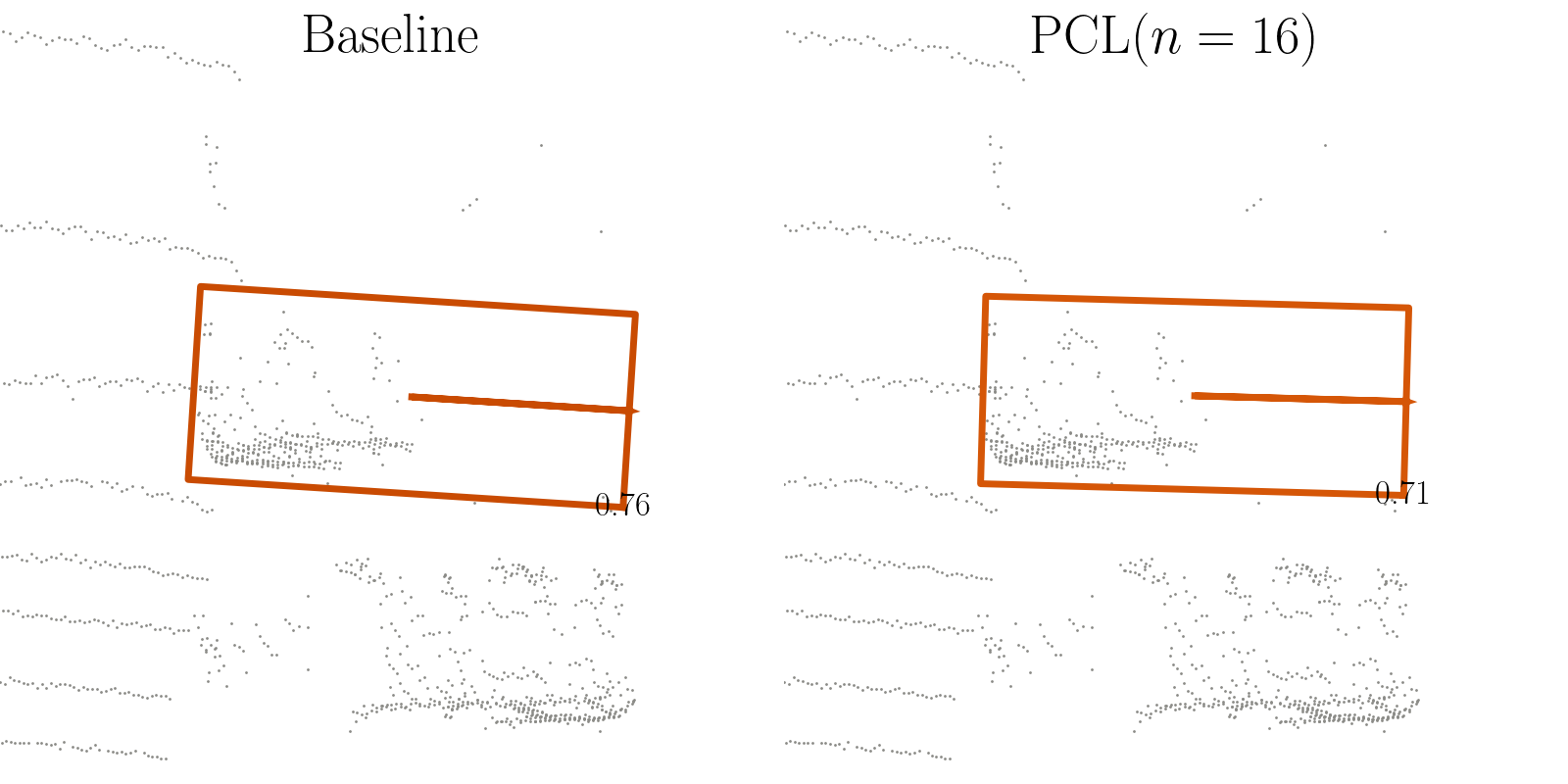}
    \includegraphics[trim={0pt 0pt 0pt 0pt},clip, width=0.23\textwidth]{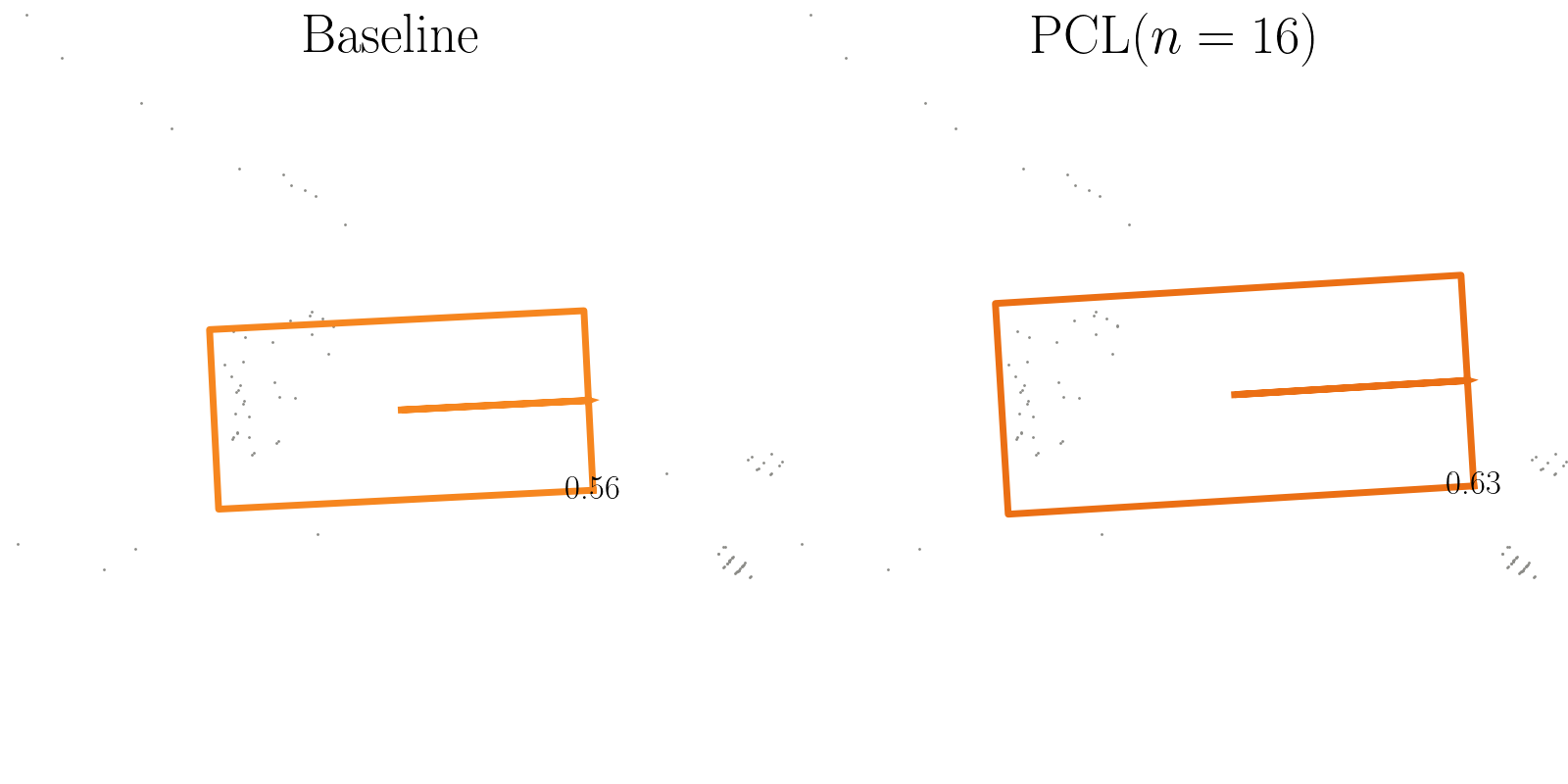}
    \includegraphics[trim={0pt 0pt 0pt 0pt},clip, width=0.23\textwidth]{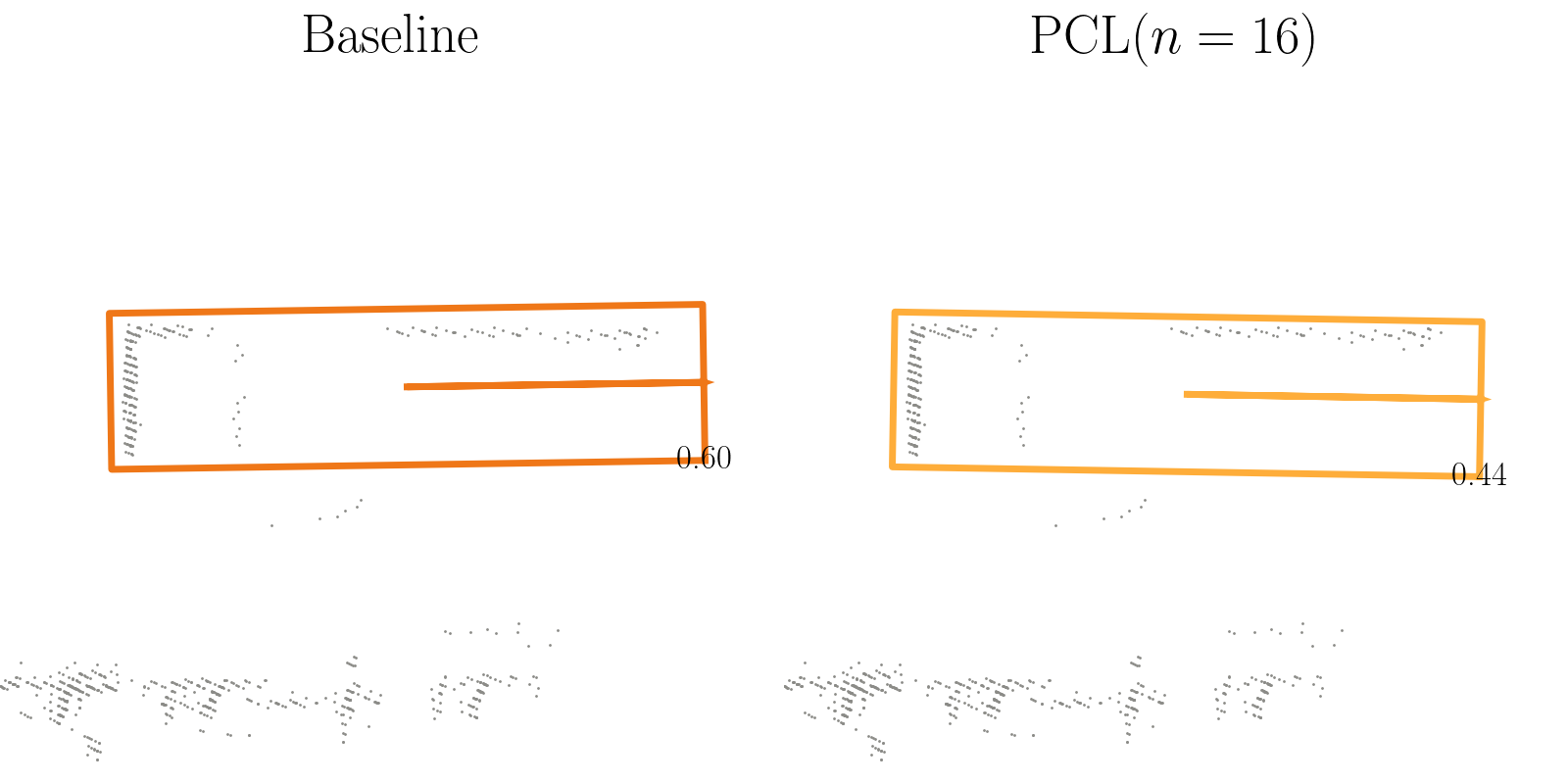}
    \includegraphics[trim={0pt 0pt 0pt 0pt},clip, width=0.23\textwidth]{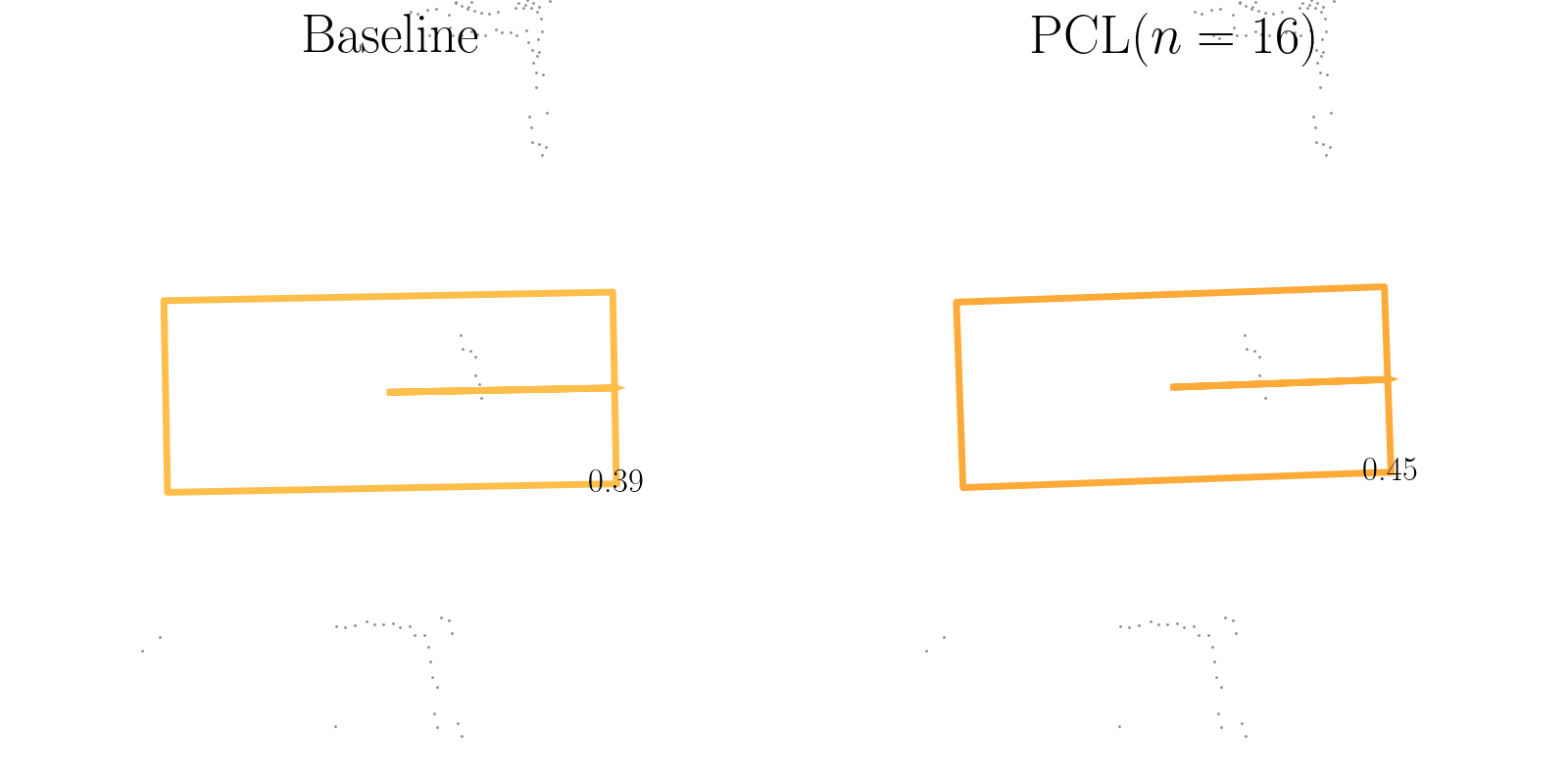}
    \put(-320, 55){\footnotesize (a) Confidence}
    \put(-235, 55){\footnotesize (b) Localization}
    \put(-145, 55){\footnotesize (c) Extent}
    \put(-60, 55){\footnotesize (d) Head}

    \includegraphics[trim={0pt 0pt 0pt 0pt},clip, width=0.23\textwidth]{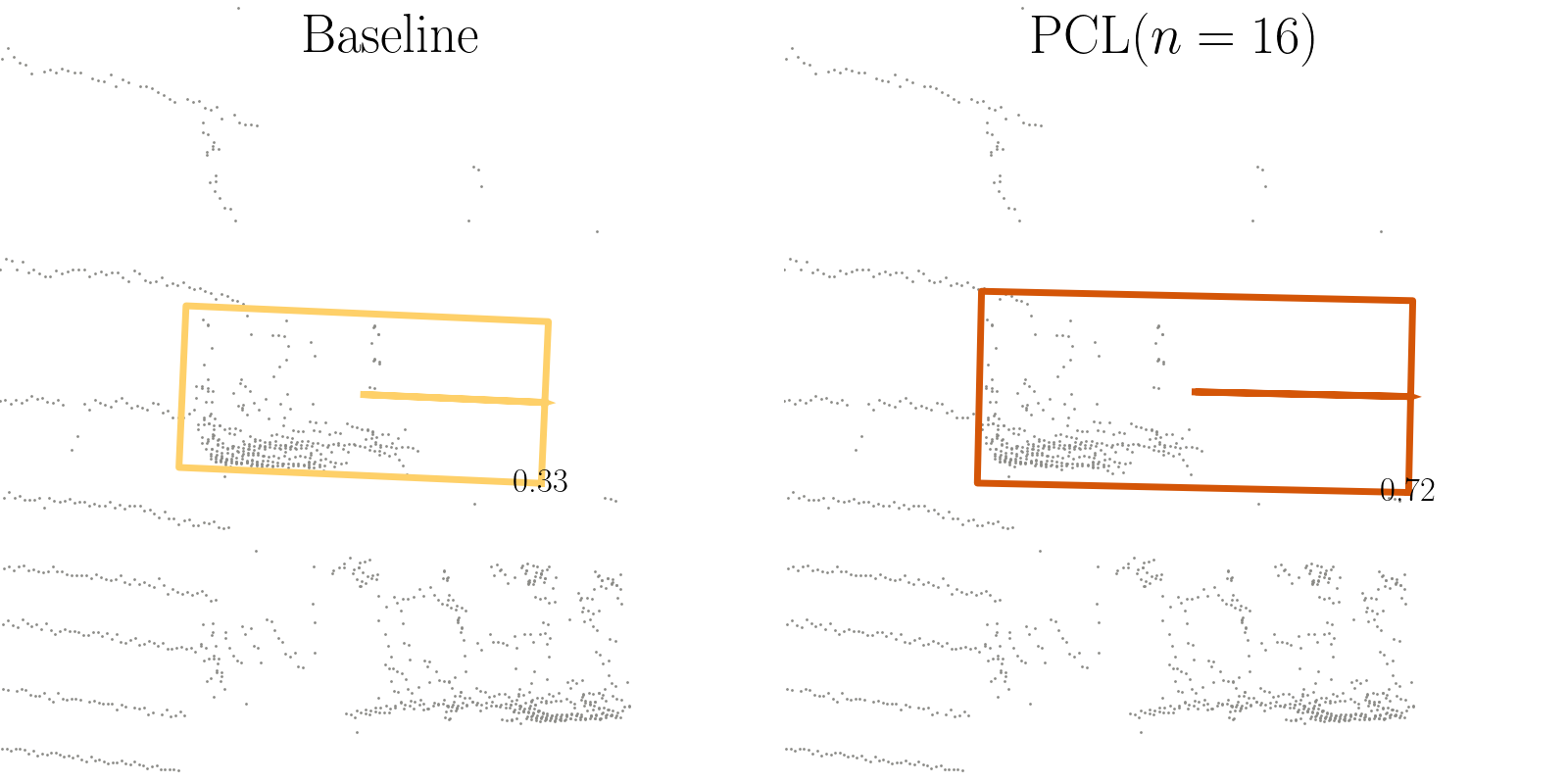}
    \includegraphics[trim={0pt 0pt 0pt 0pt},clip, width=0.23\textwidth]{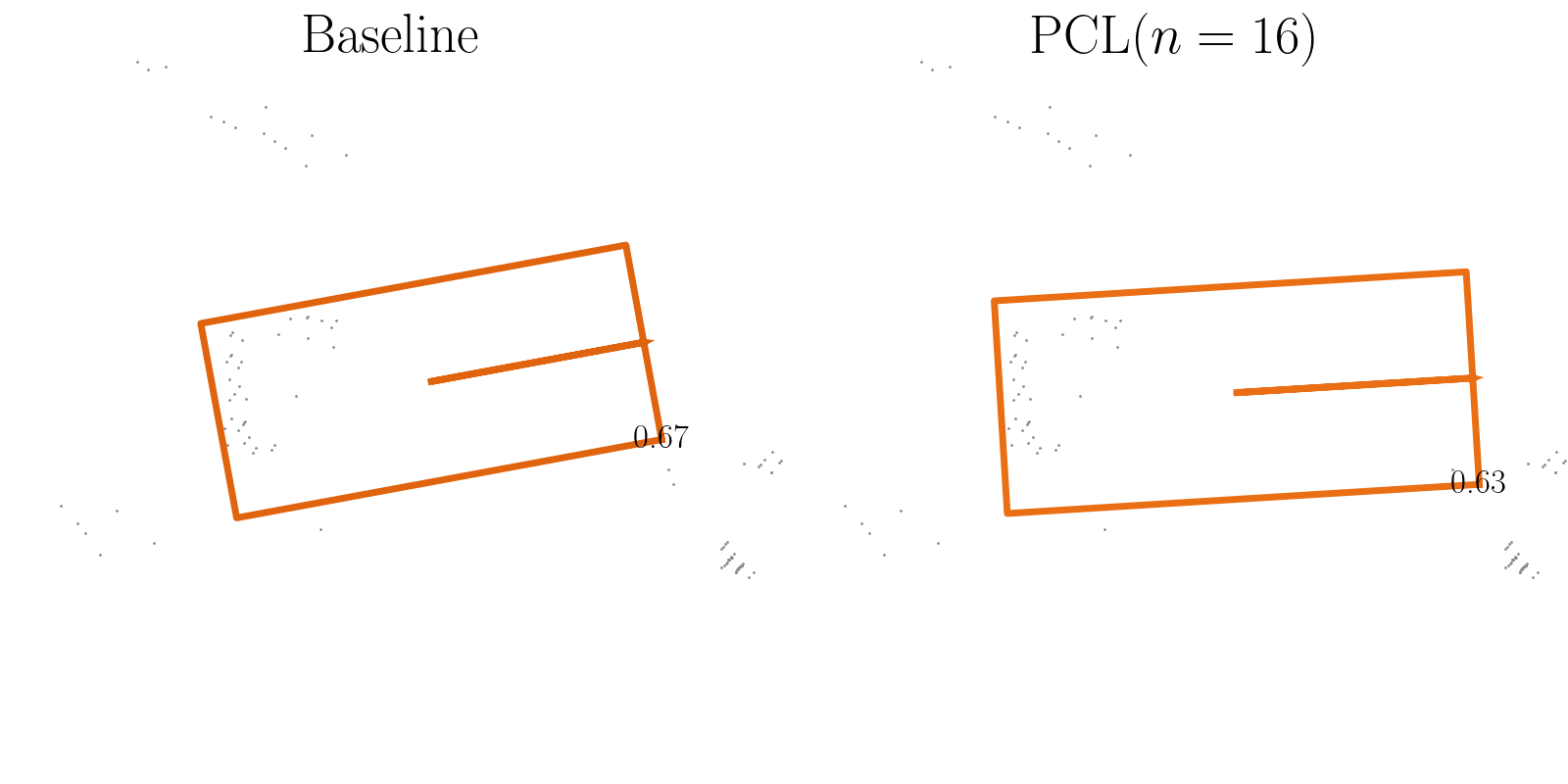}
    \includegraphics[trim={0pt 0pt 0pt 0pt},clip, width=0.23\textwidth]{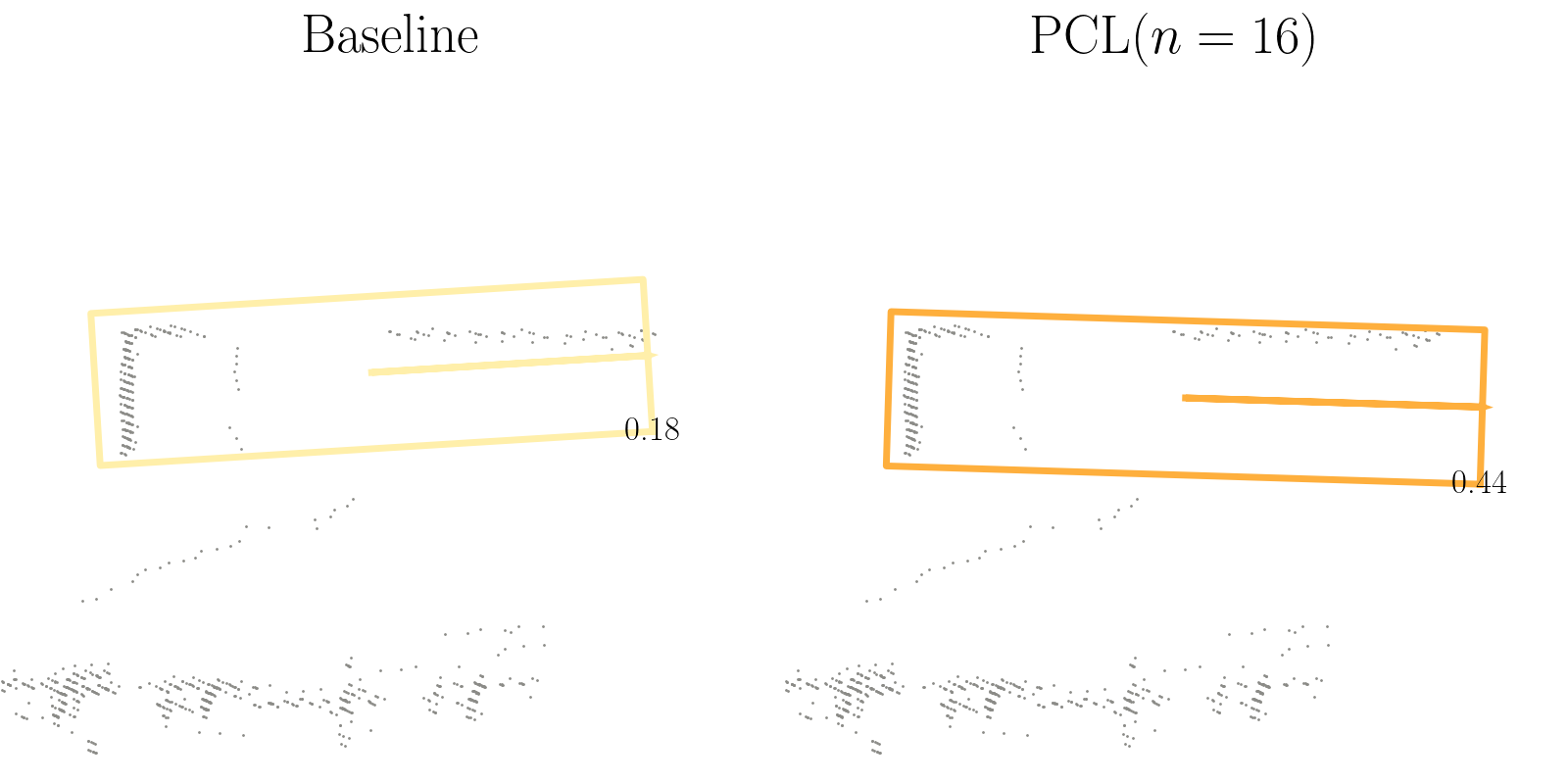}
    \includegraphics[trim={0pt 0pt 0pt 0pt},clip, width=0.23\textwidth]{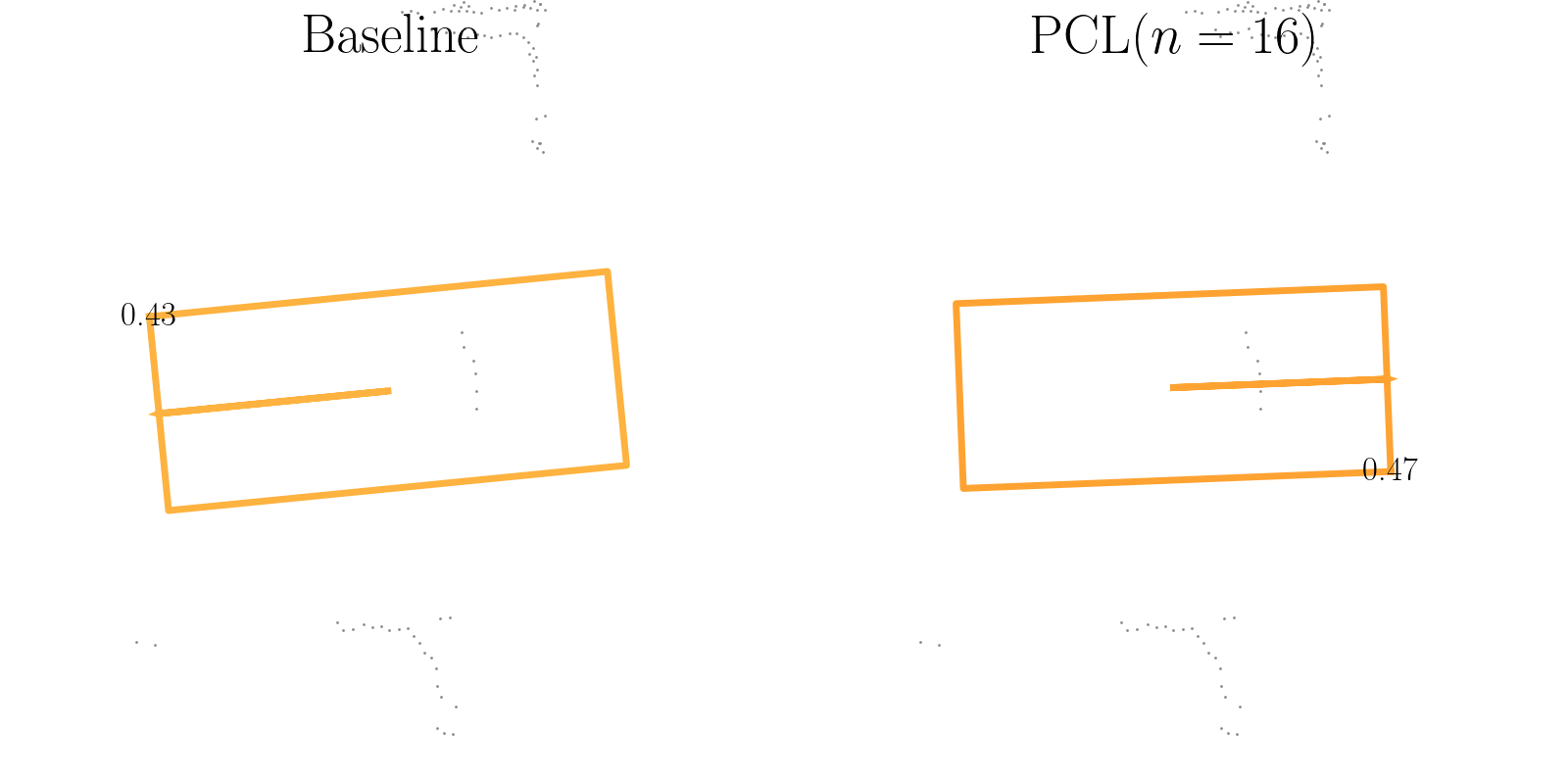}

    \includegraphics[trim={0pt 0pt 0pt 0pt},clip, width=0.23\textwidth]{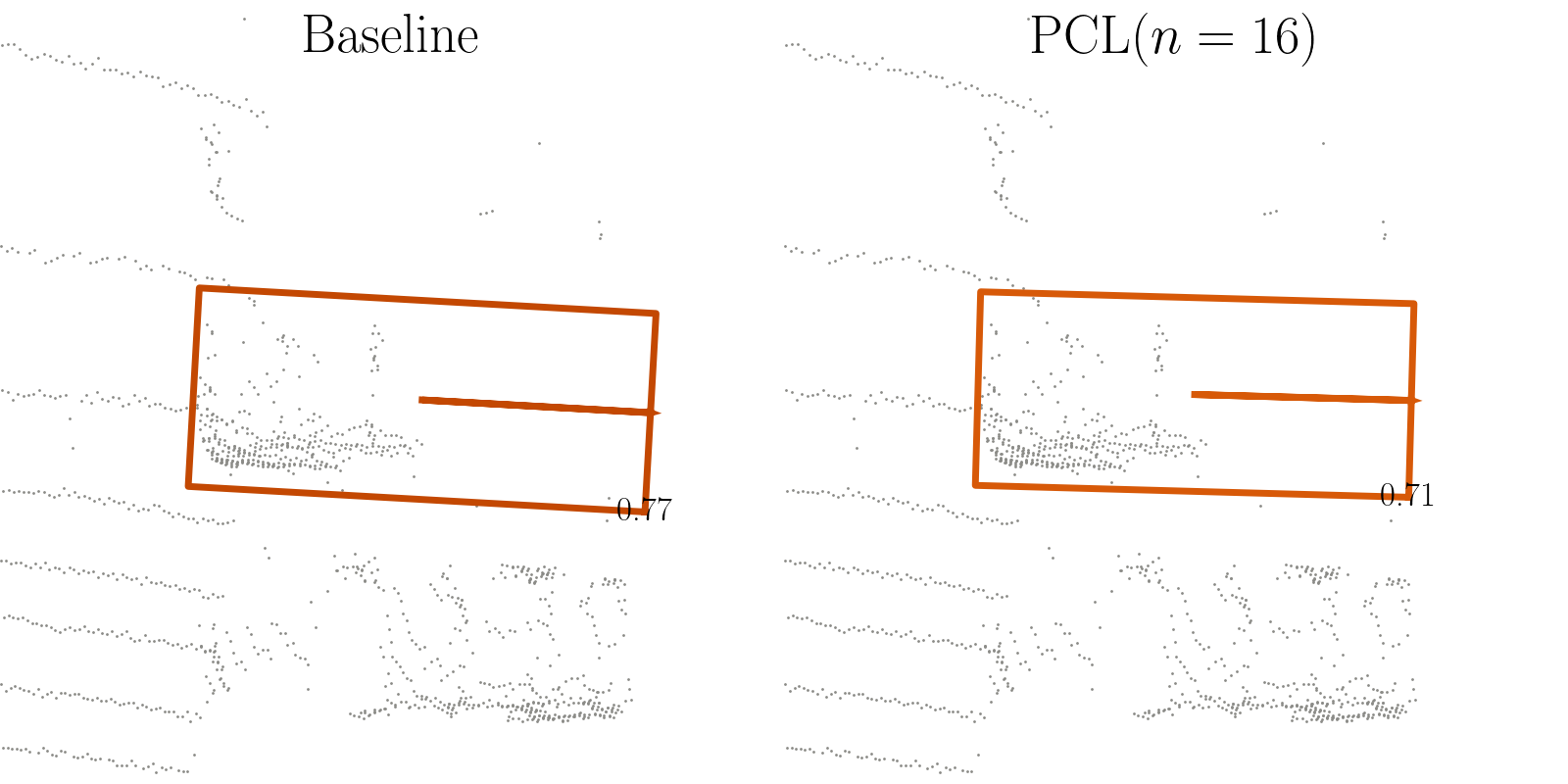}
    \includegraphics[trim={0pt 0pt 0pt 0pt},clip, width=0.23\textwidth]{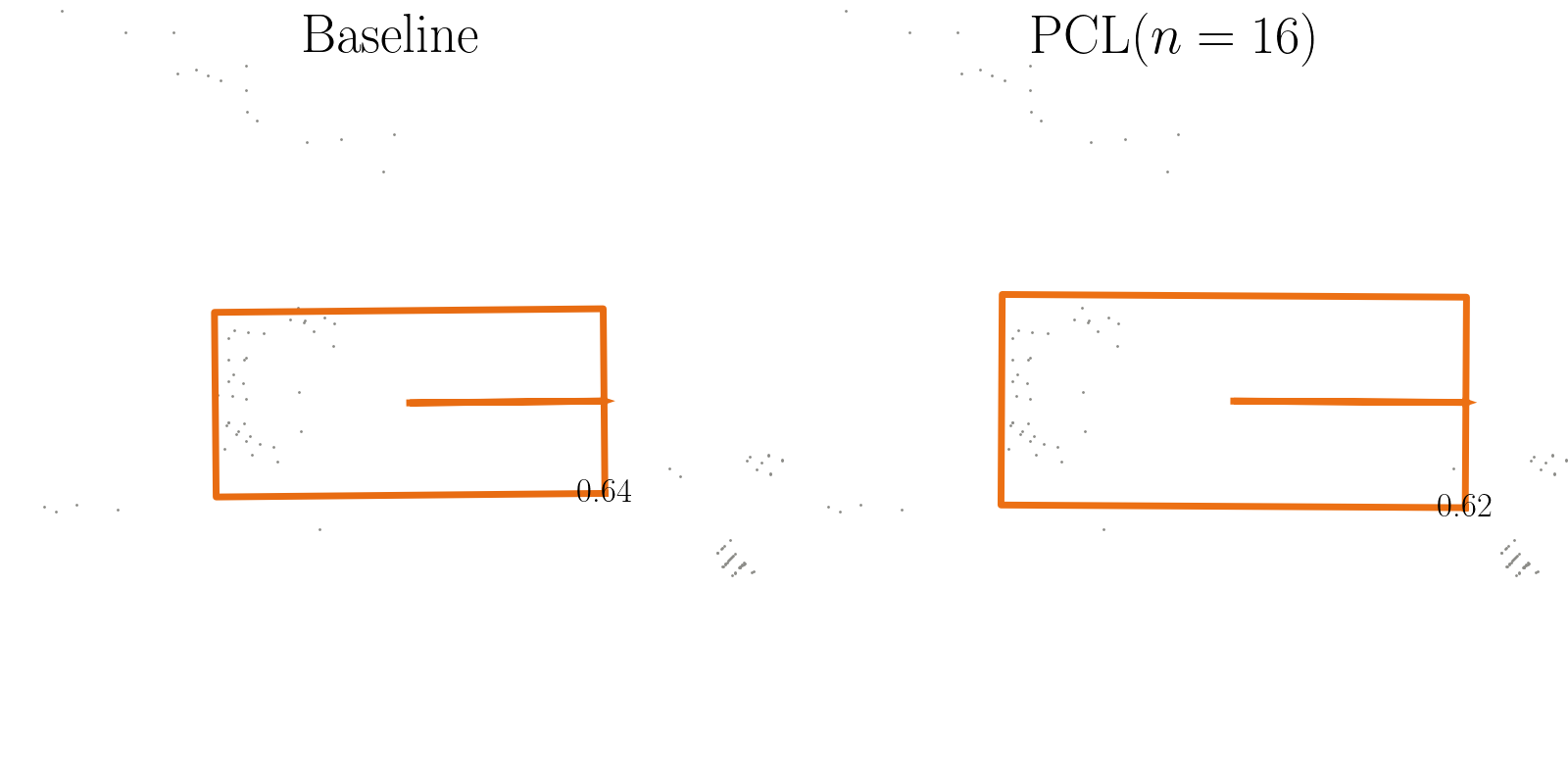}
    \includegraphics[trim={0pt 0pt 0pt 0pt},clip, width=0.23\textwidth]{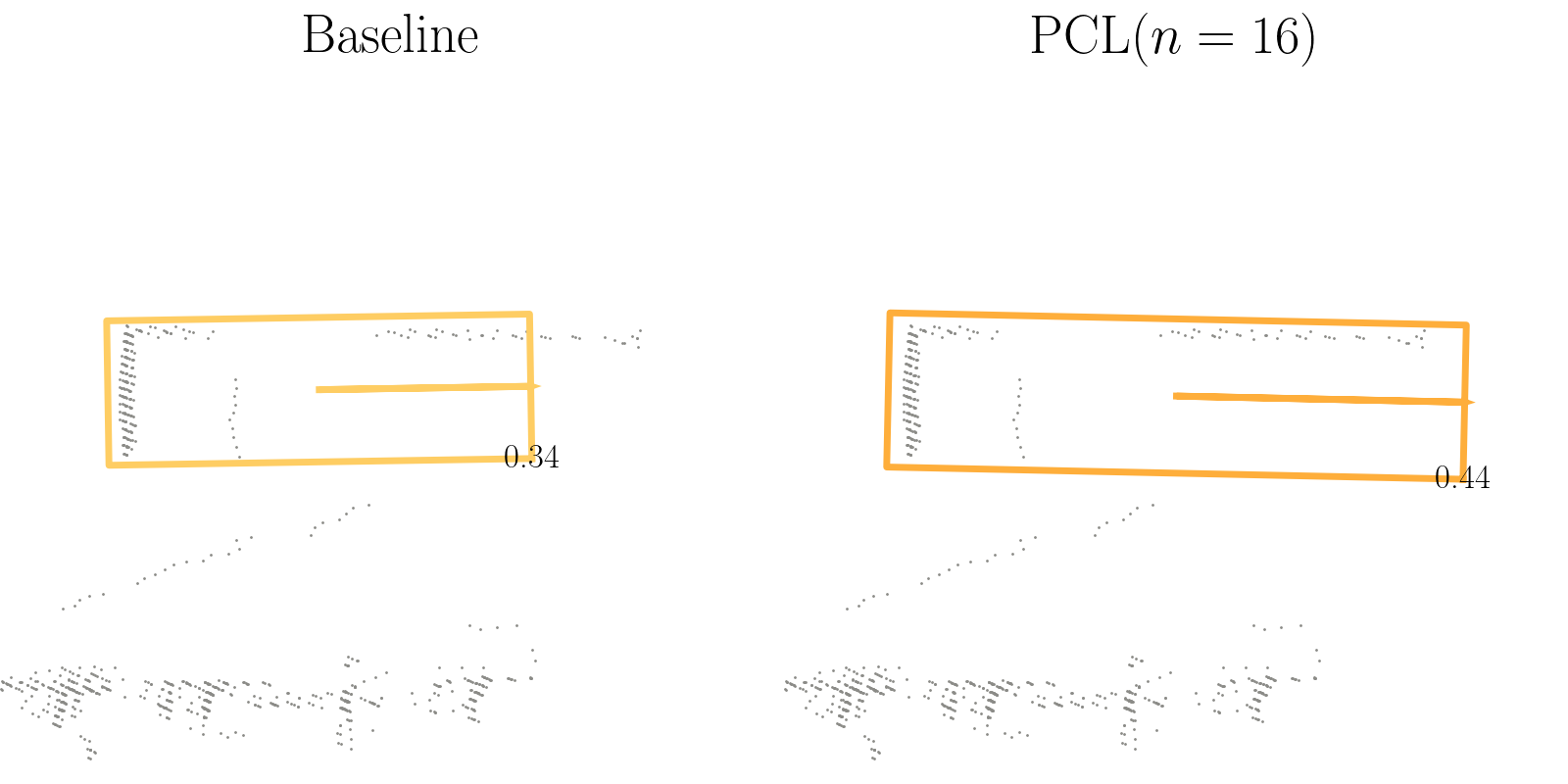}
    \includegraphics[trim={0pt 0pt 0pt 0pt},clip, width=0.23\textwidth]{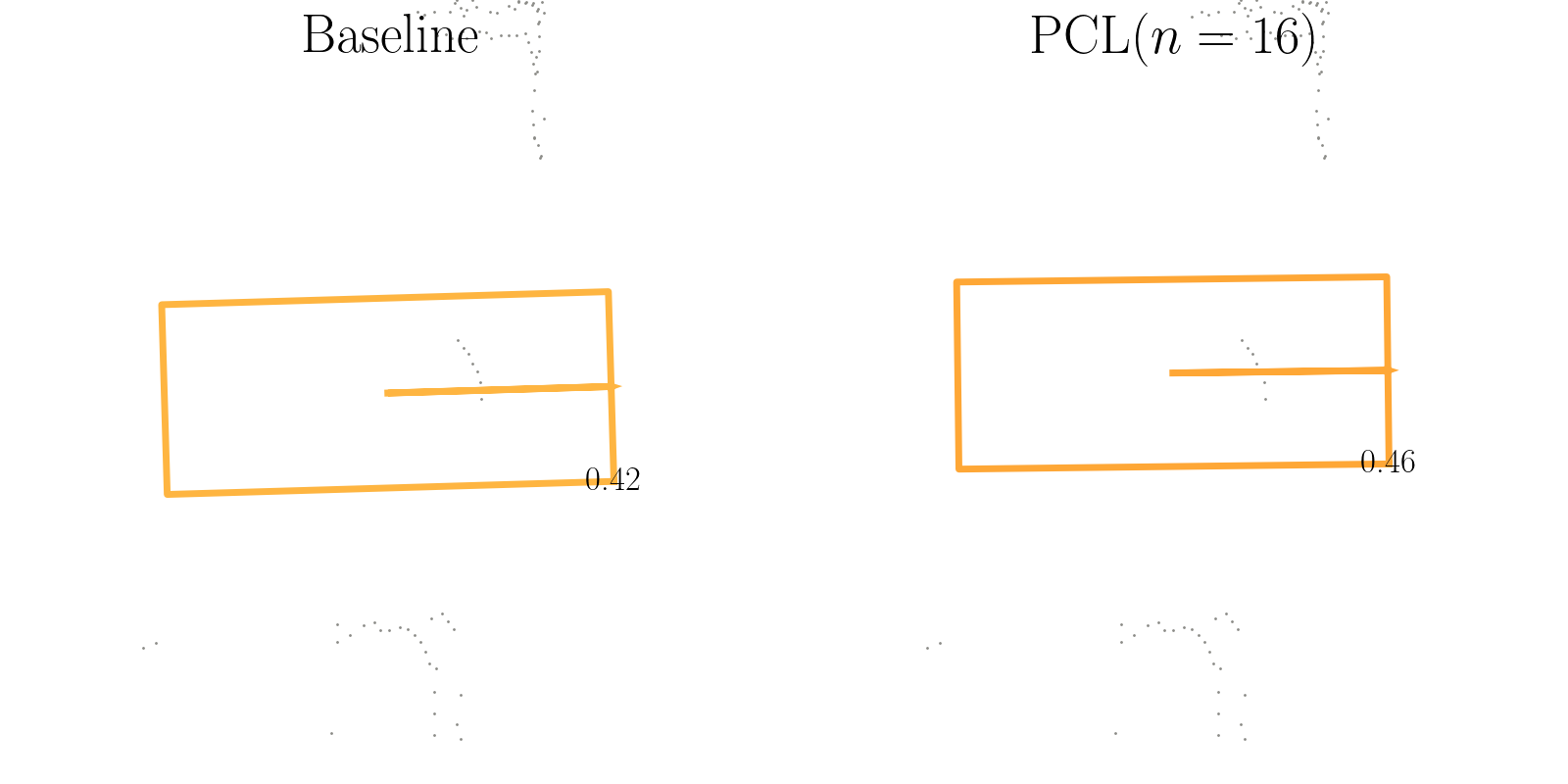}

    \includegraphics[trim={0pt 0pt 0pt 0pt},clip, width=0.23\textwidth]{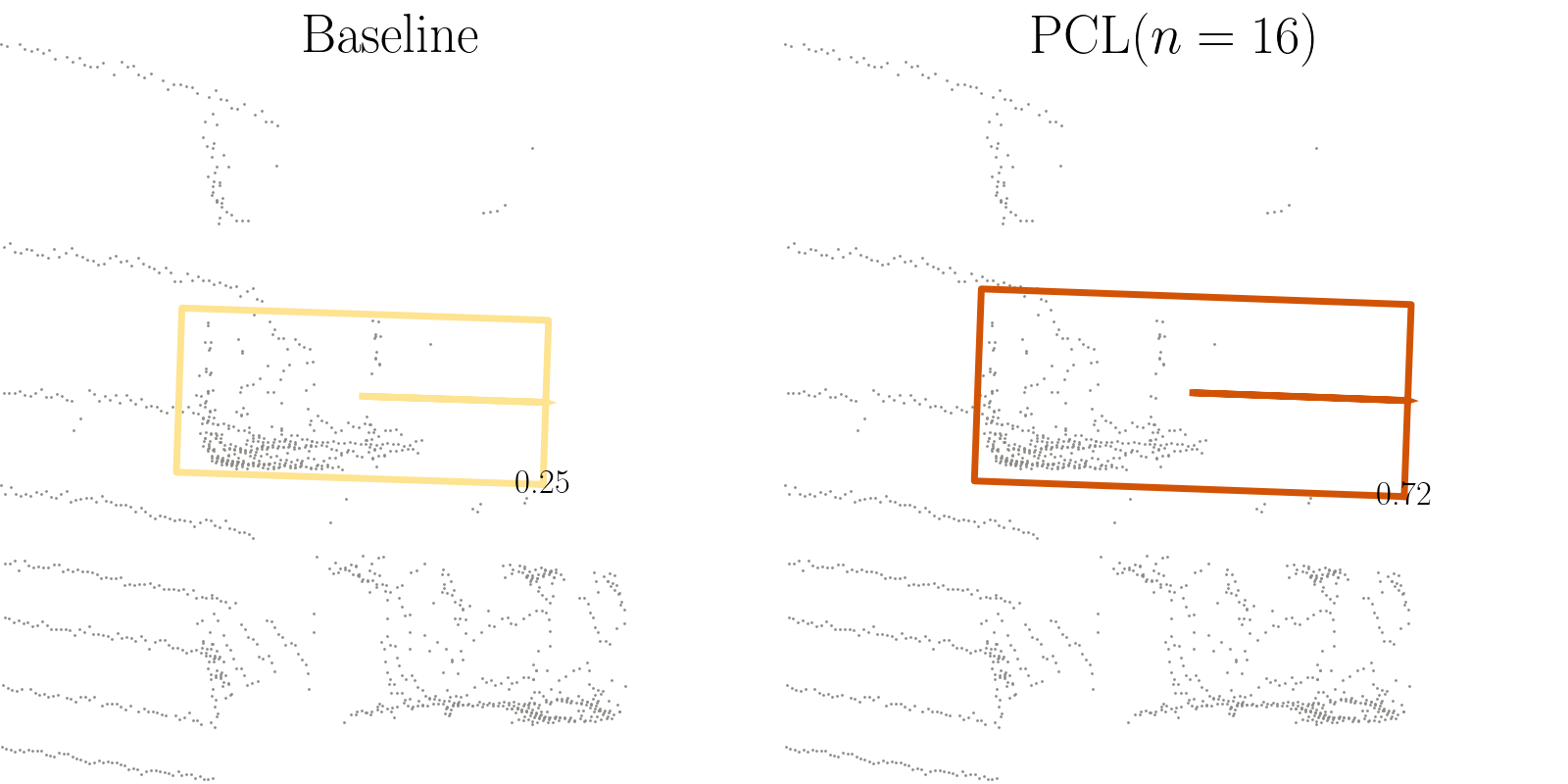}
    \includegraphics[trim={0pt 0pt 0pt 0pt},clip, width=0.23\textwidth]{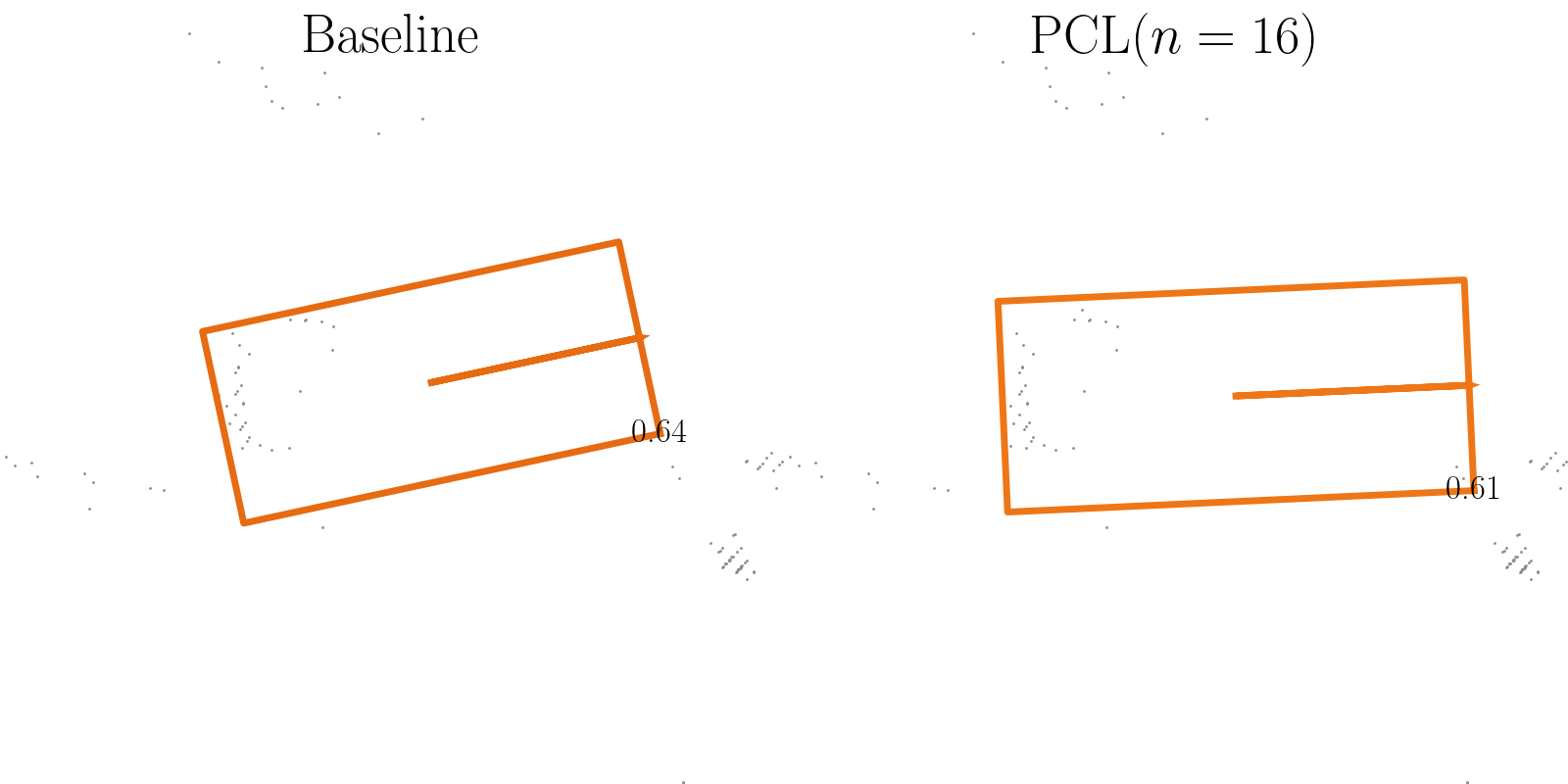}
    \includegraphics[trim={0pt 0pt 0pt 0pt},clip, width=0.23\textwidth]{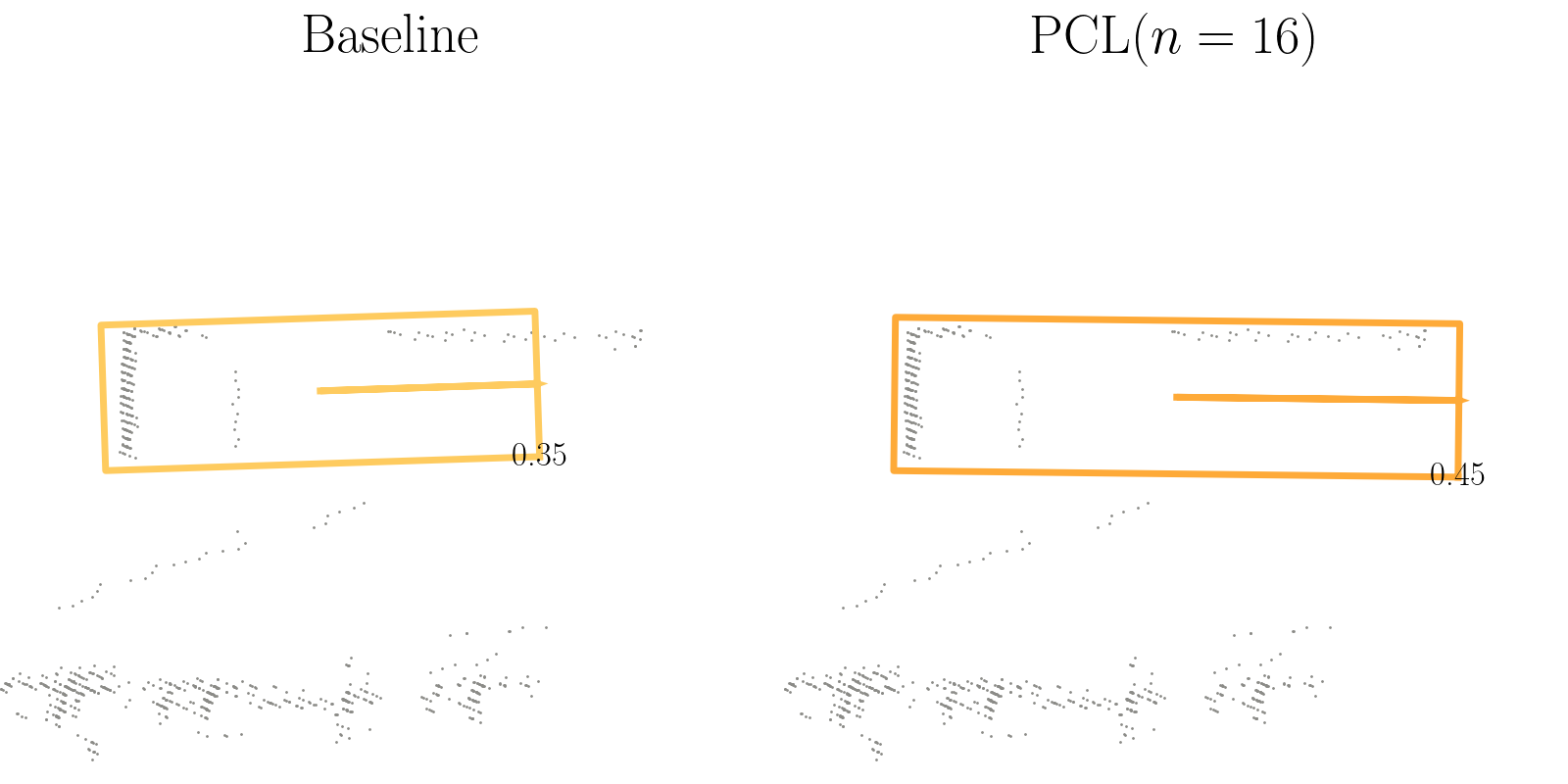}
    \includegraphics[trim={0pt 0pt 0pt 0pt},clip, width=0.23\textwidth]{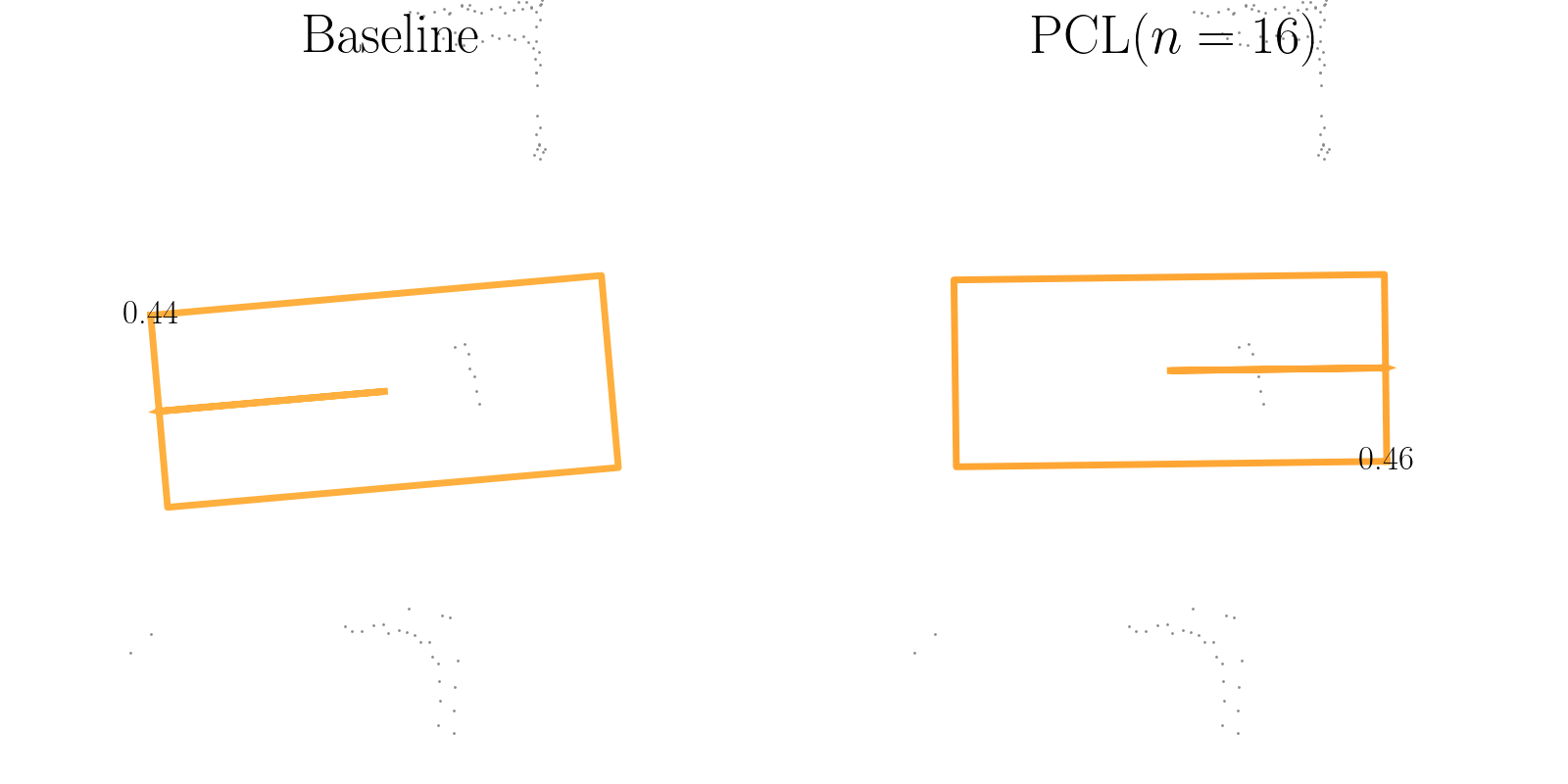}

    \includegraphics[trim={0pt 0pt 0pt 0pt},clip, width=0.23\textwidth]{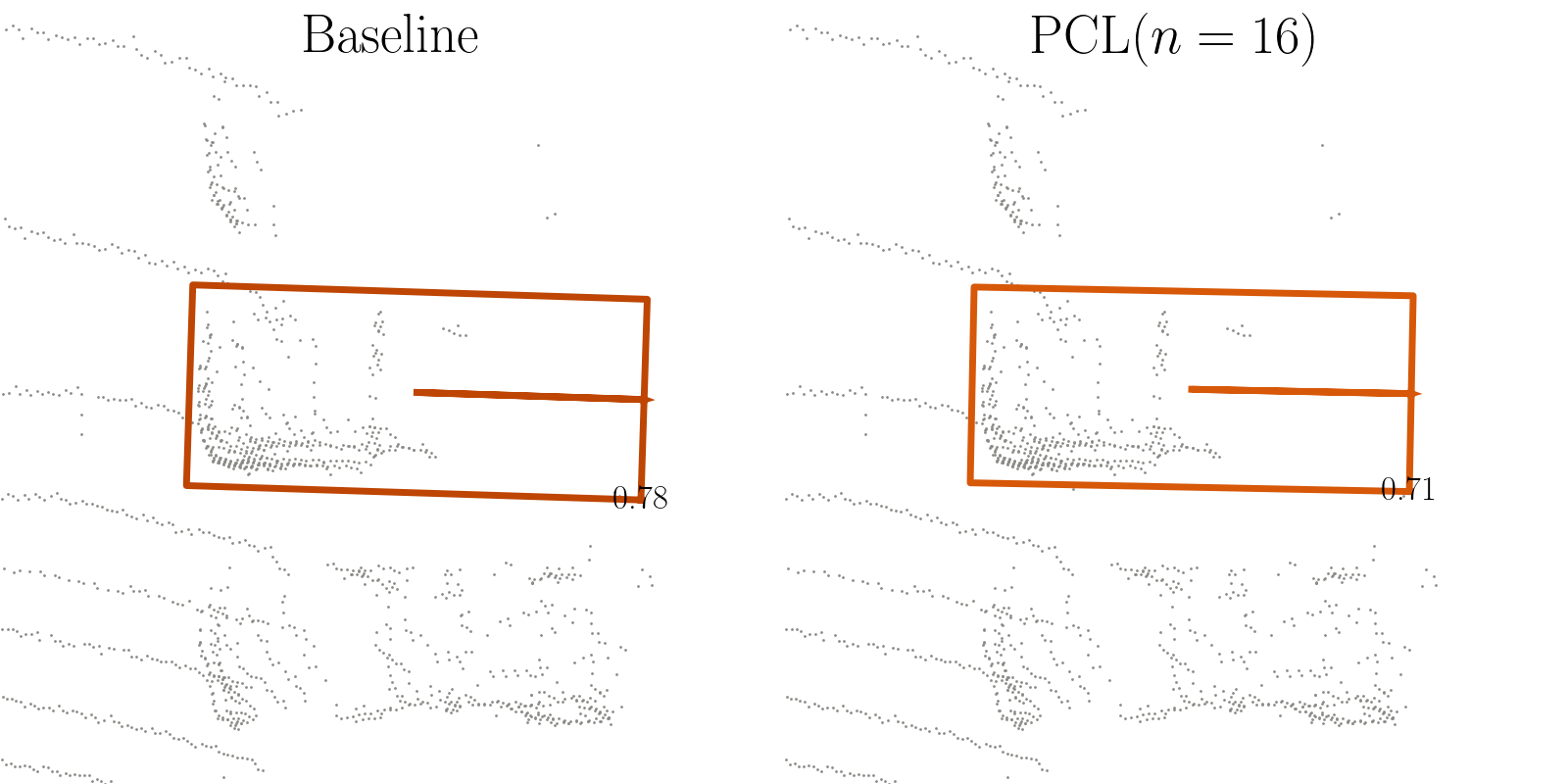}
    \includegraphics[trim={0pt 0pt 0pt 0pt},clip, width=0.23\textwidth]{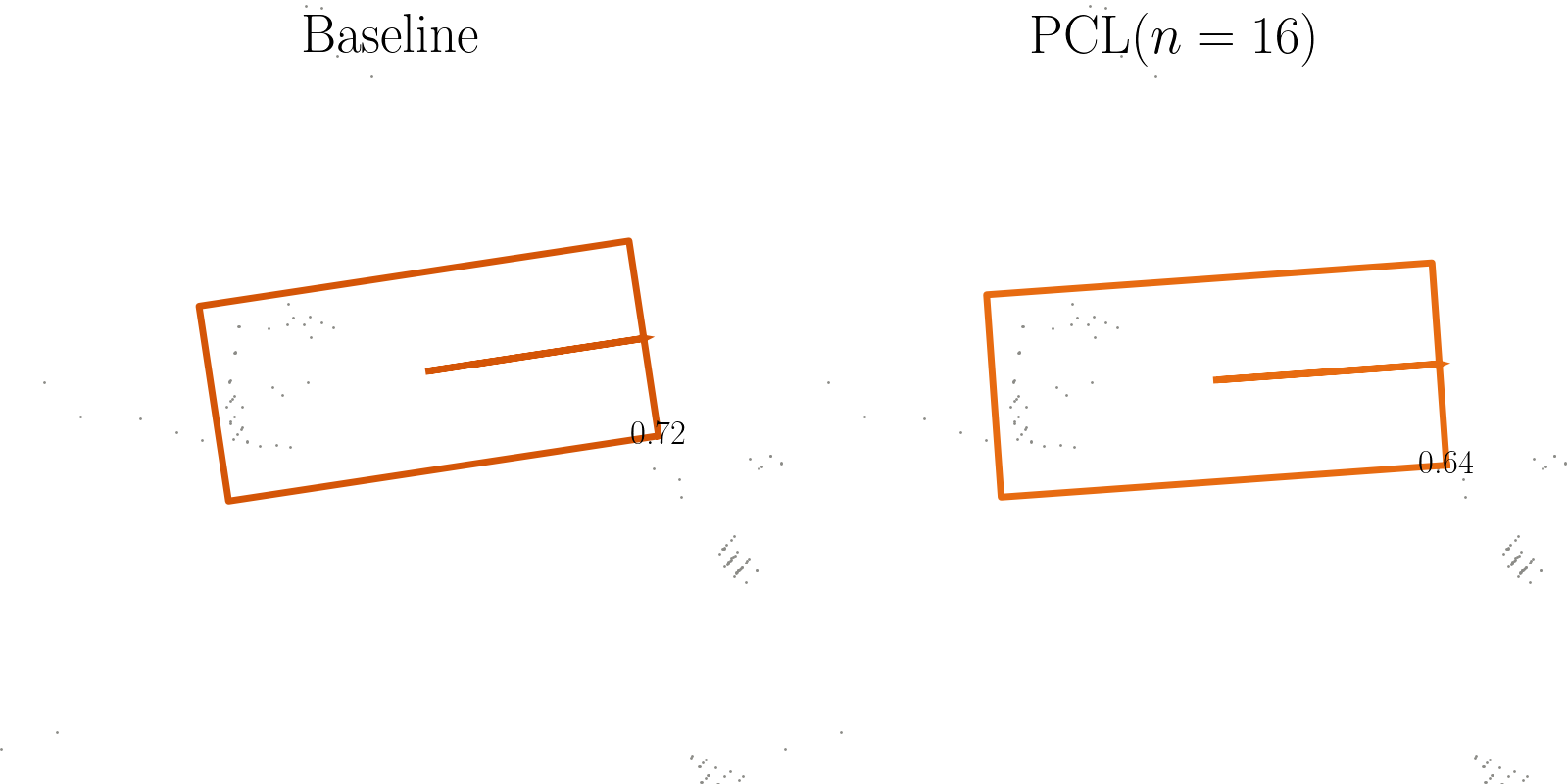}
    \includegraphics[trim={0pt 0pt 0pt 0pt},clip, width=0.23\textwidth]{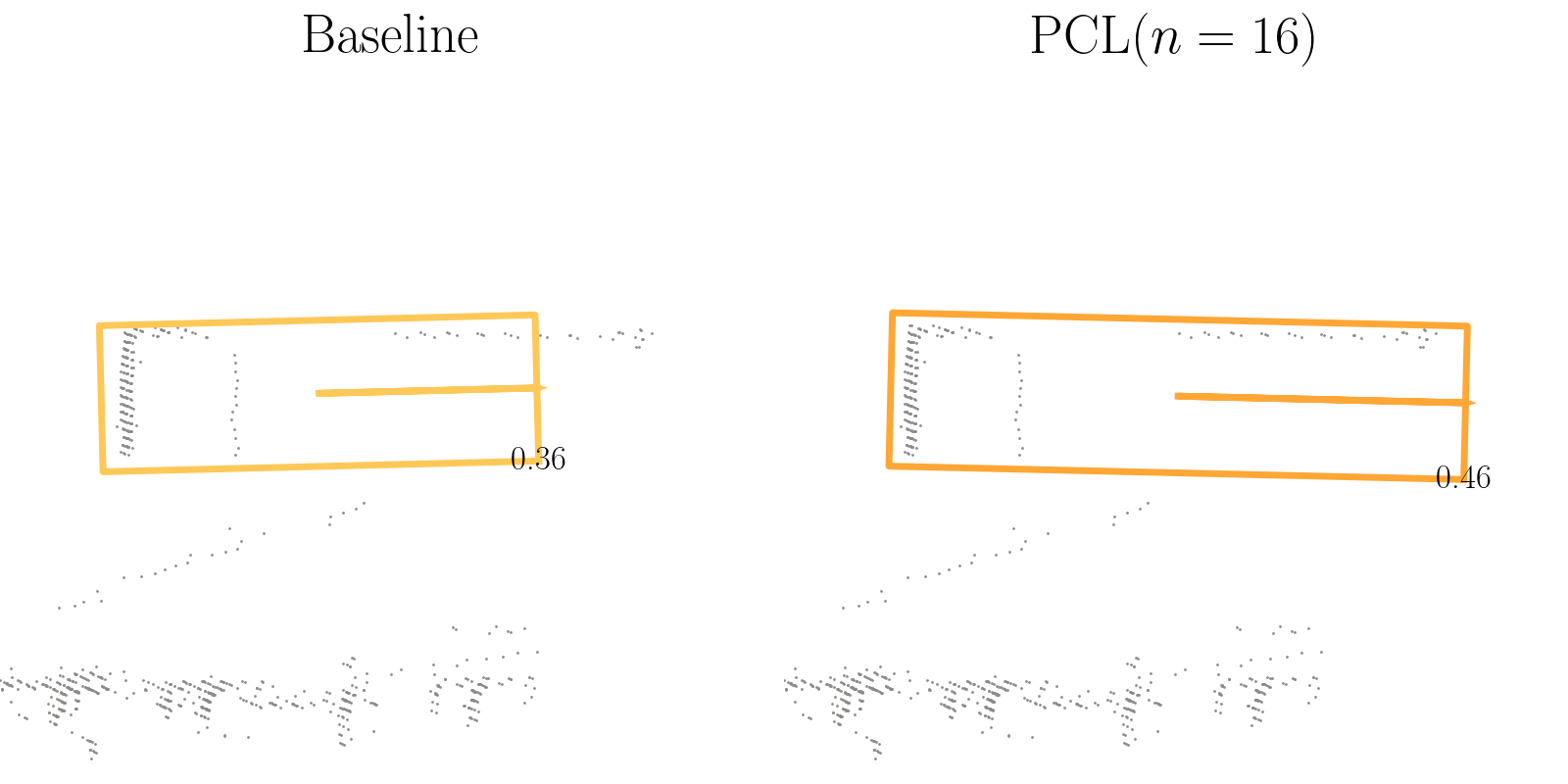}
    \includegraphics[trim={0pt 0pt 0pt 0pt},clip, width=0.23\textwidth]{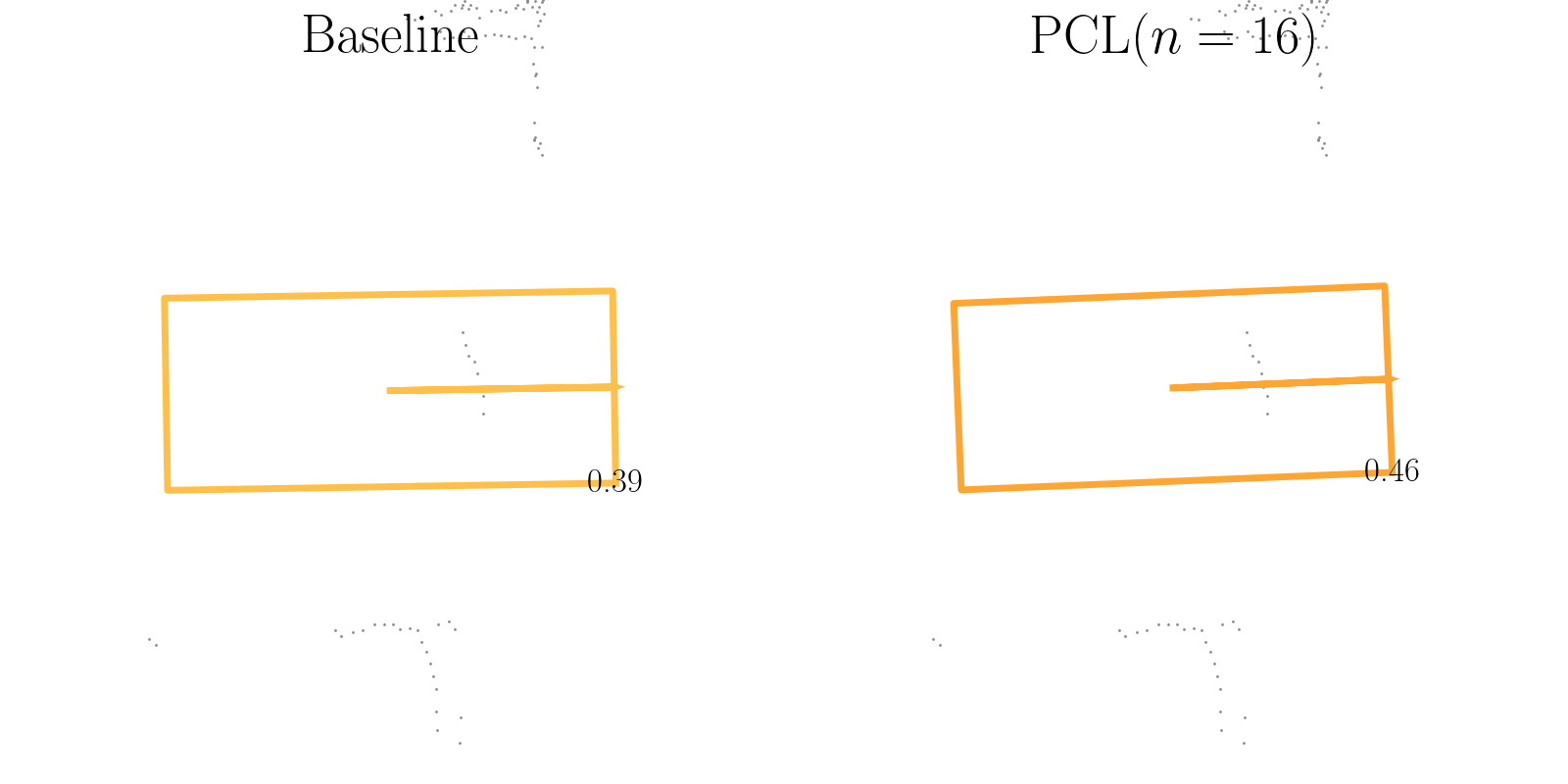}

    \includegraphics[trim={0pt 0pt 0pt 0pt},clip, width=0.23\textwidth]{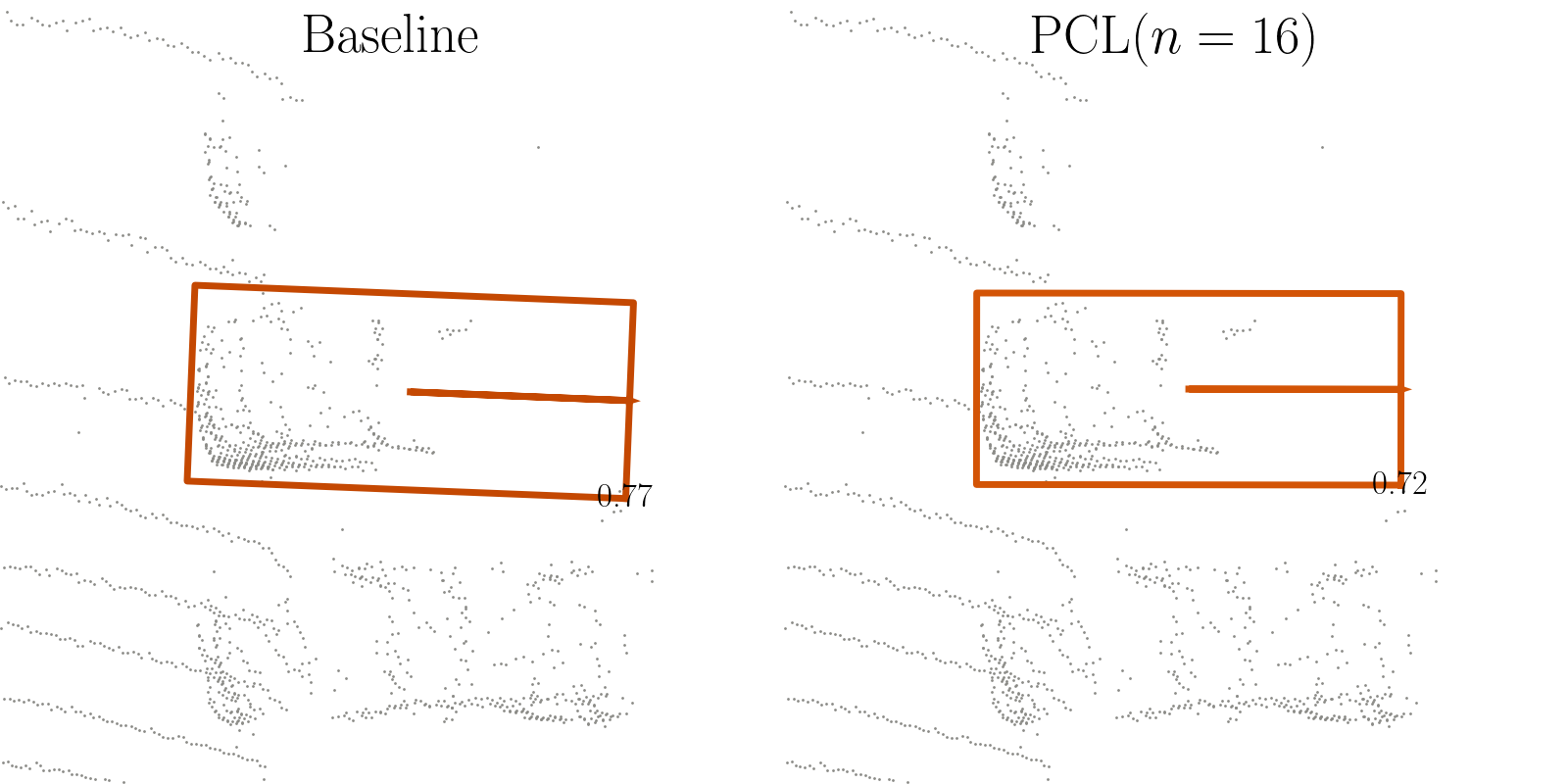}
    \includegraphics[trim={0pt 0pt 0pt 0pt},clip, width=0.23\textwidth]{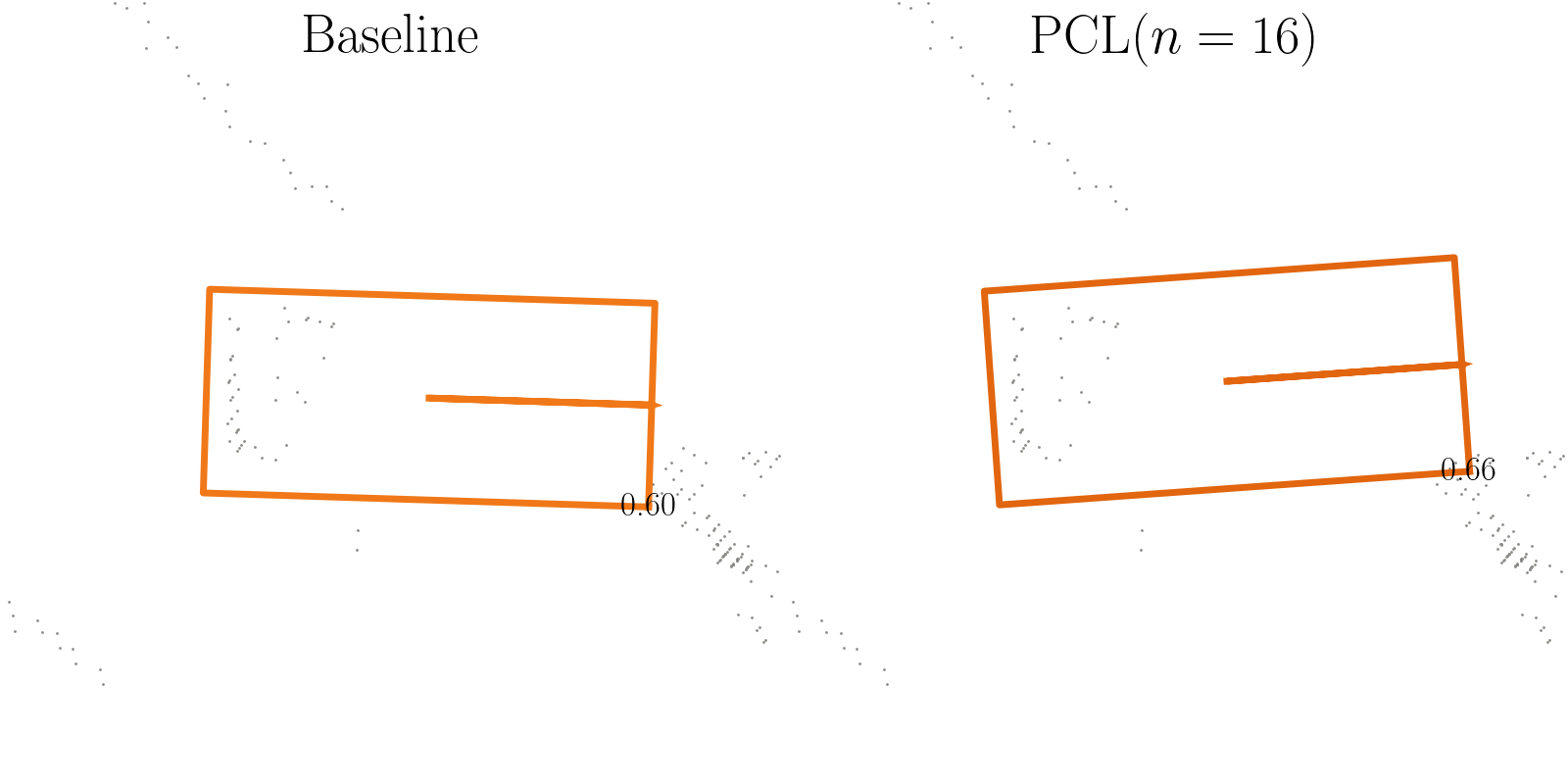}
    \includegraphics[trim={0pt 0pt 0pt 0pt},clip, width=0.23\textwidth]{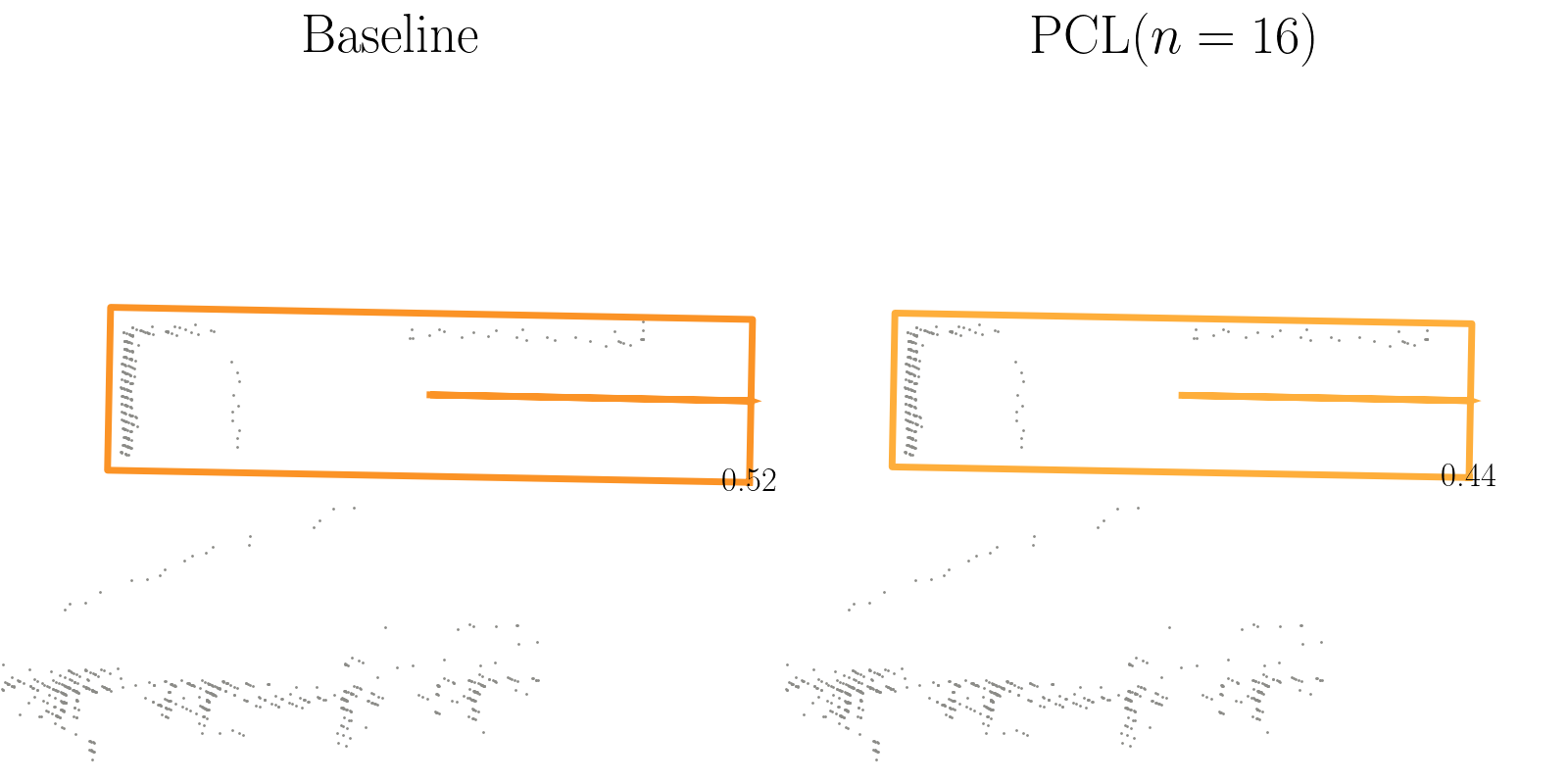}
    \includegraphics[trim={0pt 0pt 0pt 0pt},clip, width=0.23\textwidth]{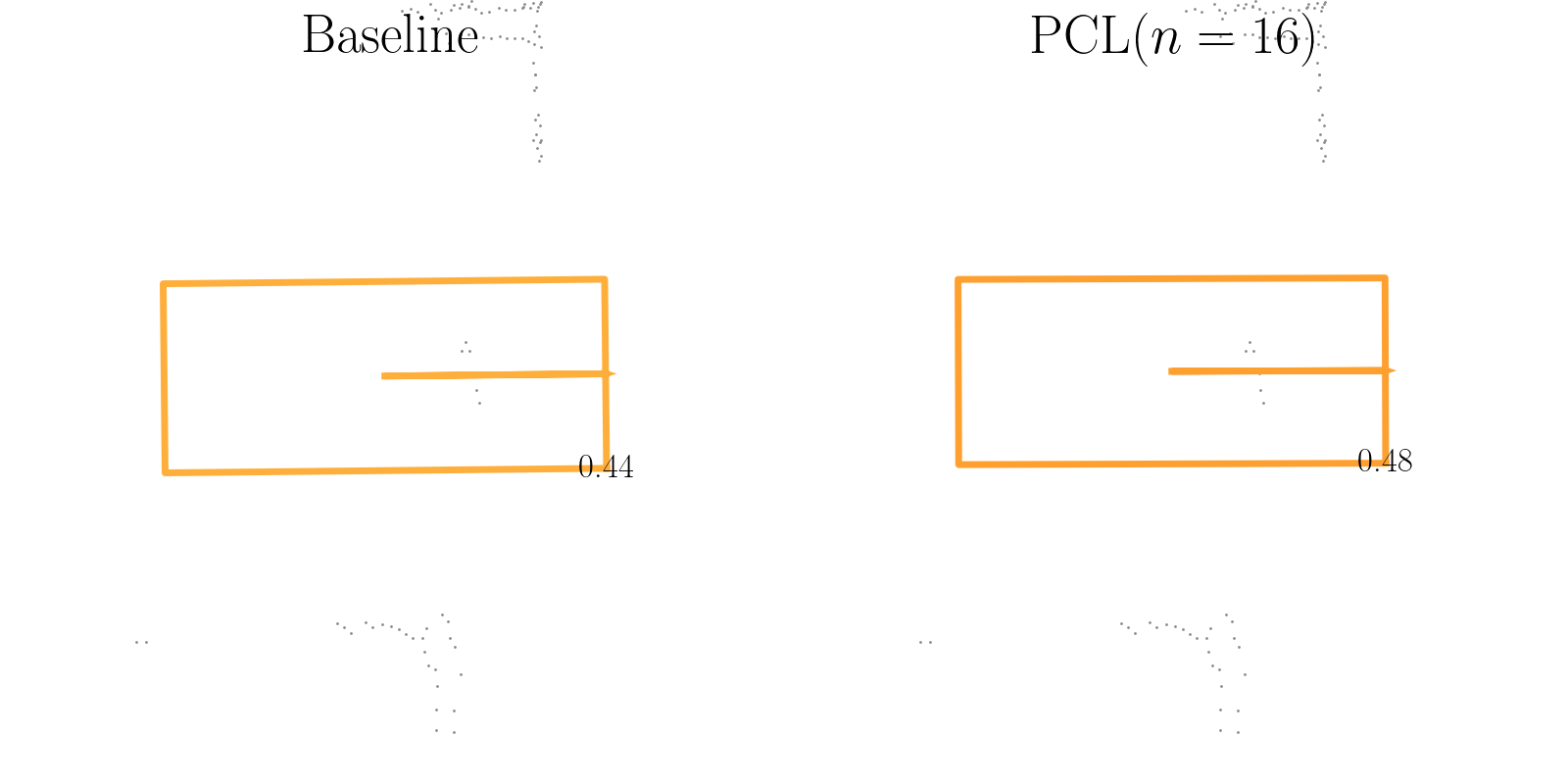}

    \includegraphics[trim={0pt 0pt 0pt 0pt},clip, width=0.23\textwidth]{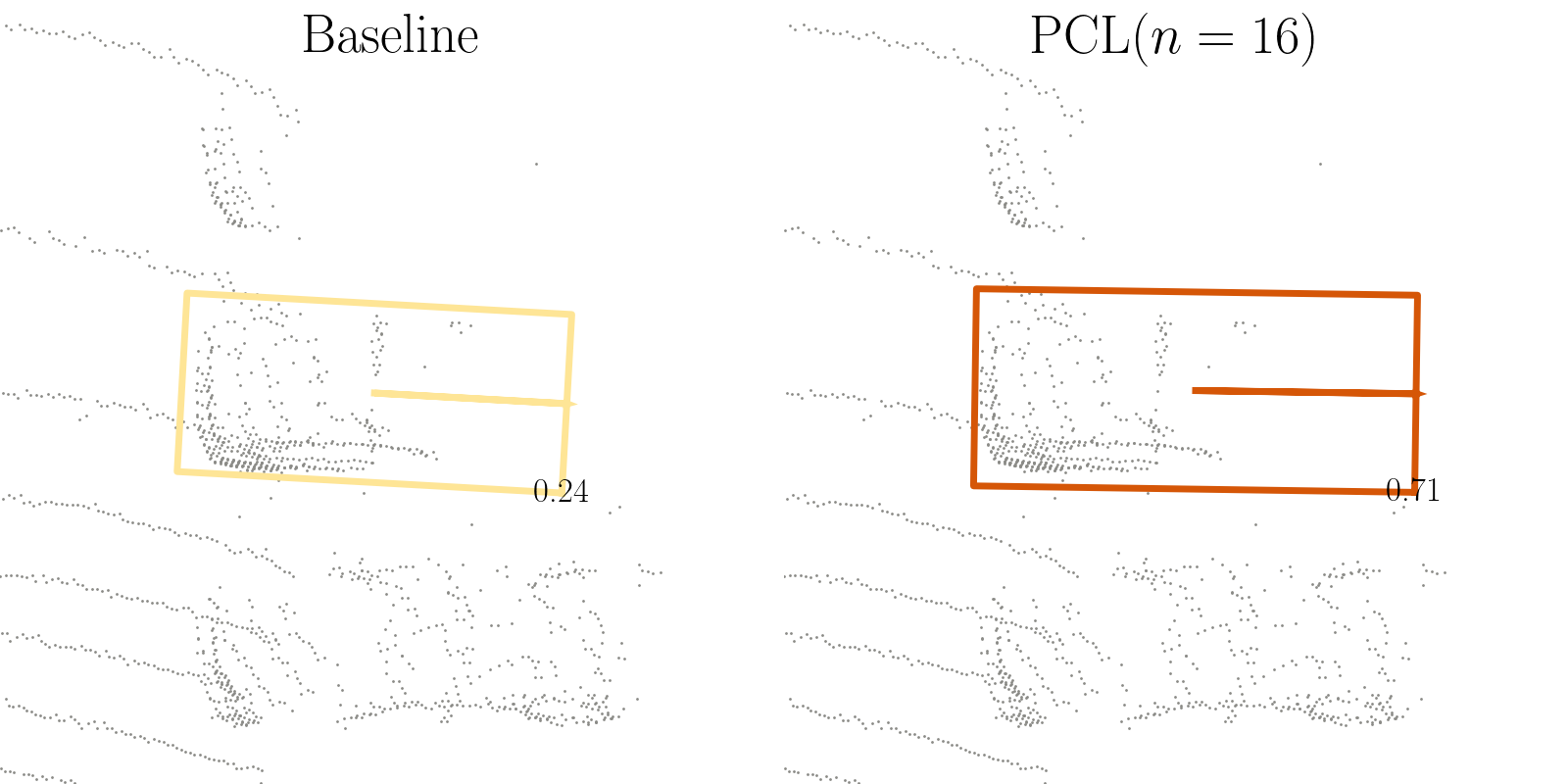}
    \includegraphics[trim={0pt 0pt 0pt 0pt},clip, width=0.23\textwidth]{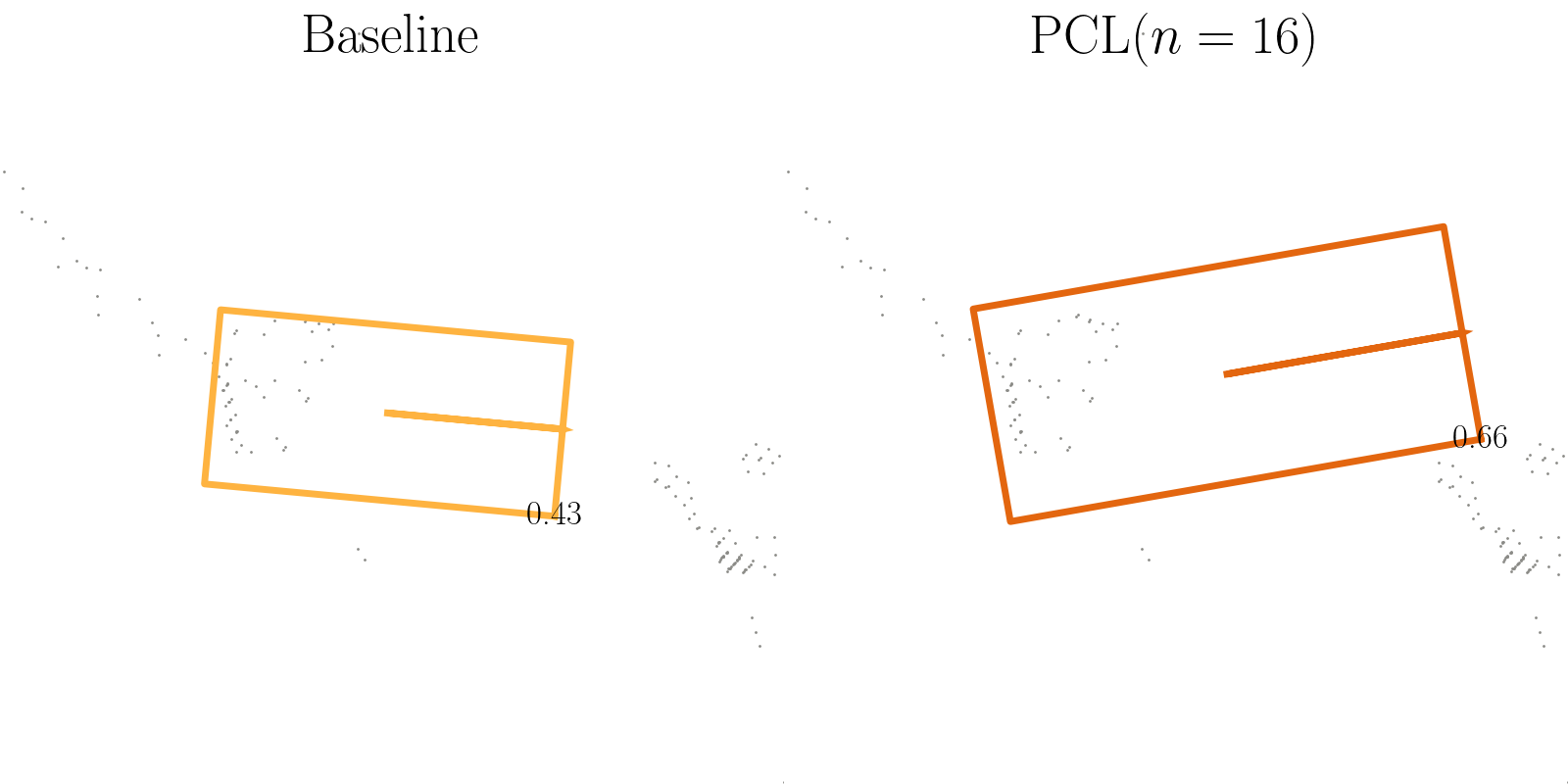}
    \includegraphics[trim={0pt 0pt 0pt 0pt},clip, width=0.23\textwidth]{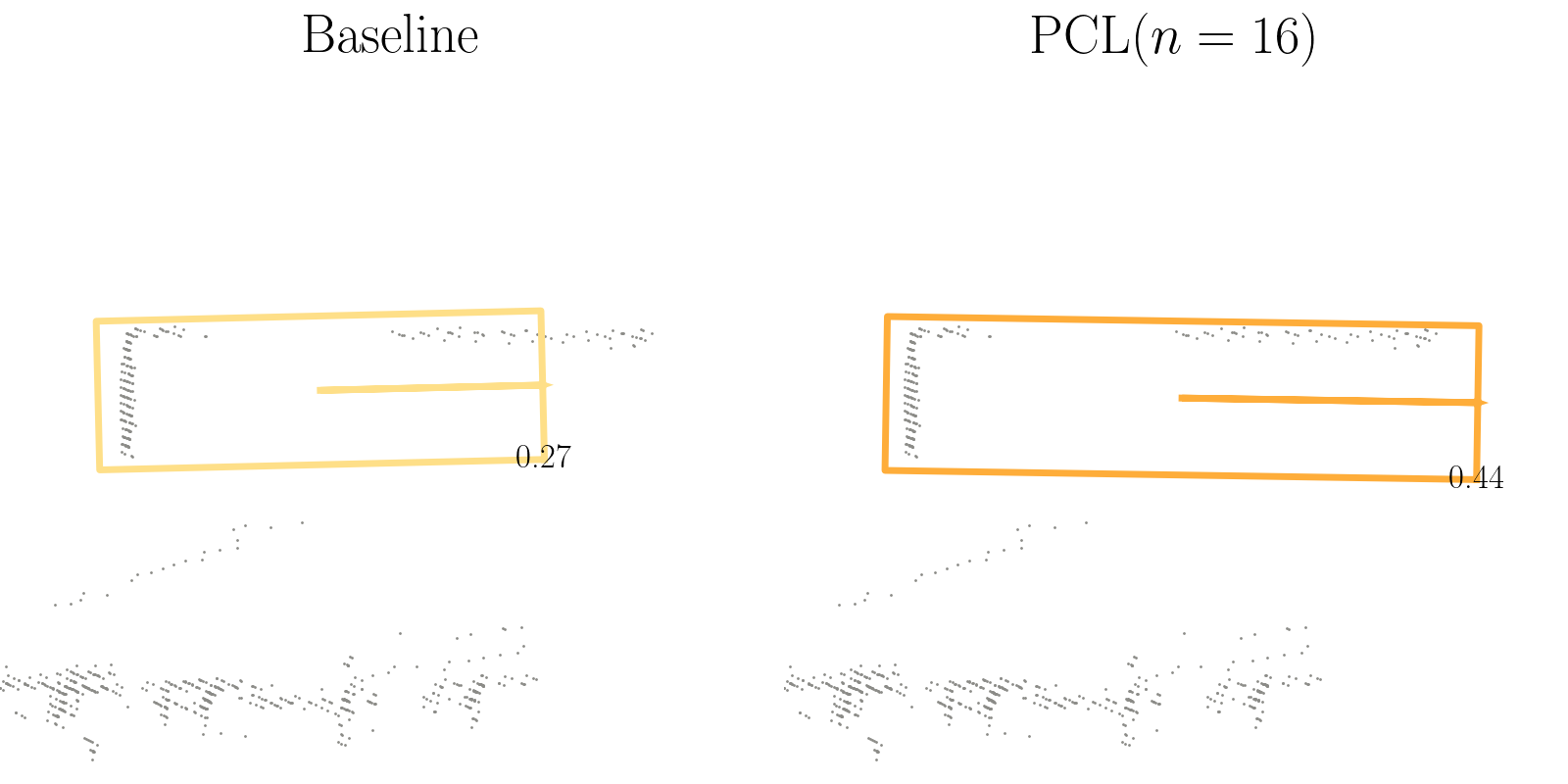}
    \includegraphics[trim={0pt 0pt 0pt 0pt},clip, width=0.23\textwidth]{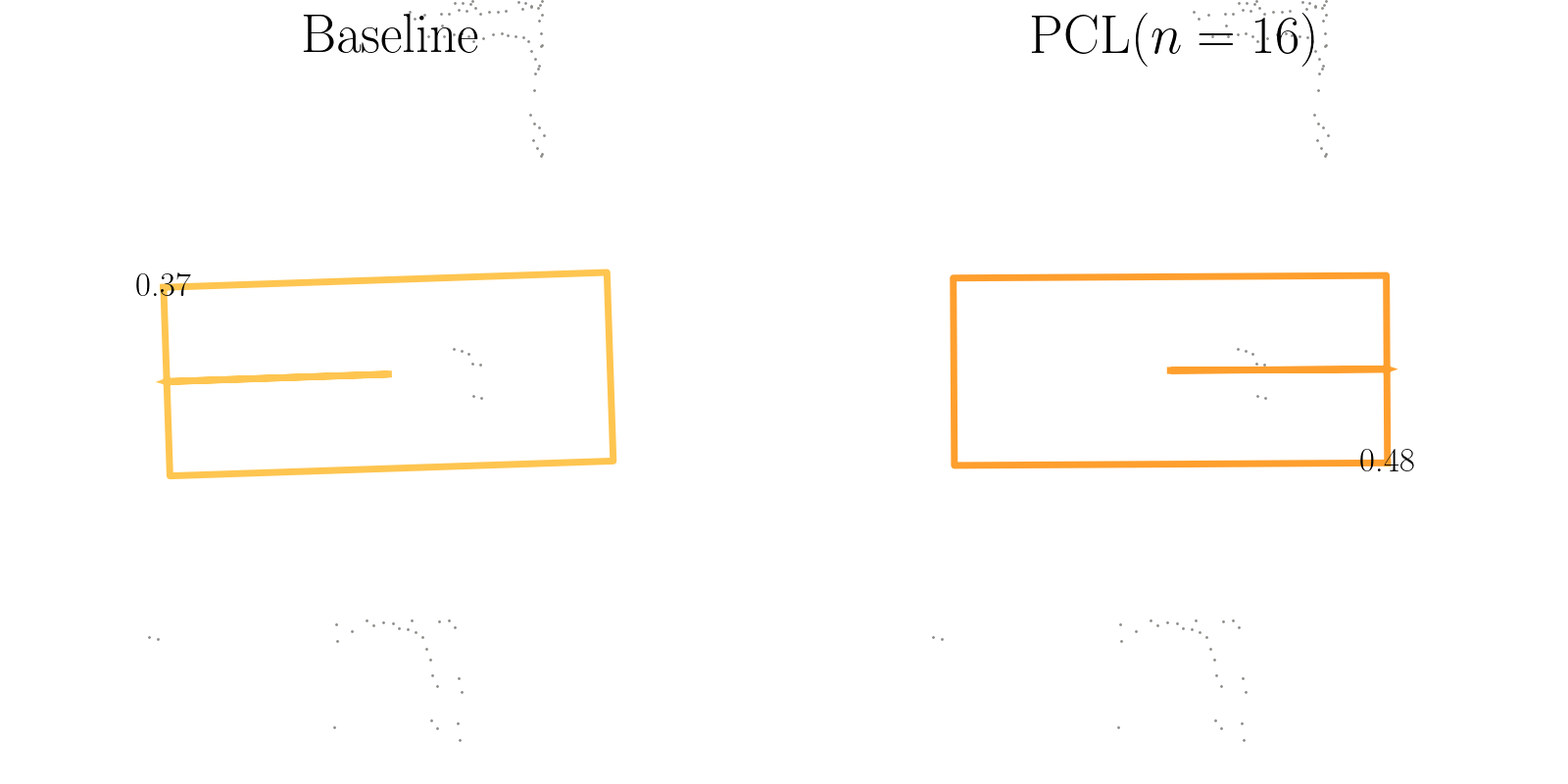}

    \includegraphics[trim={0pt 0pt 0pt 0pt},clip, width=0.23\textwidth]{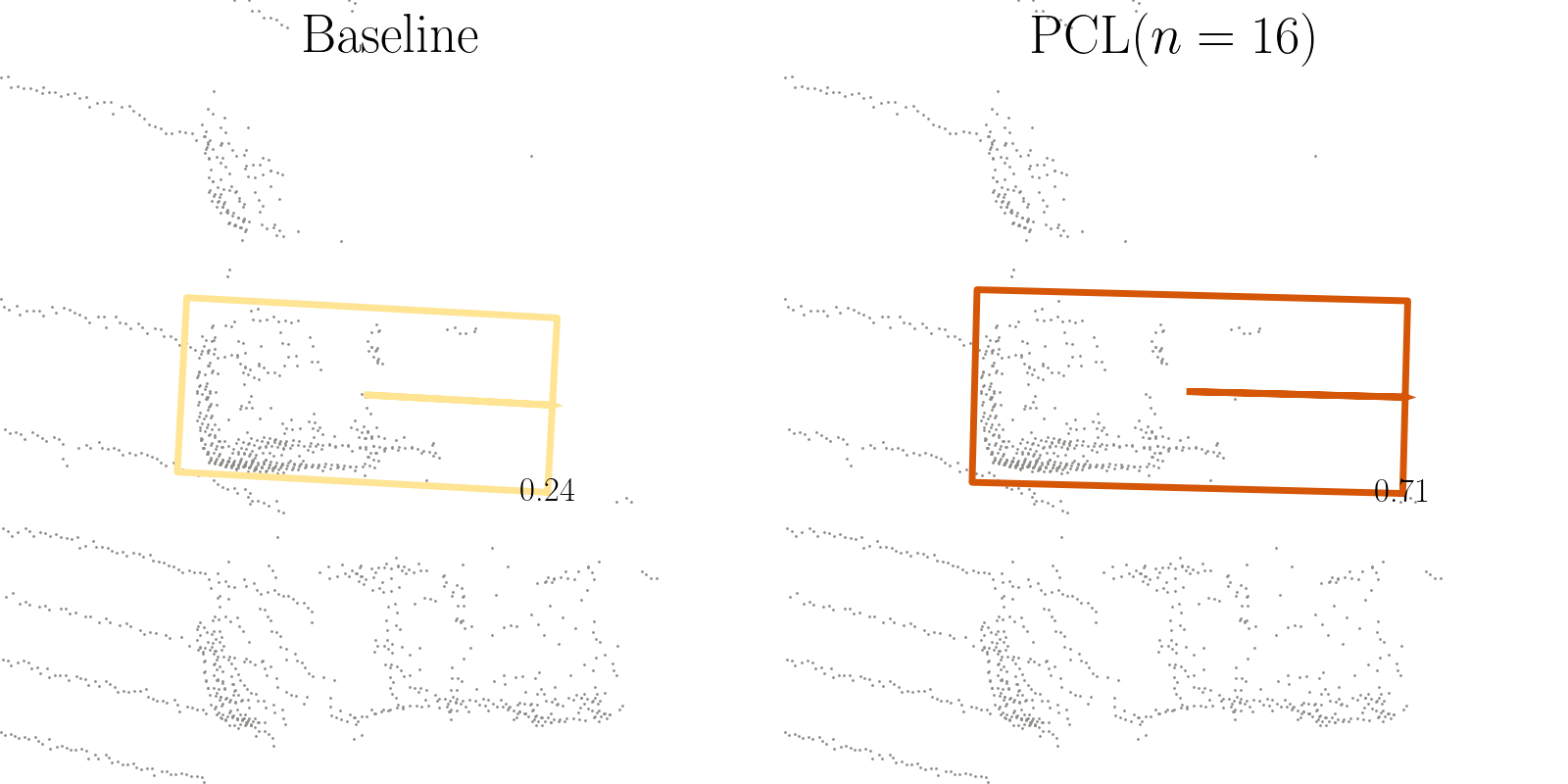}
    \includegraphics[trim={0pt 0pt 0pt 0pt},clip, width=0.23\textwidth]{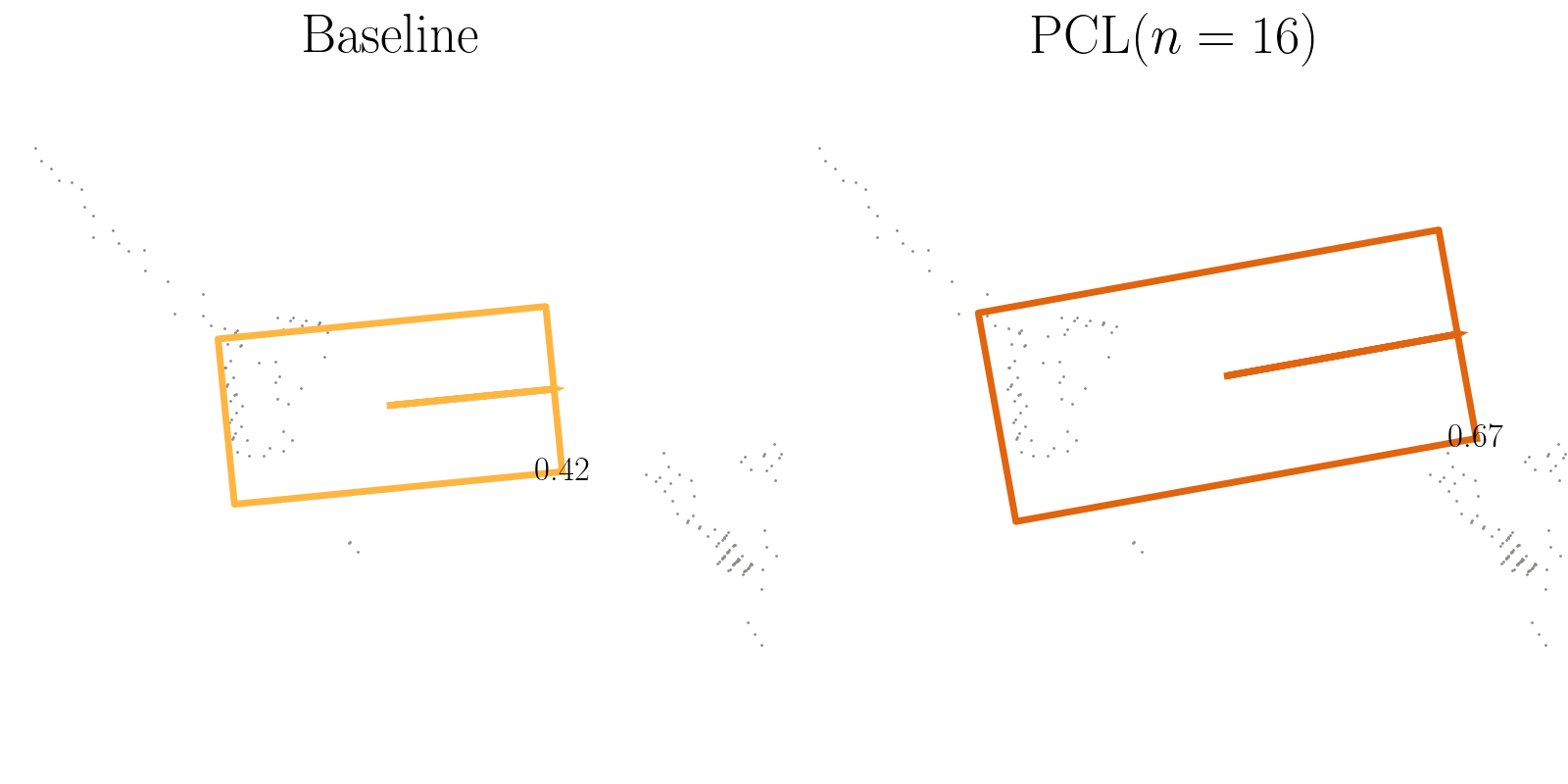}
    \includegraphics[trim={0pt 0pt 0pt 0pt},clip, width=0.23\textwidth]{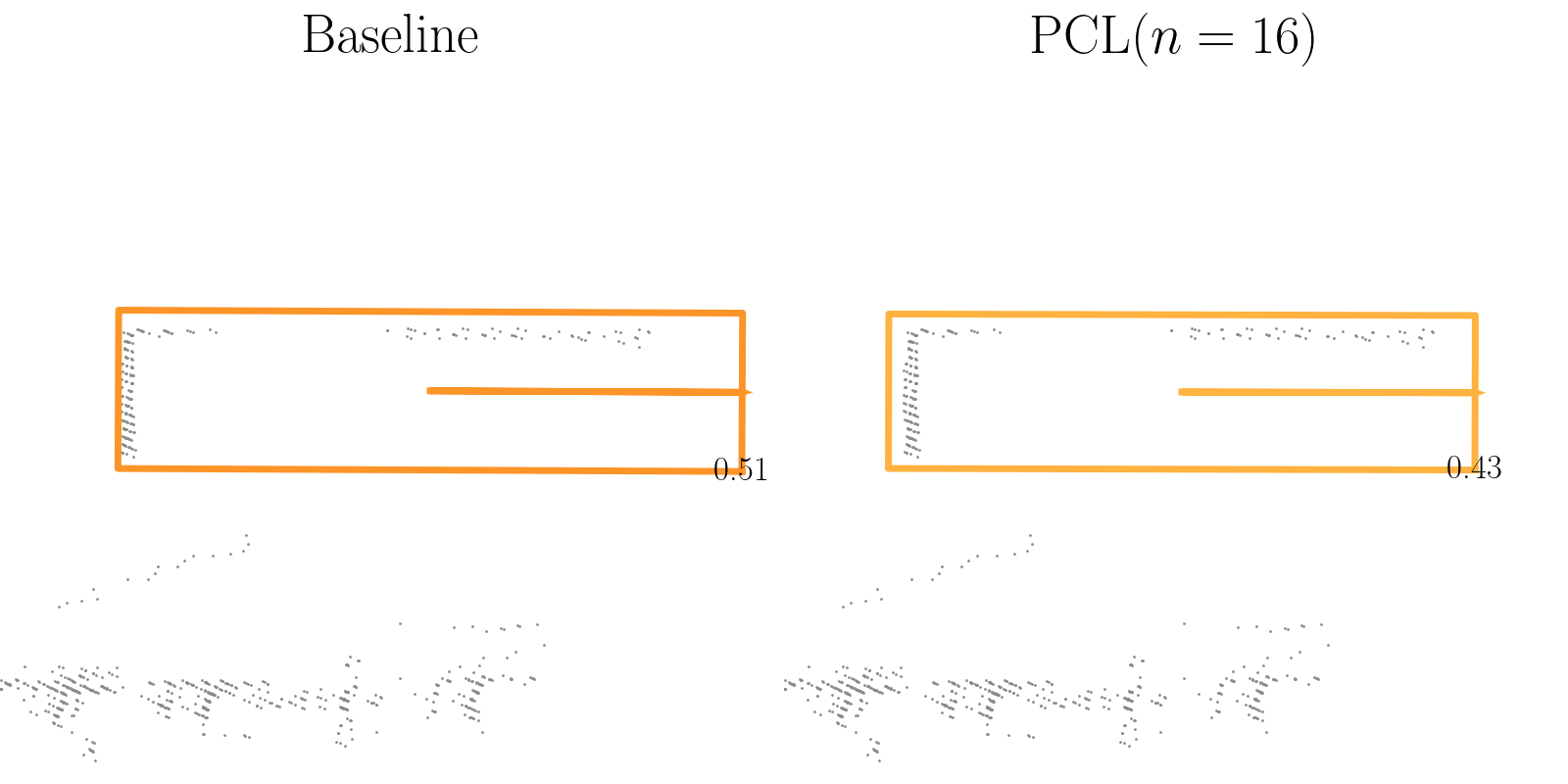}
    \includegraphics[trim={0pt 0pt 0pt 0pt},clip, width=0.23\textwidth]{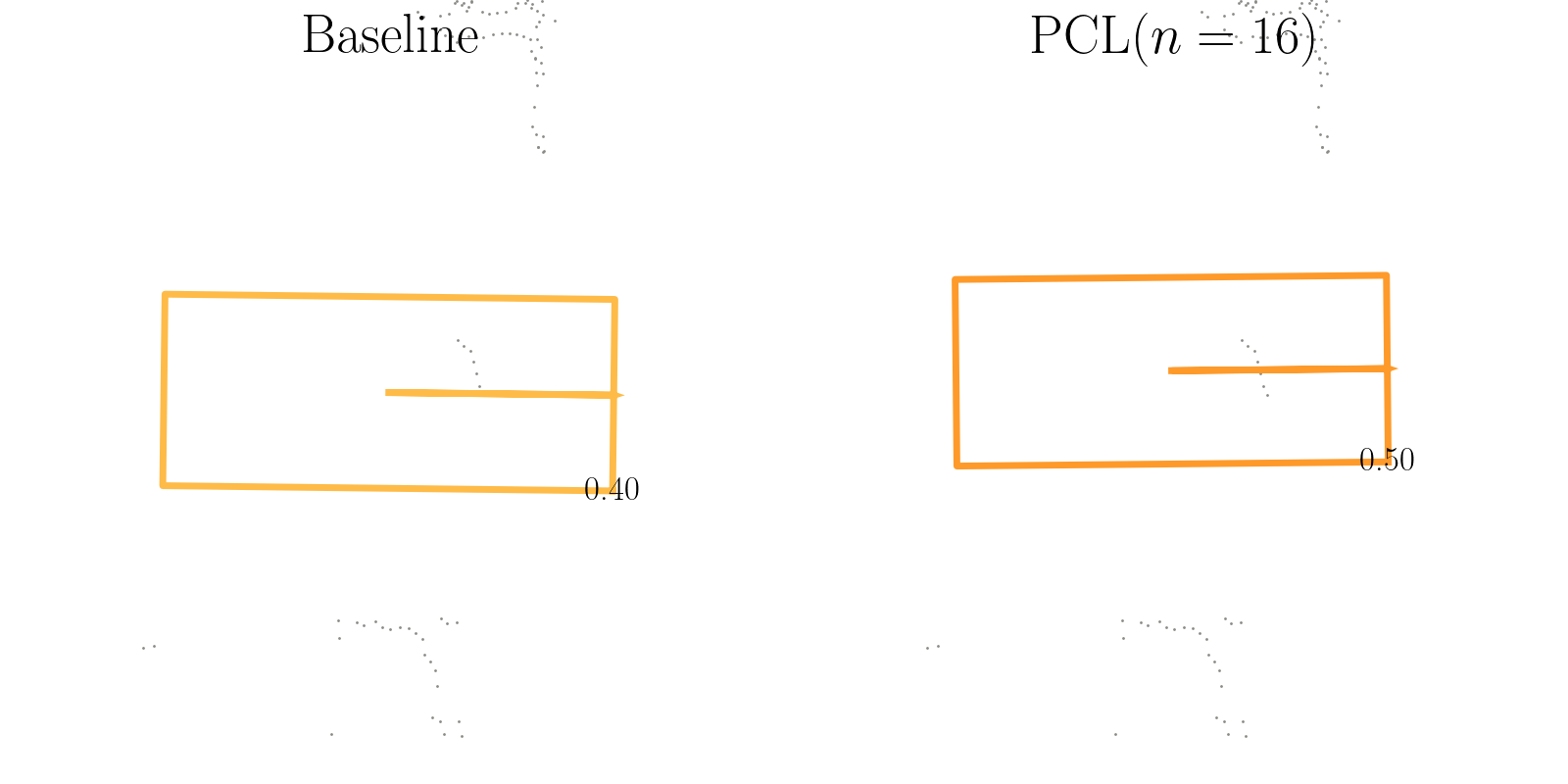}

    \includegraphics[trim={0pt 0pt 0pt 0pt},clip, width=0.23\textwidth]{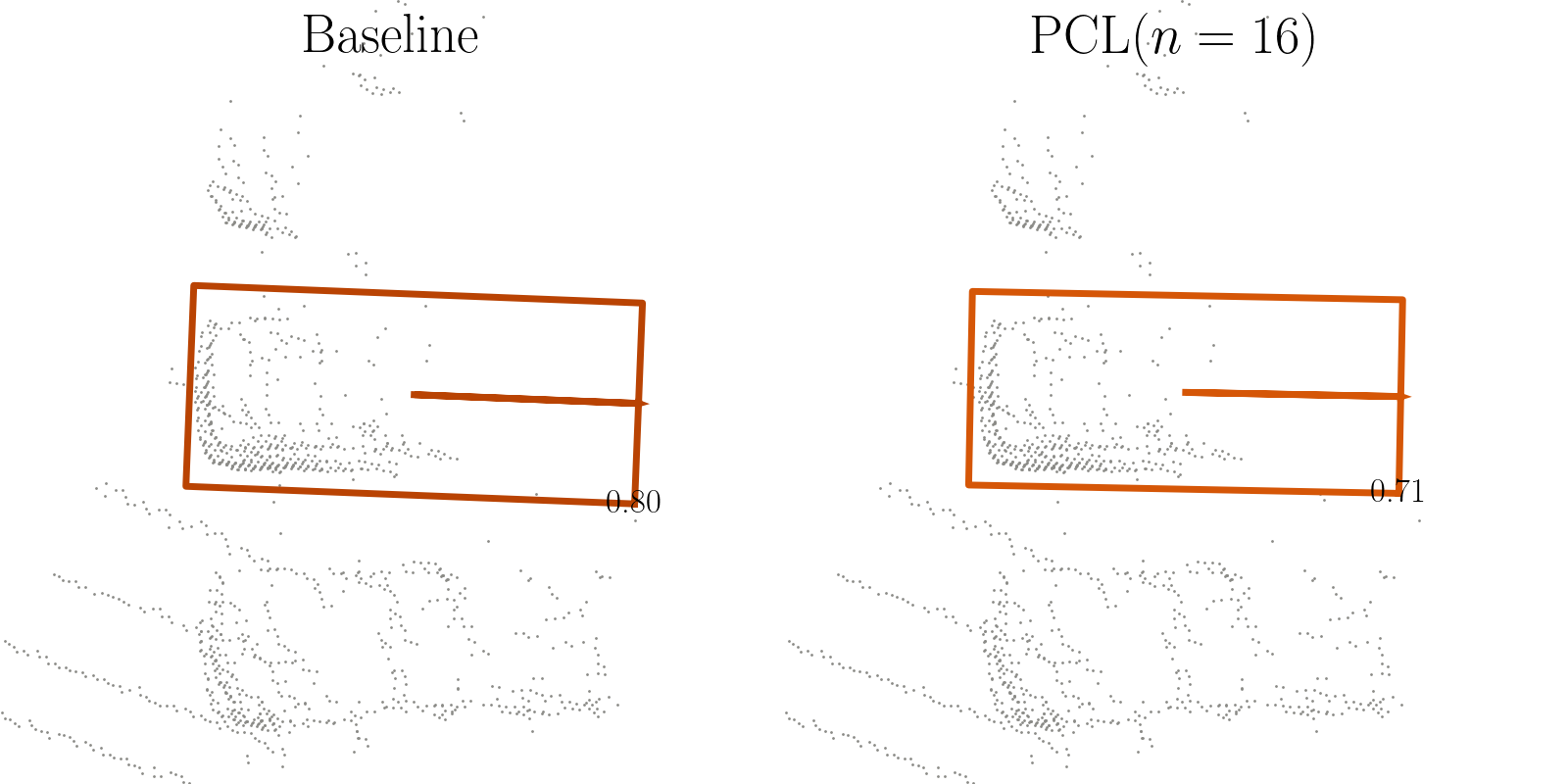}
    \includegraphics[trim={0pt 0pt 0pt 0pt},clip, width=0.23\textwidth]{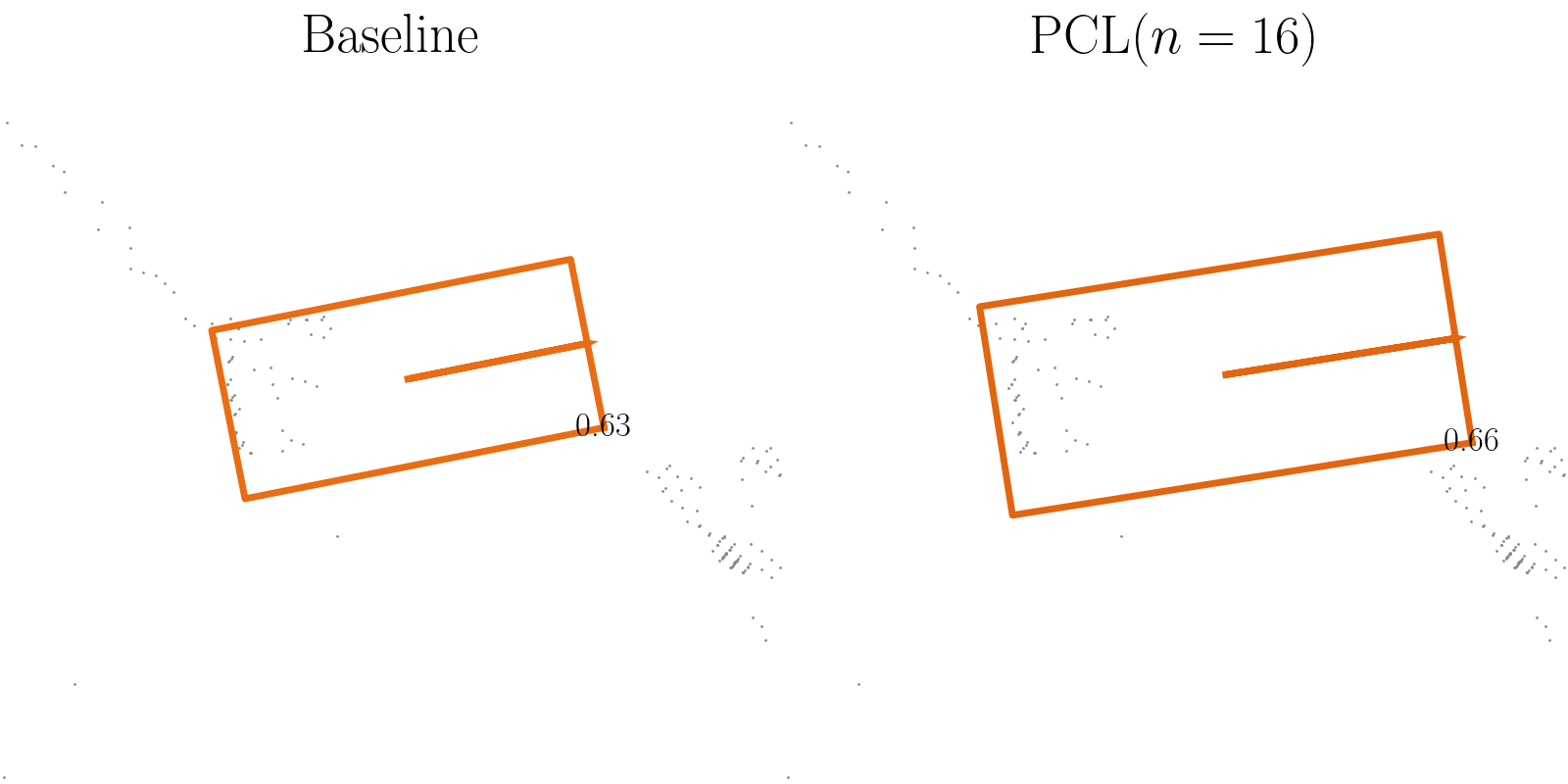}
    \includegraphics[trim={0pt 0pt 0pt 0pt},clip, width=0.23\textwidth]{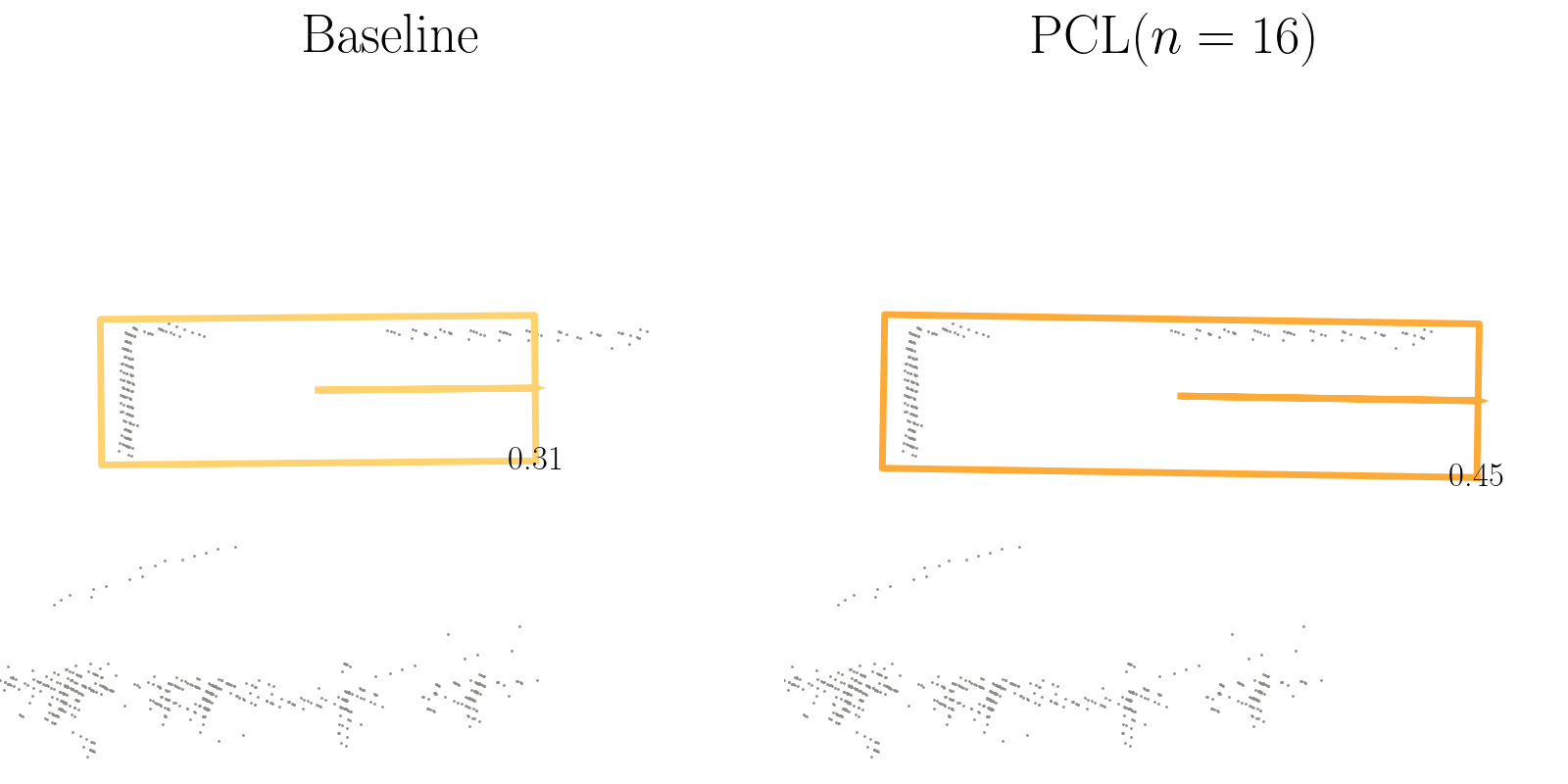}
    \includegraphics[trim={0pt 0pt 0pt 0pt},clip, width=0.23\textwidth]{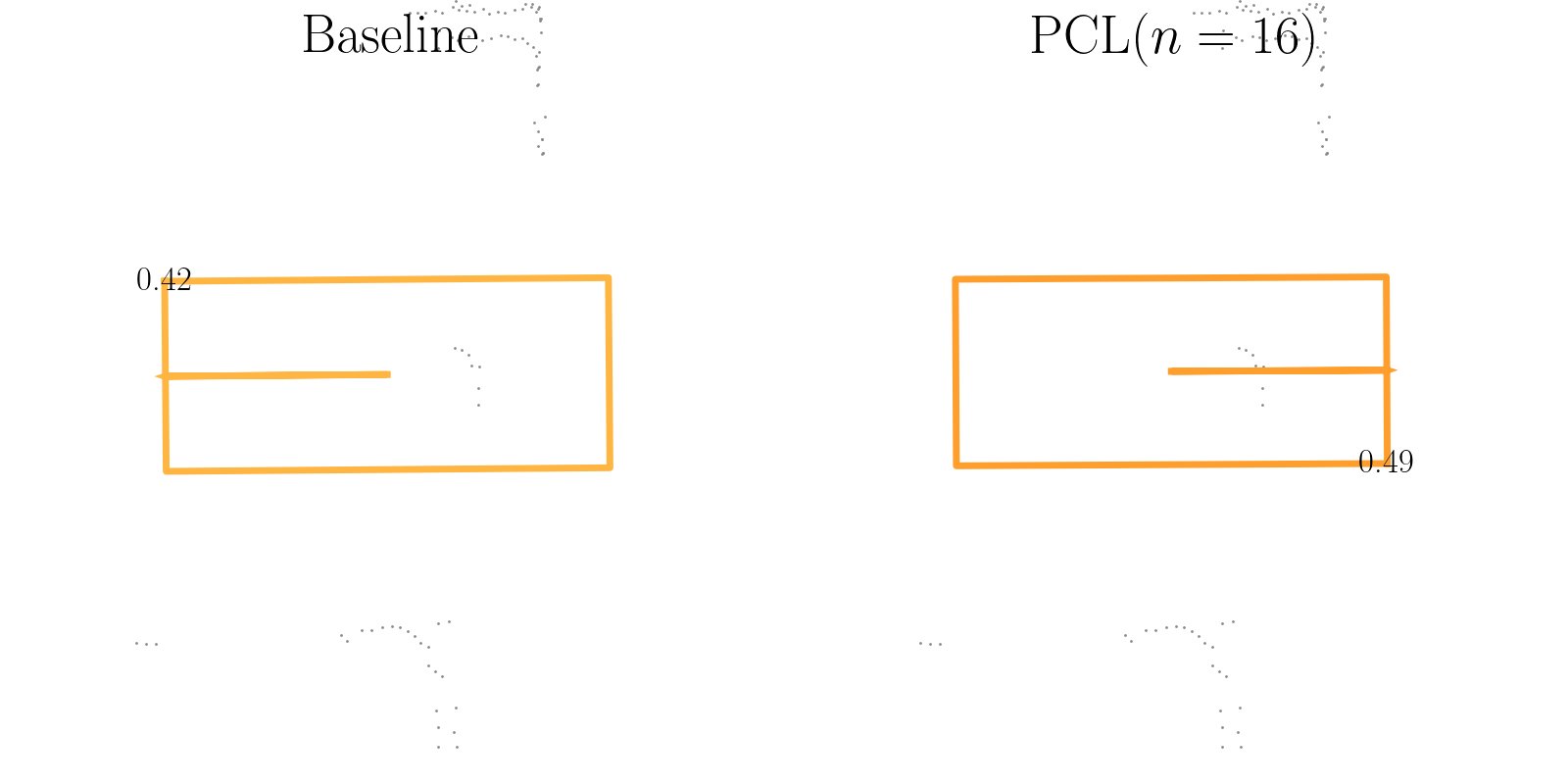}

    \includegraphics[trim={0pt 0pt 0pt 0pt},clip, width=0.23\textwidth]{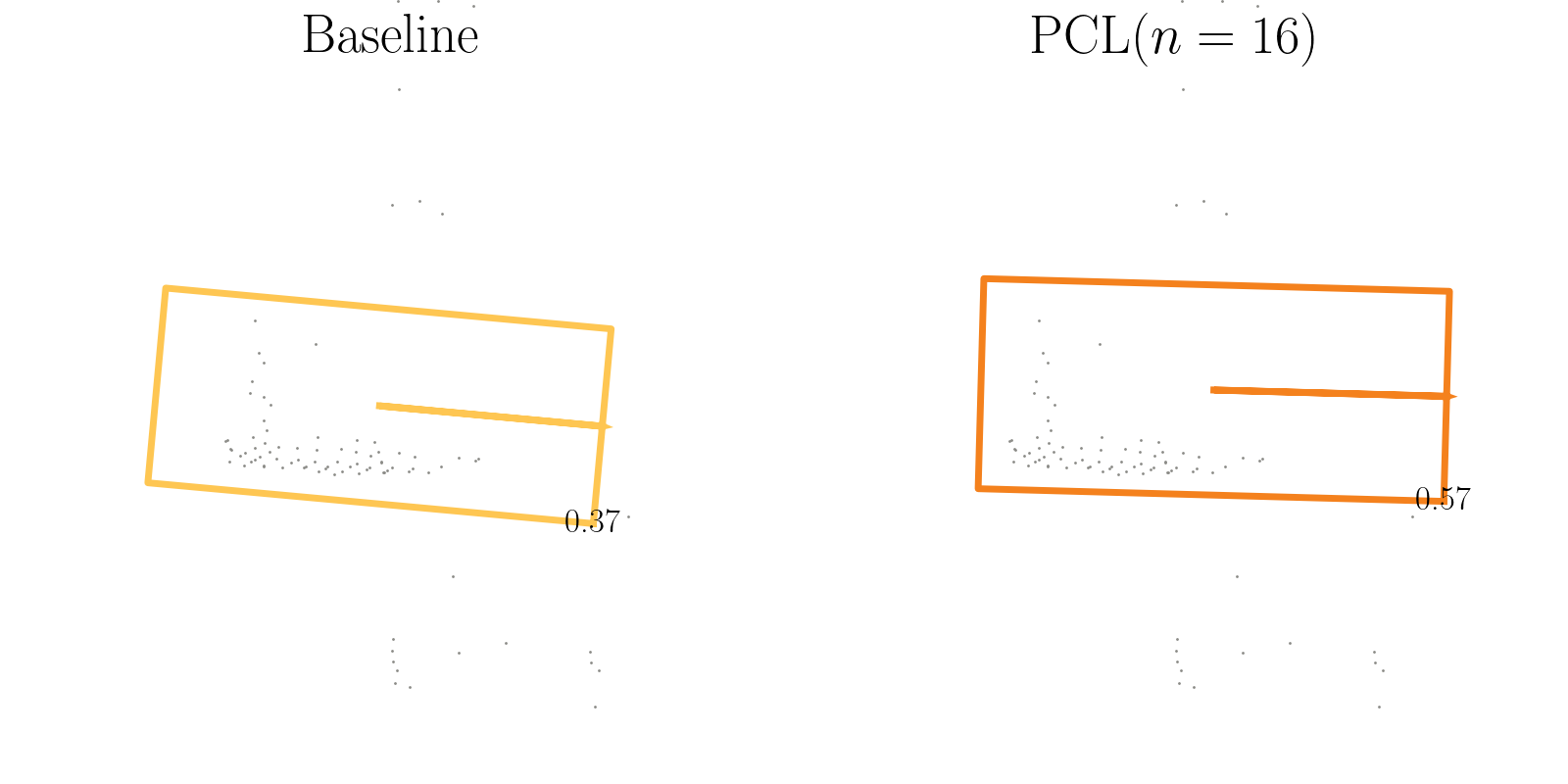}
    \includegraphics[trim={0pt 0pt 0pt 0pt},clip, width=0.23\textwidth]{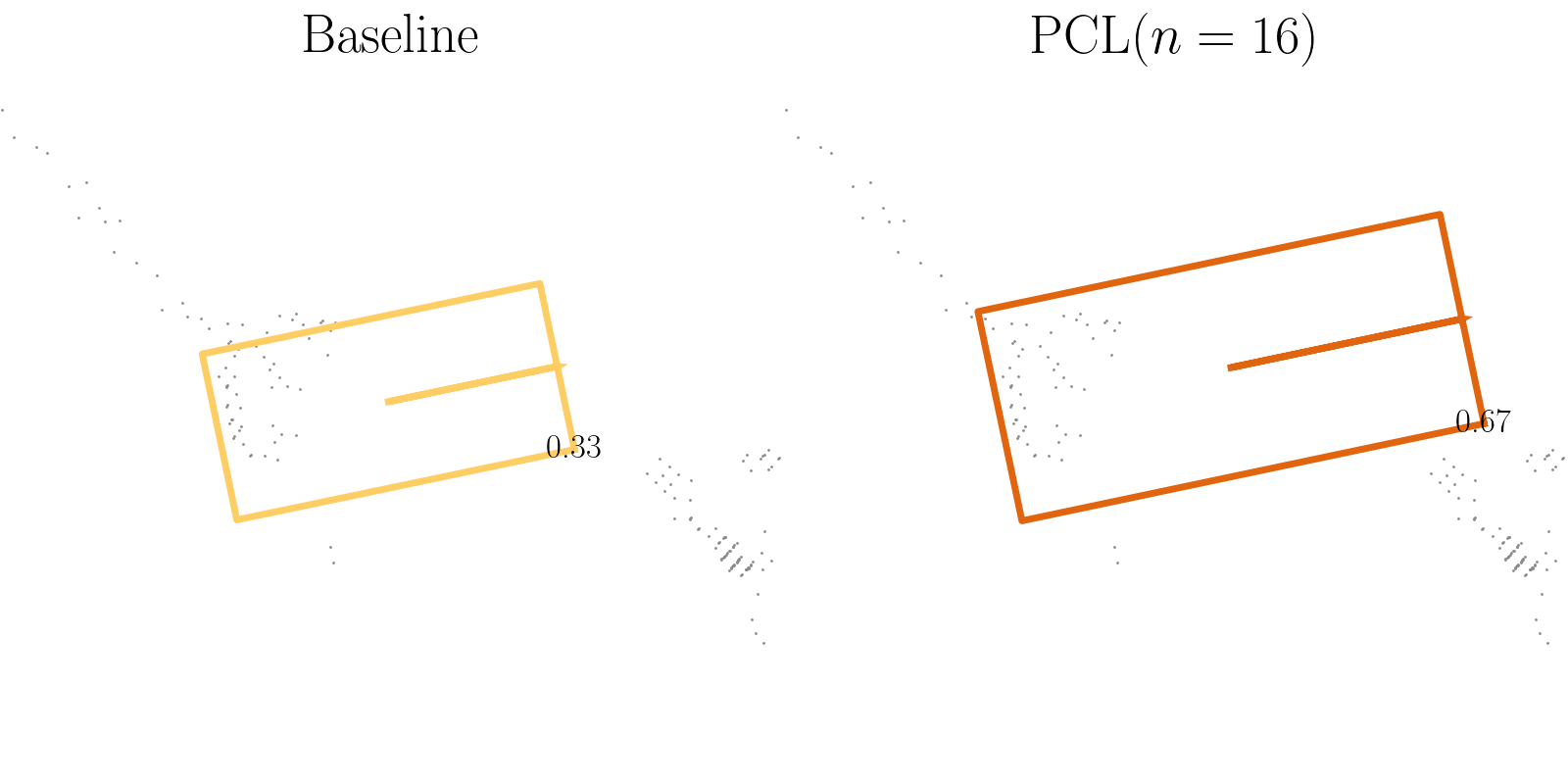}
    \includegraphics[trim={0pt 0pt 0pt 0pt},clip, width=0.23\textwidth]{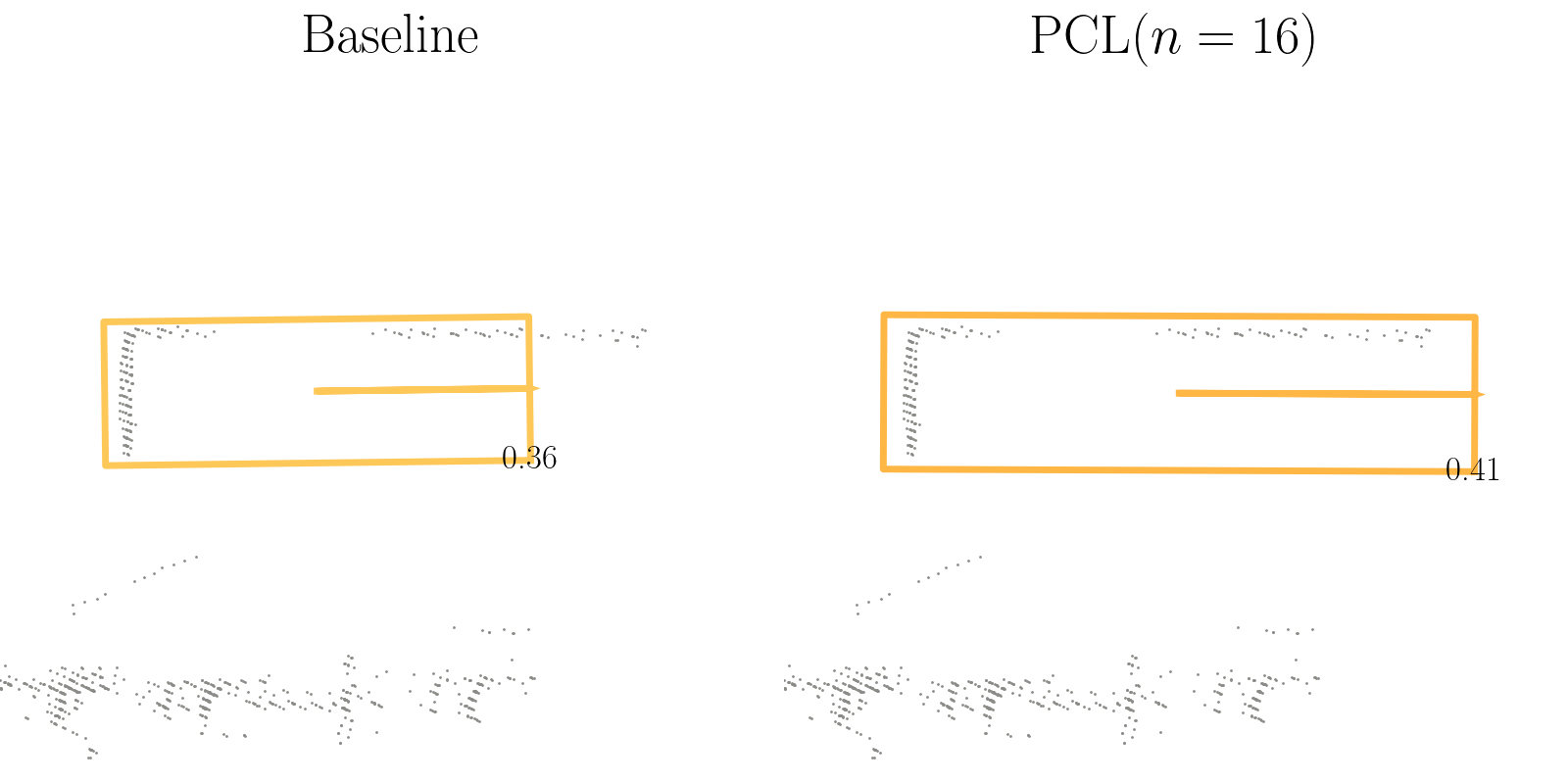}
    \includegraphics[trim={0pt 0pt 0pt 0pt},clip, width=0.23\textwidth]{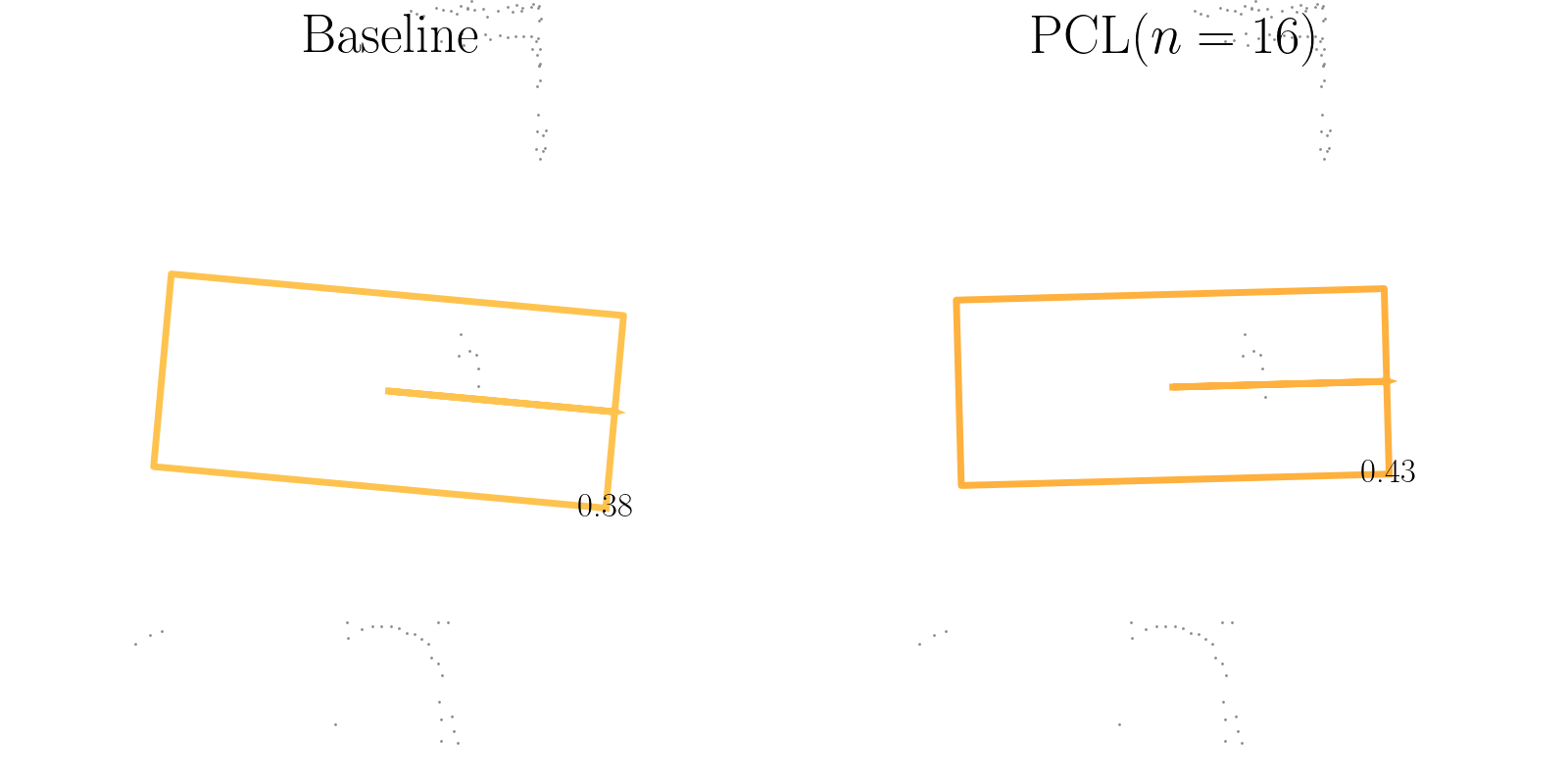}

    \caption{
        Comparisons between the detections predicted from the CetnerPoint trained with the baseline and PCL strategies.
        }
    \label{fig:comparison}
  \end{figure}
 
% In \cref{fig:tracking}, we present the object tracking results using the predictions from the CenterPoint~\cite{yin2021center} trained with the baseline and PCL strategies.
% As illustrated, the proposed PCL exhibits a strong improvement in tracking, where the erroneous velocity estimations are obviously reduced.
% This indicates that our PCL can alleviate the unstable detection problem, leading to better object tracking results.

To compare the detection stability of the baseline and PCL models more intuitively, we visualize a series of predictions across consecutive frames in \cref{fig:comparison}.
As shown in \cref{fig:comparison}, the detections predicted from the PCL model have less fluctuation than those from the baseline model in all aspects, which further demonstrates the effectiveness of PCL in enhancing the model stability.